%% file: main.tex
\documentclass[manuscript, screen, nonacm]{jair}

\setcopyright{none}

\usepackage[linesnumbered, titlenumbered,ruled,noend]{algorithm2e}
\usepackage{siunitx}
\usepackage{algpseudocode}
\usepackage{tikz}
\usepackage{textcomp}
\usepackage{graphicx}
\usepackage{multirow}
\usepackage{xcolor}
\usepackage{manyfoot}
\usepackage{booktabs}
\usepackage{float}
\usepackage{bm}
\usepackage{centernot}
\usepackage{listings}
\usepackage{scalerel}
\usepackage{subfig}
\usepackage{enumitem}
\usepackage{accents}
\usepackage{array}
\usepackage{xparse}
\usepackage{tabularx}
\usepackage{pdpshapes}
\usepackage{nomencl}
\usepackage{pdflscape}
\usepackage{orcidlink}

\input{commands}

\RequirePackage[
  datamodel=acmdatamodel,
  style=acmauthoryear,
  backend=biber,
  giveninits=true,
  uniquename=init
  ]{biblatex}

\begin{document}

\title[Sequence Variables]{Sequence Variables: A Constraint Programming Computational Domain for Routing and Sequencing}

\author{Augustin Delecluse\,\orcidlink{0000-0001-6285-6515}}
\authornote{Corresponding Author.}
\orcid{0000-0001-6285-6515}
\email{augustin.delecluse@kuleuven.be}
\affiliation{%
  \institution{KU Leuven}
  \city{Leuven}
  \country{Belgium}
}
\affiliation{%
	\institution{UCLouvain}
	\city{Louvain-La-Neuve}
	\country{Belgium}
}

\author{Pierre Schaus\,\orcidlink{0000-0002-3153-8941}}
\orcid{0000-0002-3153-8941}
\email{pierre.schaus@uclouvain.be}
\affiliation{%
  \institution{UCLouvain}
  \city{Louvain-la-Neuve}
  \country{Belgium}
}

\author{Pascal Van Hentenryck\,\orcidlink{0000-0001-7085-9994}}
\orcid{0000-0001-7085-9994}
\email{pascal.vanhentenryck@isye.gatech.edu}
\affiliation{%
  \institution{GATECH}
  \city{Atlanta}
  \country{USA}
}

\renewcommand{\shortauthors}{Delecluse, Schaus \& Van Hentenryck}

\begin{abstract}
Constraint Programming (CP) offers an intuitive, declarative framework for modeling Vehicle Routing Problems (VRP). 
While classical successor-based CP models can be adapted to handle optional visits or insertion-based heuristics, sequence variables provide a significantly more natural and elegant formulation for these requirements. 
Building upon our prior work that introduced the initial concept, the main contribution of this article is the complete semantic and operational formalization of sequence variables as a computational domain. 
Specifically, we formally define the sequence domain and its update operations, and detail the implementation and data structures required to integrate sequence variables into trail-based CP solvers. 
Furthermore, we introduce consistency levels for associated constraints on this domain alongside specialized global constraints tailored for routing problems. 
Finally, we demonstrate that sequence variables simplify problem modeling while achieving competitive computational performance on Pickup and Delivery Problems with and without Time Windows, the Dial-a-Ride Problem, and a Prize-Collecting Scheduling Problem.
\end{abstract}

\maketitle

\section{Introduction}
\label{sec:introduction}

The Vehicle Routing Problem (VRP) has many variants \citep{braekers2016vehicle}. Constraint Programming (CP) is a widely used approach for solving them \citep{kilby2006vehicle}, as its declarative modeling paradigm allows for flexible adaptation to constraints and objectives.
However, existing CP approaches struggle to handle optional visits and insertion-based search strategies, which are crucial to quickly obtain high-quality solutions.
As a motivating example for using sequence variables, we consider the Dial-A-Ride Problem (DARP).

Intuitively, the DARP involves routing a fleet of vehicles to satisfy a set of transportation requests for people.
Each request corresponds to a user's need to be picked up at an origin location and dropped off at a destination.
When boarding, the passenger(s) occupy a certain amount of capacity in the vehicle (which may vary, for instance, if a wheelchair is required or if there are multiple passengers to load), and this capacity is freed upon reaching their destination.
Formally, the input to DARP consists of the number $m$ of vehicles, $n$ requests, the maximum ride time $L$ for customers, the maximum route duration $D$, and the planning time horizon $T$, i.e., the hours between which the vehicles can operate. A DARP is defined on a complete graph $G = (V, E)$ where $V = \{v_0, v_1, \dots, v_{2n}\}$ is the set of nodes and $E = \{(v_i, v_j) : v_i, v_j \in V, i \neq j\}$ is the set of edges. Node $v_0$ denotes the depot. Each request $i$ ($1 \le i \le n$) consists of a pair of nodes $(v_i, v_{i+n})$. For each node $v_j \in V$, we associate a service duration $\bm{s}_j \ge 0$, a load change $\bm{q}_j$ and a time window $[e_j, l_j]$. We also have $\bm{s}_0 = 0, \bm{q}_0 = 0$, and $e_0 = 0, l_0 = T$. The distance between nodes $v_i$ and $v_j$ is denoted by $\bm{d}_{i,j}$. For each request $i$, we denote its pickup node by $\pick{r_i} = v_i$ (where $\bm{q}_i > 0$) and its delivery node by $\drop{r_i} = v_{i+n}$ (where $\bm{q}_{i+n} = -\bm{q}_i$).
Given these definitions, a DARP consists in finding a route for each of the $m$ vehicles such that:
\begin{enumerate}
    \item The route begins and ends at the depot;
    \item The load of a vehicle never exceeds its capacity $c$;
    \item The total duration of each route (i.e., the difference between the end time and the start time at the depot) never exceeds $D$;
    \item For each request $i$, $\pick{r_i}$ and $\drop{r_i}$ are serviced by the same vehicle and $\drop{r_i}$ is visited after $\pick{r_i}$;
    \item The ride time of any customer $i$ (i.e., the difference between the serving time at the delivery node $\drop{r_i}$ and the departure time at the pickup node $\pick{r_i}$) does not exceed $L$;
    \item For each node $v_j$, the starting time of its service lies between $[e_j, l_j]$; and
    \item The total routing cost of all vehicles is minimized.
\end{enumerate}

Two CP models and their search procedures are given next, illustrating the \emph{CP = model + search} paradigm.

\paragraph{Comparison of CP formulations}

To provide a clear high-level understanding, a classical successor-based CP model (Equations \ref{eq:minizinc-darp-objective}--\ref{eq:minizinc-darp-max-route-duration}) and a model using the proposed sequence variables (Equations \ref{eq:darp-objective}--\ref{eq:darp-max-route-duration}) are presented in parallel below to highlight their fundamental differences. Both models use standard mathematical notation, and the precise semantics of the sequence-based global constraints are detailed later in Section \ref{sec:global_constraints}.
For both models, let $K = \{1, \dots, m\}$ and $R = \{r_1, \dots, r_n\}$ denote the sets of vehicles and requests, respectively. The set of nodes $V$ and the distance matrix $\bm{d}$ are extended to introduce nodes representing the start ($\first_k$) and end ($\last_k$) nodes for each vehicle $k \in K$, which are duplicates of the depot. We write $V_{\mathrm{end}} = \{\last_k \mid k \in K\}$ and $V_{\mathrm{tr}} = V \setminus \{\first_k, \last_k \mid k \in K\}$ for the end nodes and transportation nodes, respectively.

The successor model relies on a successor variable \(\cpvar{Succ}_v\) for each node $v$, representing the node visited immediately after $v$. First, the objective \eqref{eq:minizinc-darp-objective} minimizes the total distance between nodes and their successors. To ensure that these predecessor-successor relationships form a valid set of disjoint routes without disconnected sub-tours, the model connects all routes into a single giant Hamiltonian circuit by chaining the end node of vehicle $k$ to the start node of vehicle $k+1$ \eqref{eq:minizinc-darp-circuit-link}. This allows the use of a single global \emph{circuit} constraint \eqref{eq:darp-circuit} over all nodes. Time tracking is achieved by local channeling constraints that propagate time to the successor \(\cpvar{Succ}_v\) \eqref{eq:minizinc-darp-time-channeling}, and time windows are explicitly enforced via \eqref{eq:minizinc-darp-tw}. Similarly, load is initialized to zero \eqref{eq:minizinc-darp-load-init} and tracked along the route \eqref{eq:minizinc-darp-load}, with the vehicle capacity explicitly enforced via \eqref{eq:minizinc-darp-capacity}. Vehicle assignment is handled by channeling the vehicle along the route \eqref{eq:minizinc-darp-vehicle-channeling} and pairing pickup and delivery nodes \eqref{eq:minizinc-darp-same-vehicle-request}.

In contrast, the sequence variable model directly represents the route of each vehicle as an explicitly ordered sequence of nodes. It relies on an insertion-based sequence variable $\cpvar{Route}_k$ for each vehicle $k \in K$. First, the objective \eqref{eq:darp-objective} minimizes the sum of distances $\cpvar{Dist}_k$ linked to each route via the \emph{Distance} constraint \eqref{eq:darp-distance}. Structurally, every transportation node $v \in V_{\mathrm{tr}}$ must be visited exactly once \eqref{eq:darp-disjoint}, and the sequence is explicitly constrained to start and end with the designated depot nodes via \eqref{eq:darp-first} and \eqref{eq:darp-last}. Because a sequence variable inherently maintains the total order of visits, it naturally guarantees valid routes without the need for a \emph{circuit} constraint. This explicit ordered structure allows the use of specialized global constraints that operate directly on the entire sequence. For instance, time tracking is managed by the \emph{TransitionTimes} constraint \eqref{eq:darp-transition}, with time windows explicitly enforced via \eqref{eq:darp-tw}. Similarly, load tracking and vehicle assignment are handled by the \emph{Cumulative} constraint \eqref{eq:darp-cumulative}, which maintains the capacity limits, ensures that the pickup and delivery nodes of a request are serviced by the same vehicle, and enforces that the delivery occurs after the pickup.

Finally, the maximum ride time and route duration constraints \eqref{eq:minizinc-darp-max-ride-time}--\eqref{eq:minizinc-darp-max-route-duration}, \eqref{eq:darp-max-ride-time}--\eqref{eq:darp-max-route-duration} are identical in both models.

\begin{minipage}[t]{0.485\textwidth}
\centering
\textbf{Successor Model}
\begin{equation}
    \min \sum_{v \in V \setminus V_{\mathrm{end}}}{\bm{d}_{v, \cpvar{Succ}_v}} \label{eq:minizinc-darp-objective}
\end{equation}
{\small
\begin{align}
    & \cpvar{Succ}_{\last_k} = \first_{k \bmod m + 1} &\quad \forall k \in K \label{eq:minizinc-darp-circuit-link} \\
    & \mathrm{circuit}(\cpvar{Succ}) \label{eq:darp-circuit} \\
    & \cpvar{Time}_v + \bm{s}_v + \bm{d}_{v, \cpvar{Succ}_v} \leq \cpvar{Time}_{\cpvar{Succ}_v} &\quad \forall v \in V \setminus V_{\mathrm{end}} \label{eq:minizinc-darp-time-channeling}\\
    & e_v \leq \cpvar{Time}_v \leq l_v &\quad \forall v \in V \label{eq:minizinc-darp-tw} \\
    & \cpvar{Load}_{\first_k} = 0 &\quad \forall k \in K \label{eq:minizinc-darp-load-init} \\
    & \cpvar{Load}_v + \bm{q}_v = \cpvar{Load}_{\cpvar{Succ}_v} &\quad \forall v \in V \setminus V_{\mathrm{end}} \label{eq:minizinc-darp-load}\\
    & 0 \leq \cpvar{Load}_v \leq c &\quad \forall v \in V \label{eq:minizinc-darp-capacity} \\
    & \textbf{\textit{Vehicle}}_v = \textbf{\textit{Vehicle}}_{\cpvar{Succ}_v} &\quad \forall v \in V \setminus V_{\mathrm{end}} \label{eq:minizinc-darp-vehicle-channeling}\\
    & \textbf{\textit{Vehicle}}_{\pick{r_i}} = \textbf{\textit{Vehicle}}_{\drop{r_i}} &\quad \forall r_i \in R \label{eq:minizinc-darp-same-vehicle-request}\\
    & \cpvar{Time}_{\drop{r_i}} - (\cpvar{Time}_{\pick{r_i}} + \bm{s}_{\pick{r_i}}) \leq L &\quad \forall r_i \in R \label{eq:minizinc-darp-max-ride-time}\\
    & \cpvar{Time}_{\last_k} - \cpvar{Time}_{\first_k} \leq D &\quad \forall k \in K \label{eq:minizinc-darp-max-route-duration}
\end{align}
}
\end{minipage}
\hfill
\begin{minipage}[t]{0.485\textwidth}
\centering
\textbf{Sequence Variable Model}
\begin{equation}
    \min \sum_{k \in K}{\cpvar{Dist}_k} \label{eq:darp-objective}
\end{equation}
{\small
\begin{align}
    & \mathrm{Distance}(\cpvar{Route}_k, \bm{d}, \cpvar{Dist}_k) &\quad \forall k \in K \label{eq:darp-distance} \\
    & \sum_{k \in K} \nodeRequired{\cpvar{Route}_k}{v} = 1 &\quad \forall v \in V_{\mathrm{tr}} \label{eq:darp-disjoint}\\
    & \mathrm{first}(\cpvar{Route}_k) = \first_k &\quad \forall k \in K \label{eq:darp-first} \\
    & \mathrm{last}(\cpvar{Route}_k) = \last_k &\quad \forall k \in K \label{eq:darp-last} \\
    & \mathrm{TransitionTimes}(\cpvar{Route}_k, \cpvar{Time}, \bm{s}, \bm{d}) &\quad \forall k \in K \label{eq:darp-transition}\\
    & e_v \leq \cpvar{Time}_v \leq l_v &\quad \forall v \in V \label{eq:darp-tw} \\
    & \mathrm{Cumulative}(\cpvar{Route}_k, \pick{r}, \drop{r}, \bm{q}, c) &\quad \forall k \in K \label{eq:darp-cumulative}\\
    & \cpvar{Time}_{\drop{r_i}} - (\cpvar{Time}_{\pick{r_i}} + \bm{s}_{\pick{r_i}}) \leq L &\quad \forall r_i \in R \label{eq:darp-max-ride-time}\\
    & \cpvar{Time}_{\last_k} - \cpvar{Time}_{\first_k} \leq D &\quad \forall k \in K \label{eq:darp-max-route-duration}
\end{align}
}
\end{minipage}

\paragraph{The search}
In Constraint Programming (CP), finding a solution typically relies on a backtracking depth-first search (DFS). This search is customized by a branching procedure that selects a variable and assigns it a value at each step. A branching procedure is also often referred to as a heuristic. In particular, the leftmost branch of the search tree should mimic a strong greedy construction heuristic in order to quickly find a high-quality first solution.

In the classical successor model, branching typically consists of selecting a node $v$ without assigned successor $\cpvar{Succ}_v$, often following a \emph{first-fail} principle by choosing the node with the smallest $\cpvar{Succ}_v$ domain size, and trying different values for its successor. By evaluating values in order of increasing distance, the leftmost branch naturally mimics a \emph{nearest neighbor} strategy, building routes forward by repeatedly appending the closest unvisited node. This standard procedure is illustrated in Algorithm~\ref{alg:branching_darp_succ}. However, as shown by \cite{rosenkrantz1977analysis}, nearest neighbor strategies often suffer from the ``long edge effect,'' where a very long and costly connection is required to close the tour. 

\begin{algorithm}[!ht]
\caption{Creation of the branching points in the Successor Model.}
\label{alg:branching_darp_succ}
\If{$\bigwedge_{v \in V} \cpvar{Succ}_v.\mathrm{isFixed}()$}{
    \Return solution \;
}
$v \gets \underset{u \in V \mid \neg \cpvar{Succ}_u.\mathrm{isFixed}()}{\arg\!\min}{\; |\cpvar{Succ}_u.\mathrm{domain}()| }$  \label{alg:branching_succ_get_node} \;
$\textit{branches} \gets \{\}$ \;
\For{$w \in \cpvar{Succ}_v.\mathrm{domain}()$}{
    $\textit{branches} \gets \textit{branches} \cup \{ \cpvar{Succ}_v = w \}$ \;
}
sort $\textit{branches}$ by increasing distance $\bm{d}_{v, w}$ \;
\Return $\textit{branches}$ \;
\end{algorithm}

Conversely, an \emph{insertion-based} heuristic avoids this issue by selecting an unvisited node and inserting it into the position that minimizes the cost increase within a partial route. While an insertion-based branching strategy cannot be easily achieved in a successor model, the sequence variable model natively supports it. Because sequence variables explicitly represent the ordered route, the branching decision directly corresponds to choosing an uninserted node and inserting it at a specific position in the sequence. By sorting these insertion decisions by their cost increase, the leftmost branch exactly mimics a greedy insertion heuristic.

An insertion-based branching procedure for the sequence variable model is presented in Algorithm~\ref{alg:branching_seq}. This procedure generates the alternative child nodes of the search tree by considering all feasible insertions of a chosen unvisited node into the sequence variables representing the active routes.

\begin{algorithm}[!ht]
\caption{Creation of the branching points in the Sequence Variable Model.}
\label{alg:branching_seq}
\If{$\bigwedge_{k \in K} \cpvar{Route}_k.\mathrm{isFixed}()$}{
    \Return solution \; \label{alg:branching_seq-solution}
}
$v \gets \underset{u \in V \mid u \text{ not yet inserted}}{\arg\!\min}{\; \sum_{k\in K} { \cpvar{Route}_k.\mathrm{nInsert}(u)} }$  \label{alg:branching_seq_get_node} \;
$\textit{branches} \gets \{\}$ \;
\For{$k \in K$}{
    $I \gets \cpvar{Route}_k.\mathrm{getInsert}(v)$ \; \label{alg:branching_seq:select-insert}
    \For{$p \in I$}{
        $\textit{branches} \gets \textit{branches} \cup \left\{ \cpvar{Route}_k.\mathrm{insert}(p, v) \right\}$ \label{alg:branching_seq:perform-insertion}
    }
}
sort $\textit{branches}$ by increasing heuristic value \label{alg:branching_seq_sort} \;
\Return $\textit{branches}$ \;
\end{algorithm}

The branching proceeds in two steps.
First it prioritizes the unvisited node $v$ with the fewest remaining feasible insertions (line~\ref{alg:branching_seq_get_node}).
This is in line with the \emph{first-fail} principle, which aims to minimize the depth of the search tree by making decisions that are most likely to lead to an early failure if no solution exists in the current sub-tree.
Given this node, all valid insertion points $I$ in a given sequence variable $\cpvar{Route}_k$ are retrieved (line \ref{alg:branching_seq:select-insert}).
Each feasible predecessor $p$ is considered as an alternative decision to explore (line \ref{alg:branching_seq:perform-insertion}).
Since DFS is used, the most promising insertions should be explored first. 
This can, for instance, be achieved by sorting all candidate insertions based on their impact on tour length, prioritizing those that minimize the increase in distance (line~\ref{alg:branching_seq_sort}).
A solution is reached when no further insertions are possible in any vehicle: all paths are fixed (line~\ref{alg:branching_seq-solution}).

A DFS using the branching of Algorithm~\ref{alg:branching_seq} can be used to find an initial solution by stopping at the first feasible solution encountered.
A limited-discrepancy DFS search, keeping only the leftmost few branches, as was done in \cite{shaw1998using}, can also be used to quickly discover good solutions.

\paragraph{Large Neighborhood Search}
These strategies can also be embedded into a Large Neighborhood Search (LNS) framework \cite{shaw1998using}. Intuitively, LNS iteratively improves a solution by destroying a part of it (relaxing a subset of variables) and then reconstructing it using a search procedure. In the context of vehicle routing, this typically means removing a set of customer visits from the current routes and then re-inserting them at cheaper positions. Because LNS relies on a sequence of local, restricted searches rather than an exhaustive exploration of the full search space, it is a highly effective heuristic to quickly discover good solutions, though it provides no guarantee of finding the optimal one.

When applying LNS to the classical successor model, relaxing a set of requests means unbinding the $\cpvar{Succ}$ variables associated with their pickup and delivery nodes, as well as the $\cpvar{Succ}$ variables of the nodes that preceded them in the previous solution. As shown in Algorithm \ref{alg:lns_darp_succ}, this process breaks the sequence into multiple disconnected segments. The CP solver must then re-link these segments during the search phase while inserting the relaxed nodes.

\begin{algorithm}[!ht]
\caption{A LNS iteration for DARP in the Successor Model}
\label{alg:lns_darp_succ}
\SetKwInput{Input}{Input}
\Input {$S$: a successor array representing a visit mapping for each vehicle}
$\mathcal{R} \gets $ relaxedRequests(S) \;
$\mathcal{V}^\mathcal{R} \gets \bigcup_{r_i \in \mathcal{R }} \{\pick{r_i}, \drop{r_i}\}$ \;
\For{$v \in V$}{
    $w \gets S[v]$ \tcp{The successor of $v$ in the previous solution}
    \If{$v \notin \mathcal{V}^\mathcal{R}$ \textbf{and} $w \notin \mathcal{V}^\mathcal{R}$}{
        $\cpvar{Succ}_v \gets w$ \tcp{Keep link if both nodes are kept}
    } \Else {
        $\cpvar{Succ}_v \gets \text{unbound}$ \tcp{Relax connection}
    }
}
$S \gets \text{Solve problem starting with partial }\cpvar{Succ}$ \;
\end{algorithm}

In contrast, applying LNS to the sequence variable model naturally maintains the partial tours, exactly corresponding to the original LNS for vehicle routing presented by \cite{shaw1998using}. As depicted in Algorithm \ref{alg:lns_darp}, the nodes belonging to relaxed requests are simply omitted when reconstructing the partial sequence (line \ref{alg:lns_darp:insert_previous_path}), while all remaining nodes seamlessly keep their relative ordering and adjacent connections. A search is then performed (line \ref{alg:lns_darp:launch_search}) to insert the relaxed requests back into these continuous partial paths, leading to a new solution $\hat{S}$. This process is repeated until a given stopping criterion is met.

\begin{algorithm}[!ht]
\caption{A LNS iteration for DARP in the Sequence Variable Model}
\label{alg:lns_darp}
\SetKwInput{Input}{Input}
\Input {$S=[S_1,\ldots,S_K]$: a sequence of visits for each vehicle}
$\mathcal{R} \gets $ relaxedRequests(S) \label{alg:lns_darp:select_request}\;
$\mathcal{V}^\mathcal{R} \gets \bigcup_{r_i \in \mathcal{R }} \{\pick{r_i}, \drop{r_i}\}$ \label{alg:lns_darp:relaxedNodes}\;
\For{$k \in K$}{
    $\cpvar{Route}_k \gets $ empty sequence variable \tcp{empty the vehicle path}
    \For(\tcp*[h]{iterate over the previous path $S_k$, in order}){$v \in S_k$ }{
        \If{$v \notin \mathcal{V}^\mathcal{R} $}{
            $\cpvar{Route}_k.\mathrm{insertAtEnd}(v)$ \tcp{append the visit in same order as in $S$}  \label{alg:lns_darp:insert_previous_path}
        }
    }
}
$S \gets \text{Solve problem starting with }\cpvar{Route}$  \tcp{Using branching from Algorithm \ref{alg:branching_darp}}\label{alg:lns_darp:launch_search}
\end{algorithm}

The structural difference between Algorithm~\ref{alg:lns_darp_succ} and Algorithm~\ref{alg:lns_darp} illustrates how sequence variables natively support local search heuristics like the \emph{relocate} move \cite{savelsbergh1985local}. In a sequence variable model, relaxing a node $v$ maintains the relative ordering of the remaining nodes, directly corresponding to removing $v$ from the route. The solver then evaluates its re-insertion at all feasible positions. Conversely, emulating this move in a successor model requires unbinding the predecessors and successors of both $v$ and its target insertion position. Because the optimal target is unknown a priori, one must relax the successors of all potential targets, which effectively destroys the route structure and unnecessarily expands the search space. By preserving partial routes, sequence variables natively generalize the relocate move to multi-node re-insertions, exploring neighborhoods analogous to $k$-opt or exchange moves efficiently within a single search tree.

\section{Limitations of Existing Approaches in CP}
\label{sec:limitations}

The sequence variable model and search strategy introduced for the DARP in the previous section rely on insertion-based sequence variables.
The successor model, also presented in the previous section and widely used in CP, does not support insertion-based search strategies.
A successor model uses an integer variable $\textit{succ}_i$ for representing the direct successor of customer $i$ in a tour.
To ensure that all customers are visited on a tour, the (Hamiltonian) \emph{circuit} global constraint \citep{Lauriere1978} and its weighted variants \citep{benchimol2012improved} can be used.
This kind of model has two main limitations.
First, representing the optional nature of visits with the successor model is not straightforward\footnote{Some modeling languages \citep{minizinc, xcsp3, opl} offer a subcircuit version that allows a node to be excluded from the circuit by assigning it a self-loop ($\textit{succ}_i = i$) but this complicates the model semantics, as constraints need to be aware that self-loops are permitted and designate an unvisited node}.
Second, at the search level, the first solution along the leftmost branch of the search tree typically relies on a \emph{nearest neighbor} heuristic.
These nearest neighbor heuristics tend to add small edges near the root of the search tree, but as decisions progress, they tend to add very long edges at the end, making it difficult to quickly find good solutions with CP.

An advanced CP alternative to the successor model, which enables dealing with optional visits more naturally, is to use the \emph{head-tail sequence variables} implemented in IBM CP Optimizer \citep{ibm_ilog_cp,Reasoning_with_Conditional_Time_Intervals_2}.
According to the documentation and in the context of VRP, the domain consists of two growing sequences of nodes. One is the head (prefix) and the other the tail (suffix) of nodes to be visited in that order.
The domain update operations append a node to the tail, append one to the head, or merge the two to form the final sequence.
These variables handle optionality by design since not every node needs to be added to the sequence.
However, similarly to the successor model, a heuristic on those variables would also rely on a nearest neighbor strategy.

Since both the successor and the head-tail sequence models do not easily support insertion heuristics, a new type of variable, the insertion-based sequence variable, was recently introduced in \citep{thomas2020insertion, delecluse2022sequence} to address these limitations.
These variables have a domain composed of the possible insertions for each node within a partial sequence. 
They thus consume more space than the head-tail sequence variables.
A summary of the main properties of the modeling variables for VRP in CP is given in Table \ref{tab:variables-comparison}.

\begin{table}[!ht]
    \centering
    \caption{Characteristics of CP variables.}
    \label{tab:variables-comparison}
    \begin{tabular}{rccc}
    \toprule
    & \multicolumn{3}{c}{Type of variable} \\
    \cmidrule(lr){2-4} 
        Feature & Successor & Head-tail sequence & Insertion-based sequence \\
        \midrule
        Nearest neighbor heuristics & \checkmark & \checkmark & \checkmark \\
        Insertion-based heuristics &  &  & \checkmark\\
        Optional visits & & \checkmark & \checkmark\\[2pt]
        \begin{tabular}{@{}c@{}}Memory complexity on VRP \\ with $n$ nodes and $k$ vehicles\end{tabular}
  & $\mathcal{O}(n^2)$ & $\mathcal{O}(k \cdot n)$ & $\mathcal{O}(k \cdot n^2)$\\
    \bottomrule
    \end{tabular}
\end{table}

In the rest of the article, insertion-based sequence variables are simply denoted \emph{sequence variables}.
This work significantly extends the previous work on sequence variables from \cite{thomas2020insertion, delecluse2022sequence} by:

\begin{itemize}
\item \textbf{Formalizing the computational model for sequence variables}, providing a robust theoretical foundation.
\item \textbf{Introducing consistency levels}, including the novel concept of \emph{insert consistency}.
\item \textbf{Proposing an implementation} along with the underlying data structures to integrate sequence variables into existing trail-based CP solvers.
\item \textbf{Developing global constraints} for modeling VRP with sequence variables.
\item \textbf{Demonstrating the effectiveness of sequence variables} on several vehicle routing problems and on one sequencing problem.
\end{itemize}

This paper is organized as follows.
Section \ref{sec:limitations} discusses the limitations of existing CP approaches for vehicle routing.
Section \ref{section:related} provides a brief overview of VRP solving with CP and reviews work related to sequence variables.
Section \ref{section:sequence_variables} details the computational model for sequence variables, including domain representations, operations, and consistency definitions.
Section \ref{sec:global_constraints} discusses global constraints tailored for sequence variables and their integration with vehicle routing problems.
Section \ref{section:search} describes search strategies to efficiently solve CP models involving sequence variables.
Finally, Section \ref{section:xp} presents experimental results on the Dial-A-Ride Problem, its variants without some or all of the temporal constraints, and on one sequencing problem. Those results validate the theoretical contributions and practical implementations.

\section{Related Work}
\label{section:related}

Constraint Programming (CP) is a declarative paradigm where problems are modeled through variables and constraints. Solutions are then found via a systematic search combined with constraint propagation. 
In CP, each variable $x$ has an associated domain $D(x)$ of possible values, and constraints actively filter these domains by removing inconsistent values.
At each node of the search tree, the search selects a variable and a value from its domain, applies a branching decision---typically $x = v$ or $x \neq v$---and propagates the constraints until a fix-point is reached. It backtracks if an inconsistency is detected.
We refer to \cite{minicp} for a more detailed background on CP solving.

The successor model described in \cite{kilby2006vehicle} is commonly used for modeling VRPs in CP, where each node is associated with an integer variable representing its next visited node.
An attempt to extend this model to support insertion-based search was proposed in \cite{kilby2000}, where a reversible data structure called \emph{insertion schedule} represents the direct predecessor and successor of the nodes already inserted by the branching procedure. A channeling mechanism is then used to maintain consistency with a new set of variables, called the \emph{insertion variables $\pi$}, which represent the possible insertion positions for visit $i$.
Although this approach enables advanced insertion procedures, such as ours, it does not provide a well-defined sequence variable with a computation domain that can be exploited to filter external constraints.

Another CP alternative relies on dedicated structured domains.
They have been successfully introduced for representing set variables in \citep{gervet1997interval,puget1993set} and graph variables in \citep{dooms2005cp}.
These two domains rely on the concept of subset-bounds, which maintain a lower bound of mandatory elements and an upper bound of potential elements in the set or graph.
In \cite{pesant1997genius}, the authors propose an extension of the successor model by introducing, for each visit $i$, two sets: $\mathcal{A}_i$ and $\mathcal{B}_i$, representing respectively the set of (not necessarily direct) predecessors and successors in the tour. These sets are maintained through channeling constraints with the successor and predecessor variables.
This formulation makes it possible to add precedence constraints to the model and, by introducing them dynamically, one can also implement insertion-based search strategies. However, as with the \emph{insertion variables} $\pi$ of \cite{kilby2000}, this approach does not provide a well-defined sequence variable with a computation domain that can be directly exploited by external constraints.

The closest related work to the insertion-based sequence variable domain presented in this paper is the head-tail sequence variable explained in the previous section and introduced in \citep{ibm_ilog_cp,Reasoning_with_Conditional_Time_Intervals_2}.
A similar approach is implemented in OR-Tools by Google \citep{ortools,cuvelier2023or}, which also provides sequence variable modeling. 
While primarily targeted at scheduling problems, these sequence variables have been successfully used to solve hybrid routing-and-scheduling problems in works such as \citep{liu2018modelling,10.1007/978-3-319-98334-9_32}. 
A related idea of growing prefix sequences was utilized for path variable representation in traffic engineering problems in the field of computer networks \citep{hartert2015solving}.

The approach proposed in this paper is more flexible, allowing elements to be inserted at arbitrary positions of the sequence. This insertion capability can be viewed as a domain implementation of the insertion graph concept introduced in \cite{bent2004two}. 
This domain representation was previously introduced in \citep{thomas2020insertion}. 
A simplified variant without a set of required nodes is detailed in \citep{delecluse2022sequence}.
This paper builds upon and refines these earlier works, providing a cleaner and more detailed version. 
It formally defines the computation domain and its semantics, and introduces the notion of consistency.

\section{Sequence Variable}
\label{section:sequence_variables}
We first introduce some notation before defining the domain of sequence variables.
\paragraph{Notations} 
A sequence $\seq{S}$ is defined as an ordered set of nodes belonging to a graph, without repetition. 
Let $\seq{S}$ be a sequence with the form $\seq{S} = \seq{S}_1 \cdot v_1 \cdot \seq{S}_2$ ($\seq{S}_1$ and $\seq{S}_2$ being sequences, possibly empty). 
An insertion operation $\singleDerives{(v_1, v_2)}$ defined by the pair of nodes $(v_1, v_2)$ produces a new sequence $\seq{S}' = \seq{S}_1 \cdot v_1 \cdot v_2 \cdot \seq{S}_2$. 
We denote with $\seq{S} \subset \seq{S}'$ that the sequence $\seq{S}$ is a subsequence of $\seq{S}'$, where $\seq{S}'$ preserves the order from $\seq{S}$ and has strictly more nodes.
Then $\seq{S}'$ is called a supersequence of $\seq{S}$.
If the two sequences may be the same, the relation is written $\seq{S} \subseteq \seq{S}'$.
We denote $v_i \followedIn{\seq{S}} v_j$ to indicate that $v_i$ directly precedes $v_j$ in the sequence $\seq{S}$ and $v_i \precIn{\seq{S}} v_j$ when $v_i$ precedes (not necessarily directly) $v_j$ in $\seq{S}$.
Those relations are simply written $v_i \followedIn{} v_j$ and $v_i \precIn{} v_j$ when clear from the context.  
If the nodes may be the same, the relation is written $v_i \preceq v_j$.
Given a sequence $\seq{S} = v_1 \ldots v_i \ldots v_n $, $\mathrm{prefix}(\seq{S},v_i) = v_1 \ldots v_i$ and $\mathrm{suffix}(\seq{S},v_i) = v_i \ldots v_n$.
Lastly, given a node set $V$, a start node $\first \in V$ and an end node $\last \in V$ , $\mathcal{P}(V)$ denotes all sequences $\seq{S}$ over a subset of nodes $V$, starting at node $\first \in V$ and ending at node $\last \in V$.

\paragraph{Domain}

A sequence domain denoted $\extDom \subseteq \mathcal{P}(V)$ contains a set of sequences over a subset of the nodes $V$ without repetition, starting at node $\first \in V$ and ending at node $\last \in V$.
Four elementary restrictions can be imposed on a sequence domain:
\begin{enumerate}
	\item \textbf{Require} a node to be visited, without explicitly stating the position of the node.
	\item \textbf{Exclude} a node from being visited.
	\item \textbf{Enforce a subsequence} $\seq{S}'$ to be included (i.e., enforce $\seq{S}' \subseteq \seq{S}$).
	%\item \textbf{Forbid the subsequences} $v_1 \cdot v_2 \cdot v_3$ of length 3 within the sequence, regardless of any additional nodes that may be present between them.
	\item \textbf{NotBetween}, which forbids a node $v_2$ to be placed between two other nodes $v_1$ and $v_3$ regardless of any additional nodes that may be present between them.
\end{enumerate}
Those restrictions were chosen to enable branching decisions, to allow for large neighborhood searches, and to facilitate domain filtering through constraints.

To define the sequence domain $\extDom$, we first define one subdomain for each of the four possible restrictions:

\begin{enumerate}
	\item $\extDomR{R}$ is specified by a set of required nodes $R \subseteq V$. 
	It denotes all sequences including all nodes in $R$.
	More formally,
	\begin{equation}
		\extDomR{R} = \left\{ \seq{S} \in \mathcal{P}(V) \mid \forall v \in R : v \in \seq{S} \right\}
	\end{equation}
	
	\item $\extDomX{X}$ is specified by a set of excluded nodes $X \subseteq V$.
	It denotes all sequences where no node in $X$ is included.
	More formally, 
	\begin{equation}
		\extDomX{X} = \left\{ \seq{S} \in \mathcal{P}(V) \mid \forall v \in X : v \notin \seq{S} \right\}
	\end{equation}

	\item $\extDomS{\seq{s}}$ is specified by a partial sequence of nodes $\seq{s}$. 
	It denotes all sequences $\seq{S}$ such that $\seq{s}$ is a subsequence of $\seq{S}$. 
	More formally, 
	\begin{equation}
		\extDomS{\seq{s}} = \left\{ \seq{S} \in \mathcal{P}(V) \mid \seq{s} \subseteq \seq{S} \right\}
	\end{equation}
	
	\item $\extDomF{\NB}$ is specified by a set of NotBetween triples $\NB \subseteq V \times V \times V$.
	$\extDomF{\NB}$ denotes all sequences in which no NotBetween triple from $\NB$ appears.
	More formally,
	\begin{equation}
		\extDomF{\NB} = \left\{ \seq{S} \in \mathcal{P}(V) \mid \forall (v_i \cdot v_j \cdot v_k) \in \NB : (v_i \cdot v_j \cdot v_k) \not \subseteq \seq{S} \right\}
	\end{equation}
\end{enumerate}

The intersection of those four subdomains defines the sequence domain.

\begin{definition}
	A sequence domain is defined as 
    \begin{equation}
		\mathcal{D}(R, X, \seq{s}, \NB) = \extDomR{R} \cap \extDomX{X} \cap \extDomS{\seq{s}} \cap \extDomF{\NB}
	\end{equation}
\end{definition}

\begin{definition}
	$\mathcal{D}(R, X, \seq{s}, \NB)$ is said to be fixed when $|\mathcal{D}(R, X, \seq{s}, \NB)|=1$.
\end{definition}

By abuse of notation, $\mathcal{D}(R, X, \seq{s}, \NB)$ will be simply written $\extDom$ in the following.

\begin{example}
    \label{ex:domain-definition}
    $ V= \{\first,v_1,v_2,v_3,\last\}$, $R = \{ \first, \last, v_2\}$, $X = \{ v_3 \}$, $\seq{s} = \first \cdot v_1 \cdot \last$, $\NB = \{ (\first \cdot v_2 \cdot v_1), (v_3 \cdot v_1 \cdot v_2) \}$.
    $\extDom = \extDomR{R} \cap \extDomX{X} \cap \extDomS{\seq{s}} \cap \extDomF{\NB} = \{ (\first \cdot v_1 \cdot v_2 \cdot \last) \}$.
    The compositions of the subdomains are represented in Table \ref{tab:sequence-domain-breakdown}. 
    Given that $|\extDom|=1$, the domain is fixed.
    \begin{table}[!ht]
    \centering
    \caption{The compositions of subdomains in Example \ref{ex:domain-definition}. A check mark in one of the last 4 columns indicates that the corresponding sequence in the first column is included within the subdomain.}
    \label{tab:sequence-domain-breakdown}
    \begin{small}
    \setlength{\tabcolsep}{4pt}
    \begin{tabular}{@{}rcccc@{}}
    \toprule
    $\mathcal{P}(V)$ & $\extDomR{R}$ & $\extDomX{X}$ & $\extDomS{\seq{s}}$ & $\extDomF{\NB}$ \\
    \midrule
    $\first \cdot \last$                               &            & \checkmark &            & \checkmark \\
    $\first \cdot v_1 \cdot \last$                     &            & \checkmark & \checkmark & \checkmark \\
    $\first \cdot v_2 \cdot \last$                     & \checkmark & \checkmark &            & \checkmark \\
    $\first \cdot v_3 \cdot \last$                     &            &            &            & \checkmark \\
    $\bm{\first \cdot v_1 \cdot v_2 \cdot \last}$      & $\bm{\checkmark}$ & $\bm{\checkmark}$ & $\bm{\checkmark}$ & $\bm{\checkmark}$ \\
    $\first \cdot v_1 \cdot v_3 \cdot \last$           &            &            & \checkmark & \checkmark \\
    $\first \cdot v_2 \cdot v_1 \cdot \last$           & \checkmark & \checkmark & \checkmark &            \\
    $\first \cdot v_2 \cdot v_3 \cdot \last$           & \checkmark &            &            & \checkmark \\
    \bottomrule
    \end{tabular}
    \hspace{0.015\linewidth}%
    \begin{tabular}{@{}rcccc@{}}
    \toprule
    $\mathcal{P}(V)$ & $\extDomR{R}$ & $\extDomX{X}$ & $\extDomS{\seq{s}}$ & $\extDomF{\NB}$ \\
    \midrule
    $\first \cdot v_3 \cdot v_1 \cdot \last$           &            &            & \checkmark & \checkmark \\
    $\first \cdot v_3 \cdot v_2 \cdot \last$           & \checkmark &            &            & \checkmark \\
    $\first \cdot v_1 \cdot v_2 \cdot v_3 \cdot \last$ & \checkmark &            & \checkmark & \checkmark \\
    $\first \cdot v_1 \cdot v_3 \cdot v_2 \cdot \last$ & \checkmark &            & \checkmark & \checkmark \\
    $\first \cdot v_2 \cdot v_1 \cdot v_3 \cdot \last$ & \checkmark &            & \checkmark &            \\
    $\first \cdot v_2 \cdot v_3 \cdot v_1 \cdot \last$ & \checkmark &            & \checkmark &            \\
    $\first \cdot v_3 \cdot v_1 \cdot v_2 \cdot \last$ & \checkmark &            & \checkmark &            \\
    $\first \cdot v_3 \cdot v_2 \cdot v_1 \cdot \last$ & \checkmark &            & \checkmark &            \\
    \bottomrule
    \end{tabular}
    \end{small}
    \end{table}
\end{example}

Domain updates over $\extDom$ are defined by growing the sets $R$, $X$ and $\NB$, adding more elements to them, or by inserting a node within the partial sequence $\seq{s}$.
Although $\seq{s}$ statically defines the subdomain $\extDomS{\seq{s}}$, it can only be extended through insertions.
Let $v_i \followedIn{\seq{s}} v_k$ and $v_j \notin \seq{s}$.
Inserting $v_j$ after $v_i$ replaces $v_i \cdot v_k$ by $v_i \cdot v_j \cdot v_k$ in $\seq{s}$, thereby retaining precisely the sequences of the current domain satisfying $v_i \precIn{} v_j \land v_j \precIn{} v_k$.
This insertion does not fix the final successor of $v_j$, as later insertions may occur between $v_j$ and $v_k$.
Conversely, $\mathrm{notBetween}(v_i,v_j,v_k)$ retains precisely the sequences satisfying $\neg \left( v_i \precIn{} v_j \land v_j \precIn{} v_k \right)$, allowing $v_j$ to be inserted elsewhere or to be absent.
Thus, as with Require and Exclude, insertion and $\mathrm{notBetween}$ are complementary updates associated with the candidate insertion of $v_j$ between $v_i$ and $v_k$.
Notice that elementary restrictions can be combined to enforce a \emph{between} restriction, that is, forcing a node $v_2$ to be visited between nodes $v_1, v_3$.
This can be achieved by adding $(\first \cdot v_2 \cdot v_1), (v_3 \cdot v_2 \cdot \last)$ to the set of NotBetween $\NB$ given that any sequence $\seq{s} \in \extDom$ begins at $\first$ and ends at $\last$. 
Moreover, if node $v_2$ must be included in the sequences of the domain, $v_2$ can be added to the required nodes $R$.

Each sequence variable $\seq{S}$ is associated with a sequence domain, and represents an unknown sequence.
Such a variable is particularly convenient for modeling a route from an origin node $\first$ to a destination node $\last$ in a VRP, where the nodes visited and their ordering represents the path performed by a vehicle. 

We introduce a compact encoding of the domain $\extDom$ in the next section, which only consumes $\mathcal{O}(n^2)$ memory, with $n = |V|$, and enables efficient domain updates.

\section{An $O(n^2)$ encoding of the sequence variable domain}

This section introduces a compact domain representation for the entire domain $\extDom(R, X, \seq{s}, \NB)$ as a set of pairs of nodes, requiring only \(\mathcal{O}(n^2)\) memory, with the limitation that 
for each NotBetween $(v_1 \cdot v_2 \cdot v_3) \in \NB$, the two extremities $v_1, v_3$ are within the partial sequence \(\seq{s}\).
This representation exploits the representation of the intersection of the subdomains.
First, the representation of \(\extDomS{\seq{s}} \cap \extDomF{\NB}\) is introduced in Section \ref{sec:encoding-s-and-nb}. It is then extended with the intersection of the excluded nodes subdomain $\extDomX{X}$ in Section \ref{sec:encoding-x} and, finally, with the intersection of the required nodes subdomain $\extDomR{R}$ in Section \ref{sec:encoding-r}.

\subsection{Partial Sequence and NotBetween}
\label{sec:encoding-s-and-nb}
A naive representation of the set of NotBetween $\NB \subseteq V \times V \times V$ requires a cubic space complexity. 
This can be reduced to a quadratic one by representing directly $\extDomS{\seq{s}} \cap \extDomF{\NB}$ and by restricting $\NB$ to contain only NotBetween triples with the extremities belonging to the partial sequence $\seq{s}$ and in the order in which they appear in $\seq{s}$:
$\NB \subseteq \{ (v_1 \cdot v_2 \cdot v_3) \mid v_1 \precIn{\seq{s}} v_3 \}.$
%Notice that forbidden subsequences $(v_1 \cdot v_2 \cdot v_3)$ for which ($v_3 \preceq{} v_1$) do not need to be represented in $\NB$ as those are already not allowed in the domain by their presence in $\seq{s}$.
The set of NotBetween with this restriction is written $\NBR$.
The set $\NBR$ can then be translated into a set of pairs of nodes $(v_i, v_2)$ corresponding to \textit{forbidden insertions} of $v_2$ just after $v_i$ in $\seq{S}$, due to a NotBetween:
\begin{equation}
    \forbiddenInsertions{\seq{s}, \NBR} = \{ (v_i, v_2) \mid (v_1 \cdot v_2 \cdot v_3) \in \NBR \land v_1 \preceqIn{} v_i \land v_i \precIn{} v_3 \} \label{eq:forbidden-insertions}
\end{equation}

\begin{theorem}
\label{theorem:encoding}
The forbidden insertions $\forbiddenInsertions{\seq{s}, \NBR}$ can be used instead of $\NBR$ in the representation of the domain $\extDomS{\seq{s}} \cap \extDomF{\NBR}$. 
\end{theorem}
\begin{proof}
One can enumerate all supersequences $\seq{s}'$ of $\seq{s}$ such that no forbidden pair of $\forbiddenInsertions{\seq{s}, \NBR} $ appears in that order in $\seq{s}'$:
%The domain implicitly represents all super sequences $\seq{s}'$ of $\seq{s}$, such that no forbidden pair of $\forbiddenInsertions{\seq{s}, \NBR} $ appears in that order in $\seq{s}'$:
\begin{equation}
    \extDomS{\seq{s}} \cap \extDomF{\NBR} = \{ \seq{s}' \mid (\seq{s} \subseteq \seq{s}') \land \left( \forall v \in (\seq{s}' \setminus \seq{s}) : (\mathrm{prev}(v, \seq{s}, \seq{s}'), v) \notin \forbiddenInsertions{\seq{s}, \NBR} \right) \} \label{eq:domain-intersection-definition}
\end{equation}
where $\mathrm{prev}(v, \seq{s}, \seq{s}')$ denotes the first node preceding $v$ in $\seq{s}'$ and also appearing in $\seq{s}$: 
\begin{align}
    \mathrm{prev}(v, \seq{s}, \seq{s}') &= \begin{cases}
        v & \text{ if } v \in \seq{s} \\
        \mathrm{prev}(v_i, \seq{s}, \seq{s}') \mid v_i \followedIn{\seq{s}'} v & \text{ otherwise }.
    \end{cases} \label{eq:prev-recursive-definition}
\end{align}
\end{proof}

Thanks to Theorem~\ref{theorem:encoding}, $\extDomS{\seq{s}} \cap \extDomF{\NBR}$ can be represented solely with the partial sequence $\seq{s}$ and the pairs of nodes $\forbiddenInsertions{\seq{s}, \NBR}$, computed in $\mathcal{O}(n^2)$ time and also consuming $\mathcal{O}(n^2)$ in memory.

\subsubsection*{Domain updates}
Two domain updates are possible on $\extDomS{\seq{s}} \cap \extDomF{\NBR}$: insertion of a node within $\seq{s}$, or adding a new NotBetween triple to $\NBR$.
Two possible strategies to update $\forbiddenInsertions{\seq{s}, \NBR}$ under these modifications are
\begin{enumerate}
    \item Recompute $\forbiddenInsertions{\seq{s}, \NBR}$, which still requires maintaining $\NBR$ and requires $\mathcal{O}(n^3)$ memory.
    \item Store the set $\forbiddenInsertions{\seq{s}, \NBR}$, consuming $\mathcal{O}(n^2)$ in memory, and incrementally update it.
\end{enumerate}
We follow the more efficient second strategy and show how to perform such updates in $\mathcal{O}(n)$ time while maintaining an $\mathcal{O}(n^2)$ memory bound:
\begin{itemize}
    \item If a new NotBetween $(v_1 \cdot v_2 \cdot v_3)$ is added to $\NBR$, the newly forbidden insertions are:
    \begin{equation}
        \Delta \forbiddenInsertions{\seq{s}, (v_1 \cdot v_2 \cdot v_3)} = \{ (v_i, v_2) \mid v_1 \preceqIn{} v_i \land v_i \precIn{} v_3\} \label{eq:forbidden-insertions-update-forbidden}
    \end{equation}
    This uses the same rule as \eqref{eq:forbidden-insertions}.
    At most $\mathcal{O}(|\seq{s}|) = \mathcal{O}(n)$ pairs can be added in this manner.
    \item If the sequence $\seq{s}$ is expanded into $\seq{s}'$ through an insertion $\seq{s} \singleDerives{(v_1, v_2)} \seq{s}'$, the newly forbidden insertions are:
    \begin{equation}
        \Delta \forbiddenInsertions{\seq{s}, (v_1, v_2)} = \{ (v_2, v_i) \mid (v_1, v_i) \in \forbiddenInsertions{\seq{s}, \NBR} \} \label{eq:forbidden-insertions-update-insertion}
    \end{equation}
    This means that a node $v_i$ that could not be placed directly after $v_1$ now cannot be placed directly after $v_2$ either.
    The set \eqref{eq:forbidden-insertions-update-insertion} can be created in $\mathcal{O}(n)$ time and contains $\mathcal{O}(n)$ pairs.
\end{itemize}

Whenever a NotBetween $(v_1 \cdot v_2 \cdot v_3)$ is added to $\NBR$ while it is a subsequence of the partial sequence ($(v_1 \cdot v_2 \cdot v_3) \subseteq \seq{s}$), it results in a domain wipeout.
Similarly, expanding $\seq{s}$ into $\seq{s}'$ using a forbidden insertion $(v_1, v_2) \in \forbiddenInsertions{\seq{s}, \NBR}$ (i.e., performing $\seq{s} \singleDerives{(v_1, v_2)} \seq{s}'$) results in a domain wipeout.
Next, we show how to compute the intersection with the remaining subdomains.

\subsection{Excluded Nodes}
\label{sec:encoding-x}
The subdomain $\extDomS{\seq{s}} \cap \extDomF{\NBR}$ presented previously only accounted for a partial sequence $\seq{s}$ and a set of NotBetween constraints $\NBR$.
Excluding a node $v$ from the sequence 
can be reduced to adding the NotBetween $(\first \cdot v \cdot \last)$ to $\NBR$, given that all sequences begin at $\first$ and end at $\last$. Therefore, for a set of excluded nodes $X$, we have:
\begin{equation}
     \forall v \in X: (\first \cdot v \cdot \last) \in \NBR \label{eq:domain-with-excluded}
\end{equation}
When an update $X \gets X \cup \{ v \}$ occurs,
it suffices to add the NotBetween $(\first, v, \last)$, and use \eqref{eq:forbidden-insertions-update-forbidden} to trigger the incremental update of \(\forbiddenInsertions{\seq{s}, \NBR}\).
Thus, it is not mandatory to store the excluded nodes $X$.

\subsection{Required Nodes}
\label{sec:encoding-r}
To take into account a set of required nodes $R$, we have no option other than to store them. This requires an additional $\mathcal{O}(n)$ memory.
The entire domain $\extDom$ can be represented in $\mathcal{O}(n^2)$ with i) the set of pairs of nodes $ \forbiddenInsertions{\seq{s}, \NBR}$, ii) the partial sequence $\seq{s}$, and iii) the set of required nodes $R$.
Together, those elements implicitly represent all supersequences $\seq{s}'$ of $\seq{s}$ such that no forbidden pair of $\forbiddenInsertions{\seq{s}, \NBR} $ appears in that order in $s'$ and all the required nodes are present in $\seq{s}'$:
\begin{equation}
    \extDom = \{ \seq{s}' \mid (\seq{s} \subseteq \seq{s}') \land \left( \forall v \in (\seq{s}' \setminus \seq{s}) : (\mathrm{prev}(v, \seq{s}, \seq{s}'), v) \notin \forbiddenInsertions{\seq{s}, \NBR} \right) \land \left( \forall v \in R: v \in \seq{s}' \right) \}
\label{eq:domain-intersection-definition-with-required}
\end{equation}

The next section proposes an implementation of these elements using a graph data structure.

\section{Compact domain implementation}

As explained in the previous section, a domain implementation requires representing the forbidden insertions \(\forbiddenInsertions{\seq{s}, \NBR}\), the partial sequence \(\seq{s}\), and the required nodes \(R\).

Rather than representing \(\forbiddenInsertions{\seq{s}, \NBR}\), we instead represent the complementary set \((V \times V) \setminus \forbiddenInsertions{\seq{s}, \NBR} \setminus \{(v_i,v_j) \mid v_i,v_j \in \seq{s} \wedge v_i \notFollowedIn{} v_j\}\) as a graph \(G(V,E)\). 
Each directed edge \((v_i,v_j) \in E\) with \(v_i \in \seq{s}\) represents a possible insertion of node \(v_j\) in \(\seq{s}\) just after \(v_i\) in the partial sequence \(\seq{s}\).
Notice that every pair of nodes in the partial sequence that do not directly follow each other are also removed from the set of edges.
As the partial sequence \(\seq{s}\) grows and the set of NotBetween constraints \(\NBR\) grows, the set of edges \(E\) can only shrink monotonically.
Eventually, when the variable is fixed, the set of edges forms a path that corresponds to the partial sequence \(\seq{s}\).
This complementary view enables the enumeration of all feasible insertions of a node into the partial sequence \(\seq{s}\). It also allows for counting feasible insertions for each node in the partial sequence. These are two common operations used by filtering algorithms and for implementing first-fail branching heuristics.

\subsection{A Reversible Data Structure}

The main mechanism for restoring the state of domains and constraints on backtracking in a CP solver is called trailing (we refer to \cite{minicp} for an introduction to trailing in CP solvers).
To be integrated into a standard trail-based CP solver, the compact domain implementation also needs to be \emph{reversible}, meaning that its state is restored when the search backtracks. In CP, this is typically achieved using a \emph{trail}, a data structure that records changes to variables. A \emph{reversible integer} is an integer whose modifications are automatically recorded on the trail for restoration.
The graph $G(V,E)$ is implemented as an adjacency list where each list is a reversible set of nodes and the partial sequence $\seq{s}$ uses a reversible successor array.
In more detail, the implementation relies on the following reversible data structures, most of which were already present in \citep{thomas2020insertion, delecluse2022sequence}:
\begin{itemize}
	\item The partial sequence $\lbSeq$ is encoded using a successor array of reversible integers $\directSuccessor{}$.
	It stores a pointer to the current successor of each node belonging to the partial sequence $\lbSeq$, and an element without a successor (i.e., a node $v \in V \setminus \lbSeq$) points towards itself (self-loop).
	Similarly, an additional array of reversible integers \(\directPredecessor{}\) tracks the current predecessors of each node.
	\item The set of edges $E$ is maintained by two adjacency sets per node $v \in V$: $\predecessors{v}$ (incoming) and $\successors{v}$ (outgoing) edges: $\predecessors{v} = \{ (v_i, v) \mid v_i \in V, (v_i, v) \in E\}$, $\successors{v} = \{ (v, v_i) \mid v_i \in V, (v, v_i) \in E\}$.
	Those sets are implemented using the reversible sparse-sets introduced in \cite{sparsetdomain}, allowing deletion and state restoration in constant time in a trail-based solver.
    \item A set \(I\), tracking the insertable nodes \(\{ v \in V \mid v \notin \seq{s} \land v \notin X \}\), is represented as a reversible sparse set.
	\item The size of the partial sequence \(\lbSeq\) is tracked by a reversible integer \(nS\). 
	\item Each \(v_j\in V\) has a counter \(nI_j\) tracking insertions, maintained under domain changes. Heuristics can use it, for example, to select the node with the fewest insertion points.
	\item Given that \(R\) and \(X\) are disjoint subsets of \(V\), we use a sparse-set with two reversible size markers as introduced in \cite{sparsetdomain}, ensuring removal of nodes and state restoration in constant time, and enabling iteration over \(R, X\) in \(\mathcal{O}(|R|), \mathcal{O}(|X|)\), respectively.
    We also define \(P = V \setminus X \setminus R\) as the set of \textit{possible} nodes, which are neither required nor excluded.
\end{itemize}

Those datastructures are illustrated in Figure \ref{fig:weak-seqvar-representation}.

\begin{figure}[!ht]
	\centering
	\begin{minipage}{0.47\linewidth}
		\centering
		\input{figures/weak-seqvar-representation}
	\end{minipage}
	\hfill
	\begin{minipage}{0.47\linewidth}
		\centering
		\input{figures/weak-seqvar-representation-after-insert}
	\end{minipage}
	\caption{Compact domain implementation. On the left, the partial sequence $\lbSeq$ and the graph $G(V,E)$ are shown. Below them is a table showing the edges $\predecessors{}, \successors{}$, the insertion counters $nI$, and the successors $\directSuccessor{}$ and predecessors $\directPredecessor{}$ of the nodes (only relevant for nodes $v \in \lbSeq$). The right-hand part shows the domain after performing an insertion with $(\first, v_3)$, extending the partial sequence.}
	\label{fig:weak-seqvar-representation}
\end{figure}

Invariants on the domain are presented in the Appendix \ref{sec:implementation-invariants} and the implementation of the update operations in Appendix \ref{sec:implementation-updates}.
Some coherence is enforced on the domain representation.
For instance, given that all nodes in the partial sequence will always be visited, we enforce \(\seq{s} \subseteq R\).
We also exploit the insertion counter \(nI_j\) of node \(v_j \in V\) to automatically insert a node when it is required and has only one insertion remaining (similarly to \citep{kilby2000}), or to exclude \(v_j\) from the sequence when no insertion remains for it.
This process is also detailed in Appendix \ref{sec:implementation-updates}.

\subsection{API}
\label{seq:weak-seqvar-time-complexity}

A sequence domain is fixed whenever, for each node that does not belong to the partial sequence, no insertion remains.
The worst-case time complexity for constructing a sequence is thus $O(|E|)$.
In the worst case, all edges are removed except those that form the sequence.

Table~\ref{tab:weak-seqvar-complexity-query} lists all query operations in a sequence variable domain, along with their associated time complexities. 
Domain updates are described in Table~\ref{tab:weak-seqvar-complexity-update}.

\begin{table}[!ht]
    \centering
    \caption{Queries on a compact sequence domain.}
    \label{tab:weak-seqvar-complexity-query}
    \begin{tabular}{llc}
		\toprule
		\textbf{Operation} & \textbf{Description} & \textbf{Complexity}\\
		\midrule
		$\mathrm{isFixed}()$             & Returns true if there are no remaining insertions                      & $\mathcal{O}(1)$\\
		$\mathrm{isMember}(v_i)$         & Returns true if $v_i\in\lbSeq{}$                                              & $\mathcal{O}(1)$\\
		$\mathrm{isRequired}(v_i)$       & Returns true if the node $v_i$ is required                                            & $\mathcal{O}(1)$\\
		$\mathrm{isExcluded}(v_i)$       & Returns true if the node $v_i$ is excluded                                            & $\mathcal{O}(1)$\\
		$\mathrm{isPossible}(v_i)$       & Returns true if the node $v_i$ is possible                                            & $\mathcal{O}(1)$\\
		$\mathrm{isInsertable}(v_i)$     & Returns true if the node $v_i$ is insertable                                          & $\mathcal{O}(1)$\\
		
		$\mathrm{getNext}(v_i)$          & Returns the successor $\directSuccessor{i}$ of node $v_i$                                            & $\mathcal{O}(1)$\\
		$\mathrm{getPrev}(v_i)$          & Returns the predecessor $\directPredecessor{i}$ of node $v_i$                                               & $\mathcal{O}(1)$\\
		
		\cmidrule(lr){1-3}
		
		$\mathrm{nInsert}(v_i)$          & Returns the number of feasible insertions for $v_i$                                       & $\mathcal{O}(1)$\\
		$\mathrm{nMember}()$             & Returns the length of the partial sequence                    & $\mathcal{O}(1)$\\
		
		$\mathrm{getMember}()$       & Enumerates nodes in the partial sequence $\lbSeq{}$                        & $\Theta(|\lbSeq{}|)$\\
		$\mathrm{getRequired}()$         & Enumerates the required nodes $R$                                              & $\Theta(|R|)$\\
		$\mathrm{getExcluded}()$         & Enumerates the excluded nodes $X$                                              & $\Theta(|X|)$\\
		$\mathrm{getPossible}()$         & Enumerates the possible nodes $P$                                              & $\Theta(|P|)$\\
		$\mathrm{getInsertable}()$       & Enumerates the insertable nodes $I$                                            & $\Theta(|I|)$\\
		
		\cmidrule(lr){1-3}

		$\mathrm{getEdgesTo}(v_i)$       & Enumerates $\predecessors{i}$                                                  & $\Theta(|\predecessors{i}|)$\\
		$\mathrm{getEdgesFrom}(v_i)$     & Enumerates $\successors{i}$                                                    & $\Theta(|\successors{i}|)$\\
		
		\cmidrule(lr){1-3}

		$\mathrm{canInsert}(v_i,v_j)$    & Returns true if $(v_i, v_j)$ is a feasible insertion & $\mathcal{O}(1)$\\
		$\mathrm{getInsert}(v_j)$        & Enumerates the insertions of $v_j$: $\{\,v_i \mid v_i\in\predecessors{j}\land\mathrm{canInsert}(v_i,v_j)\}$ & 
		$\Theta\!(\min(|\lbSeq{}|,|\predecessors{j}|))$\\
		
		$\mathrm{getInsertAfter}(v_j,p)$ & Enumerates the insertions of $v_j$ after $p$: $\{\,v_i \mid p\precIn{} v_i \land \mathrm{canInsert}(v_i,v_j)\}$ & $\mathcal{O}(|\lbSeq{}|)$\\
        \bottomrule
    \end{tabular}
\end{table}

\begin{table}[!ht]
    \centering
    \caption{Updates on a compact sequence domain.}
    \label{tab:weak-seqvar-complexity-update}
    \begin{tabular}{lccc}
			\toprule
			\textbf{Operation} & \textbf{Update} & \textbf{Precondition} & \textbf{Complexity}\\
			\midrule
			$\mathrm{notBetween}(v_i,v_j,v_k)$
			& $\NB \gets \NB \cup {(v_i \cdot v_j \cdot v_k)}$
			& $v_i, v_k \in \lbSeq{}$
			& $\begin{cases}
				\mathcal{O}(|I|) & \text{if } v_i \followedIn{} v_k\\
				\mathcal{O}(n)                        & \text{otherwise}
			\end{cases}$\\[2pt]
			
			$\mathrm{insert}(v_i,v_j)$
			& $\singleDerives{(v_i,v_j)}$
			& $v_i \in \lbSeq{}$
			& $\Theta(|\predecessors{j}|)$\\
			
			$\mathrm{insertAtEnd}(v_i)$
			& \hspace{0.05cm}$\singleDerives{(\mathrm{getPrev}(\last),v_i)}$
			&
			& $\Theta(|\predecessors{i}|)$\\
			
			$\mathrm{require}(v_i)$
			& $R \gets R \cup \{\,v_i\,\}$
			&
			& $\mathcal{O}(|\predecessors{i}|)$\\
			
			$\mathrm{exclude}(v_i)$
			& $X \gets X \cup \{\,v_i\,\}$
			&
			& $\mathcal{O}(|\predecessors{i}|)$\\
			\bottomrule
    \end{tabular}
\end{table}

\subsection{Note on Monotonicity of the Domain}

The domain of a sequence variable, as defined in Section~\ref{section:sequence_variables}, is partially ordered by the inclusion of its components (partial sequence $\seq{s}$, forbidden insertions $\NBR$, and required/excluded sets $R, X$). A domain $\mathcal{D}_1$ is considered "smaller" than or equal to $\mathcal{D}_2$ if it is more restricted: $\seq{s}_1$ contains $\seq{s}_2$ as a subsequence, $\NBR_1 \supseteq \NBR_2$, $R_1 \supseteq R_2$, and $X_1 \supseteq X_2$. This partial ordering, combined with the fact that domains can only shrink during search, ensures that the domain representation forms a lattice.

In CP, the propagation process iteratively applies filtering algorithms until no further domain reductions are possible, reaching a \emph{fix-point}. For this fix-point to be unique and independent of the execution order of the propagators, the filtering algorithms must be \emph{monotonic}: if $\mathcal{D}_1 \subseteq \mathcal{D}_2$, then $f(\mathcal{D}_1) \subseteq f(\mathcal{D}_2)$ for any propagator $f$. All domain update operations provided by our API (Table~\ref{tab:weak-seqvar-complexity-update}) are monotonic. Furthermore, the internal coherence rules—such as automatically inserting a required node when only one insertion remains—are also monotonic. Provided that the global constraints implemented over sequence variables are monotonic, the propagation will always converge to a unique greatest fix-point.

%While individual propagators are not strictly required to be \emph{idempotent} (i.e., $f(f(\mathcal{D})) = f(\mathcal{D})$) to guarantee a unique fix-point, idempotency is highly desirable for computational efficiency as it avoids redundant calls to the same filtering algorithm within the same propagation cycle. Most filtering algorithms presented in this work, including those for the \emph{Distance} and \emph{TransitionTimes} constraints, are designed to be idempotent.

\subsection{Visit of nodes as boolean variables}
\label{sec:domain-enriching}

Given the set of mandatory nodes $R$, a binary variable $\nodeRequired{\seq{S}}{i}$ can be associated with each node $v_i \in V$. It indicates whether $v_i$ is visited by a sequence variable $\seq{S}$ with a compact domain: fixing it to 1 requires $v_i$, whereas fixing it to 0 excludes $v_i$.
In CP, this can be implemented as a \textit{view} over the compact sequence domain, making the use of such variables inexpensive once a sequence variable has been created, as in \citep{schulte2013view, van2014domain, minicp}.
The correspondence between operations on those Boolean variables and a sequence domain is presented in Table \ref{tab:boolean-variable-operation}.
Fixing a Boolean variable $\nodeRequired{\seq{S}}{i}$ to 1 may automatically insert the node $v_i$ into the partial sequence if only one insertion point remains for it, as explained in Appendix \ref{sec:implementation-updates}.

Since node membership in a sequence variable is represented by Boolean variables, standard logical constraints can be imposed directly on sequence variables.
For instance, \(\sum_{v_i \in V} \nodeRequired{\seq{S}}{i} \geq n\) ensures that $\seq{S}$ visits at least $n$ nodes, whereas \(\nodeRequired{\seq{S}}{i} = \nodeRequired{\seq{S}}{j}\) ensures that nodes $v_i$ and $v_j$ are either both visited or both excluded.
The following section introduces more complex constraints with dedicated propagators.

\begin{table}[!ht]
  \centering
  \caption{Correspondence between operations on a boolean variable $\nodeRequired{Sq}{i}$ defined over a node $v_i$ in a sequence variable $\seq{S}$. The left part shows queries over the domains, while the right part shows updates of the domains.}
  \label{tab:boolean-variable-operation}
  \begin{minipage}[t]{0.48\linewidth}\vspace{0pt}
      \centering
      \begin{tabular}[t]{@{}lll@{}}
        \toprule
        \textbf{Boolean variable} & & \textbf{Sequence variable} \\
        \midrule
        $|\mathcal{D}(\nodeRequired{\seq{S}}{i})| = 1 $ & $\Longleftrightarrow$ & $\neg \seq{S}.\mathrm{isPossible}(v_i)$ \\
        $\textit{false} \in \mathcal{D}(\nodeRequired{\seq{S}}{i})$ & $\Longleftrightarrow$ & $\neg \seq{S}.\mathrm{isRequired}(v_i)$ \\
        $\textit{true} \in \mathcal{D}(\nodeRequired{\seq{S}}{i})$ & $\Longleftrightarrow$ & $\neg \seq{S}.\mathrm{isExcluded}(v_i)$ \\
        \bottomrule
      \end{tabular}
    \end{minipage}\hfill
    \begin{minipage}[t]{0.48\linewidth}\vspace{0pt}
      \centering
      \begin{tabular}[t]{@{}lll@{}}
        \toprule
        \textbf{Boolean variable} & & \textbf{Sequence variable} \\
        \midrule
        $\mathcal{D}(\nodeRequired{\seq{S}}{i}) \gets \{\textit{true}\}$ & $\Longleftrightarrow$ & $\seq{S}.\mathrm{require}(v_i)$ \\
        $\mathcal{D}(\nodeRequired{\seq{S}}{i}) \gets \{\textit{false}\}$ & $\Longleftrightarrow$ & $\seq{S}.\mathrm{exclude}(v_i)$ \\
        \bottomrule
      \end{tabular}
    \end{minipage}
\end{table}

\section{Global Constraints}\label{sec:global_constraints}
\label{section:global}

Many useful global constraints and their filtering algorithms can be defined on sequence variables. To limit the scope of this paper, we focus on those that are required by a wide range of Vehicle Routing Problem (VRP) applications, such as the Dial-a-Ride Problem. Properly characterizing the amount of filtering that can be achieved is of great theoretical interest. Similar to the notions of bound-consistency and domain consistency for global constraints involving integer variables (see \citep{minicp}), various consistency levels can be defined for sequence variables.

\paragraph{Consistency of a constraint on a sequence variable.}

\begin{definition}
    A constraint $\mathcal{C}$ over a sequence domain $\extDom(R, X, \seq{s}, \NB)$ is \emph{insert consistent}, if and only if for every remaining insertion \((v_1,v_2)\) we have \(\extDom(R, X, \seq{s}', \NB) \cap \mathcal{C} \neq \emptyset\) where \(\seq{s} \singleDerives{(v_1, v_2)} \seq{s}'\).
\end{definition}
\begin{definition}
    Given a sequence domain \(\extDom(R, X, \seq{s}, \NB)\), we define its relaxed sequence domain as \(\extDom(\{v \mid v \in \seq{s}\}, X, \seq{s}, \NB)\), which considers that only nodes in the partial sequence $\seq{s}$ are required.
\end{definition}
\begin{definition}
    A constraint \(\mathcal{C}\) over a sequence domain \(\extDom(R, X, \seq{s}, \NB)\) is \emph{relaxed insert consistent}, if and only if for every remaining insertion \((v_1,v_2)\) on \(\extDom\) we have \(\extDom(\{v \mid v \in \seq{s}'\}, X, \seq{s}', \NB) \cap \mathcal{C} \neq \emptyset\) where \(\seq{s} \singleDerives{(v_1, v_2)} \seq{s}'\).
\end{definition}

Given that $\{v \mid v \in \seq{s}\} \subseteq R$, the \emph{relaxed-insert consistency} property for a constraint $\mathcal{C}$ is thus a relaxed form of \emph{insert consistency}.
The filtering algorithms that we introduce are relatively basic and aim to reach \emph{relaxed-insert consistency} rather than \emph{insert consistency}, which may be NP-hard to reach in polynomial time for some constraints.

\subsection{Distance}

The $\mathrm{Distance}$ constraint is used to represent the travel length in a sequence of nodes. It links an integer variable $\cpvar{Dist}$ to the traveled distance between nodes in a sequence variable, through transitions defined in a matrix $\bm{d} \in \mathbb{Z}^{|V| \times |V|}$ satisfying the triangular inequality \eqref{eq:distance-definition}.

\begin{equation}
    \mathrm{Distance}(\seq{S},  \bm{d}, \cpvar{Dist}) \leftrightarrow \sum_{v_i \followedIn{\seq{S}} v_j} \bm{d}_{i, j} = \cpvar{Dist} \label{eq:distance-definition}
\end{equation}

\subsubsection*{Filtering} 
The filtering is presented in Algorithm \ref{alg:distance}.
First, it computes the traveled distance over the partial sequence $\lbSeq{}$ (line \ref{alg:distance:current-distance}).
If the sequence variable is fixed, this fixes the value of the integer variable $\cpvar{Dist}$ (line \ref{alg:distance:fixed}).
Otherwise, this is used to set its lower bound (line \ref{alg:distance:lower-bound}) and compute the length of the largest insertion still possible (line \ref{alg:distance:max-detour}).
All insertions for every insertable node $v_j \in I$ are then examined, and if the cost for inserting a node $v_j$ between two consecutive nodes $v_i$ and $v_k$ is too high, the insertion is removed (lines \ref{alg:distance:detour-cost} to \ref{alg:distance:remove-detour}).
The use of the $\mathrm{notBetween}$ operation here crucially relies on the triangle inequality. 
Since $\bm{d}_{i, j} + \bm{d}_{j, k} \ge \bm{d}_{i, k}$, any insertion (detour) can only increase the total distance. 
Without this property, removing an insertion because it currently exceeds the maximum distance would be invalid, as adding even more nodes later in a longer subsequence could theoretically result in a shorter overall path.
In the worst case, the filtering empties all the edges from $E$, except those in the current sequence $\lbSeq{}$, resulting in a time complexity of $\mathcal{O}(|E|)$.

\begin{algorithm}[!ht]
	\caption{$\mathrm{Distance}(\seq{S}, \bm{d}, \cpvar{Dist})$ constraint filtering.}
	\label{alg:distance}
	$\textit{length} \gets  \displaystyle \sum_{v_i \followedIn{\lbSeq} v_j} \bm{d}_{i, j}$ \label{alg:distance:current-distance}\;
	\eIf{$\seq{S}.\mathrm{isFixed}()$}{
		$\cpvar{Dist} \gets \textit{length} \label{alg:distance:fixed}$ \;    
	}{
		$\lfloor \cpvar{Dist} \rfloor \gets \max(\textit{length}, \lfloor \cpvar{Dist} \rfloor)$ \label{alg:distance:lower-bound}\;
		$\textit{maxDetour} \gets \lceil \cpvar{Dist} \rceil - \textit{length} $ \label{alg:distance:max-detour}\;
		\For{$v_j \in \seq{S}.\mathrm{getInsertable}()$}{
			\For{$v_i \in \seq{S}.\mathrm{getInsert}(v_j)$ \label{alg:distance:nested-loop}}{
				$v_k \gets \seq{S}.\mathrm{getNext}(v_i)$ \;
				$\textit{cost} \gets \bm{d}_{i, j} + \bm{d}_{j, k} - \bm{d}_{i, k}$ \label{alg:distance:detour-cost} \;
				\If{$\textit{cost} > \textit{maxDetour}$}{
					$\seq{S}.\mathrm{notBetween}(v_i, v_j, v_k)$ \label{alg:distance:remove-detour}\;
				}
			}
		}
	}
\end{algorithm}

Algorithm \ref{alg:distance} follows a structure used in most filtering algorithms over sequence variables.
The partial sequence $\lbSeq{}$ is first traversed and potentially used to filter other variables.
Then, all insertions are inspected and possibly removed if they cannot be performed.

Regarding constraints in the next sections, the reader is referred to \citep{delecluse2022sequence, thomas2020insertion} and Appendix \ref{sec:filtering} for the presentation of filtering algorithms.

\subsection{TransitionTimes}

The $\mathrm{TransitionTimes}$ constraint is used for problems where node visits involve a service duration and are restricted by time windows, with a transition time required to move from one node to the next.
More formally, each node $v_i \in V $ is attached to an integer variable $\cpvar{Time}_i$ representing the start of service at that node and a service duration value $\bm{s}_i$. We denote the array of these variables by $\cpvar{Time}$.
A matrix $\bm{d} \in \mathbb{Z}^{|V| \times |V|}$ defines the transition times between elements and satisfies the triangular inequality. The definition of the constraint is:
\begin{equation}
    \mathrm{TransitionTimes}(\seq{S}, \cpvar{Time}, \bm{s}, \bm{d}) \leftrightarrow
    \forall v_i \precIn{\seq{S}} v_j : 
    \cpvar{Time}_{i} + \bm{s}_{i} + \bm{d}_{i, j} \leq \cpvar{Time}_{j}   \label{eq:tt_definition}
\end{equation}

We consider that waiting at a given node (i.e., reaching it before its time window without beginning the task related to it) is possible, which is why \eqref{eq:tt_definition} uses inequalities.
Moreover, the time variable $\cpvar{Time}_v$ of an unvisited node $v \notin \seq{S}$ is not constrained.

\subsection{Precedence}
\label{sec:precedence}

For some applications, visiting a set of nodes in a specific order is important, such as the visit of a pickup that must be done before visiting the corresponding drop-off.
The $\mathrm{Precedence}$ constraint can be used in such scenarios, ensuring that an ordered set of nodes $\seq{o}$ appears in the same order in a sequence variable ($\seq{o}$ being a fixed sequence, not a variable). It is formally defined as
\begin{equation}
\label{eq:precedence_definition}
    \mathrm{Precedence}(\seq{S}, \seq{o}) \leftrightarrow \forall v_i \precIn{\seq{o}} v_j : v_i, v_j \in \seq{S} \implies v_i \precIn{\seq{S}} v_j
\end{equation}

Note that some or all nodes in $\seq{o}$ may be absent from the sequence variable. If the nodes from the set $\seq{O}$ must be all present or all absent, one can easily enforce this with $\forall v_i, v_j \in \seq{o} : \nodeRequired{\seq{S}}{i} = \nodeRequired{\seq{S}}{j}$.

\subsection{Cumulative}

Some variations of VRP involve pickups and deliveries, transporting goods or people.
The $\mathrm{Cumulative}$ constraint can be used to represent those scenarios. It ensures that going through all pickups and deliveries visited in a sequence respects an assigned capacity.

More specifically, let us define an \textit{activity} $i$ as a pair of nodes ($p_i, d_i$) corresponding to its start (pickup) and end (delivery), respectively. The set of all activities is written $A$. 
An activity $i \in A$ consumes a certain load $\bm{l}_i$ during its execution and can be in one of three states with respect to the partial sequence $\lbSeq{}$: \textit{fully inserted} if $p_i \in \lbSeq{} \land d_i \in \lbSeq{}$, \textit{non-inserted} if $p_i \notin \lbSeq{} \land d_i \notin \lbSeq{}$, and \textit{partially inserted} otherwise (the start or the end is inserted but not both). 
The $\mathrm{Cumulative}$ constraint with a maximum capacity $c$, with starts $\bm{p}$, corresponding ends $\bm{d}$ and loads $\bm{l}$ is defined as:
\begin{equation}
	\mathrm{Cumulative}(\seq{S}, \bm{p}, \bm{d}, \bm{l}, c) \leftrightarrow 
		\begin{cases} 
			\left( \forall v \in \seq{S} : \sum_{i \in A \mid p_i \preceq v \prec d_i} \bm{l}_i \le c \right) \land \\
			\left( \forall i \in A : p_i \in \seq{S} \iff d_i \in \seq{S} \right)  \land  \\
			\left( \forall i \in A : \mathrm{Precedence}(\seq{S}, (p_i, d_i)) \right)
		\end{cases}
    \label{eq:cumul_definition} 
\end{equation}

%This constraint implies that the start $p_i$ of an activity $i \in A$ is visited before its end $d_i$ ($\forall i \in A : \mathrm{Precedence}(\seq{S}, (p_i, d_i))$), 
This constraint implies precedence constraints between the start $p_i$ and end $d_i$ of an activity $i \in A$ ($\forall i \in A : \mathrm{Precedence}(\seq{S}, (p_i, d_i))$), 
and that its nodes are either all present or all absent ($\forall i \in A : \nodeRequired{\seq{S}}{p_i} = \nodeRequired{\seq{S}}{d_i}$).
An example of a sequence over which the constraint holds is presented in Figure \ref{fig:cumulative-example-fixed}.

\begin{figure}[!ht]
    \centering
    \input{figures/cumulative-example-fixed}
    \caption{$\mathrm{Cumulative}(\seq{S}, (s_0, s_1, s_2, s_3), (e_0, e_1, e_2, e_3), (2, 1, 1, 2), 3)$ over a fixed sequence variable $\seq{S} = \first \cdot s_0 \cdot s_1 \cdot e_1 \cdot e_0 \cdot s_3 \cdot e_3 \cdot \last$. The sequence is shown at the top, and the graph below shows the accumulated load for each node $v \in \seq{S}$. Note how nodes $s_2, e_2$ related to activity 2 are not part of the sequence.}
    \label{fig:cumulative-example-fixed}
\end{figure}

\section{Search on a sequence variable}
\label{section:search}

This section presents the basic search procedure used in conjunction with one or more sequence variables to explore the search space.

\subsection{Branching}

Filtering of the constraints is generally not enough to terminate with fixed sequences.
A search procedure is needed to explore the search space.
When working with sequence variables, this corresponds to iteratively choosing an unfixed sequence and applying alternative decisions further constraining its domain, through the use of $\emph{insert}$ and $\emph{notBetween}$ operations.
Once all sequences in the problem are fixed and the constraints are satisfied, a solution to the problem has been found.

As shown in the left part of Figure~\ref{fig:sequence-construction-split}, different intermediate sequences can ultimately lead to the same one through different insertion steps. 
This symmetry, induced by branching decisions, can cause inefficiencies. 
Ideally, we would explore search trees corresponding to disjoint search spaces.

\begin{figure}[!ht]
    \centering
    \input{figures/sequence-branching}
    \caption{Sequences created from a fully connected graph $G(V, E)$ with $V = \{\first, v_1, v_2, v_3, \last\}$, where all nodes are required ($V = R$). Nodes $\first$ and $\last$ are implicit and not shown. Left: each sequence is extended by inserting any node at each step. Right: a node is first selected for insertion, then all its feasible positions are considered for insertion before moving to the next node (first $v_1$, then $v_2$, then $v_3$).}
    \label{fig:sequence-construction-split}
\end{figure}

\subsubsection*{Disjoint search spaces}

A simple branching strategy that guarantees the generation of disjoint search spaces is the two-step $n$-ary branching:
\begin{enumerate}
    \item \textbf{Node selection}: Choose an insertable node $v_i \in I$ within a sequence variable $\seq{S}$.
    \item \textbf{Node branching}: For each insertion of $v_i$ in $\seq{S}$, create a corresponding branching decision.
\end{enumerate}

The first step is similar to the variable selection used with integer variables. It allows the integration of first-fail strategies, such as selecting nodes with the fewest possible insertion points. The second step is conceptually similar to value selection; therefore, the most promising insertions should be attempted first. It allows the integration of insertion-based heuristics, such as selecting the insertion that results in the smallest increase in tour length.
This strategy was already presented for the DARP in Algorithm \ref{alg:branching_seq}.

As can be observed in the right part of Figure~\ref{fig:sequence-construction-split}, this strategy generates only distinct sequences.
Another simple binary branching strategy that offers the same guarantees is to replace the second step with only two branches. An insertion point is selected, with the insertion performed on the left branch, while the corresponding $\mathrm{notBetween}$ constraint is enforced on the right branch.

\subsection{Large Neighborhood Search}

The usage of LNS with sequence variables was already presented in Algorithm \ref{alg:lns_darp}. In the context of VRP, this algorithm consists of relaxing some tours by removing nodes from sequences. This approach maintains partial tours, similar in spirit to the partial-order scheduling method introduced for scheduling in \cite{godard2005randomized}. It also accurately corresponds to the original LNS presented in \cite{shaw1998using}, where partial tours are extended through insertions.

The reconstruction phase uses CP and its search capabilities to reinsert the removed nodes into the restricted problem, possibly within a time limit, before restarting the process.

The set of nodes to relax can be selected in various ways, such as randomly or with more advanced strategies based on relatedness criteria, such as geographical proximity or time-based considerations between nodes, as in \cite{bent2004two, christiaens2020slack}.

\section{Experimental Results}
\label{section:xp}

We previously presented sequence variables, several global constraints and considerations on the search procedure. 
We will now use that information to solve routing and sequencing problems.

\subsection{Dial-A-Ride Problem}
\label{sec:darp-xp}

We consider the DARP model from \eqref{eq:darp-objective}-\eqref{eq:darp-max-route-duration}.
The state-of-the-art, to the best of our knowledge, is the Adaptive LNS proposed by \cite{gschwind2019adaptive}, combining the operators proposed by \cite{ropke2006adaptive} (random, worst and Shaw removal with greedy and regret based insertions) with additional relaxations, exploiting 9 removal and 5 insertion operators in total. Insertions are evaluated based on a feasibility check specific to the DARP, and obtained solutions close to the best found so far are further optimized using an adaptation of the neighborhood proposed in \cite{balas2001linear}. For further reading on the DARP, the reader is referred to the literature review in \cite{ho2018survey}.

\subsubsection{Search}

The search procedure for the DARP experiments is a specific instantiation of the generic sequence branching (Algorithm~\ref{alg:branching_seq}) adapted for pickup and delivery requests, as presented in Algorithm~\ref{alg:branching_darp}. This procedure generates the alternative child nodes of the search tree by considering all valid insertions of both the pickup and delivery nodes of a chosen request into the sequence variables representing the active routes.

\begin{algorithm}[!ht]
\caption{Creation of the branching points for the DARP in the Sequence Variable Model.}
\label{alg:branching_darp}
\If{$\bigwedge_{k \in K} \cpvar{Route}_k.\mathrm{isFixed}()$}{
    \Return solution \; \label{alg:branching-solution}
}
$r_i \gets \underset{r_j \in R \mid r_j \text{ not yet inserted}}{\arg\!\min}{\; \sum_{k\in K} { \cpvar{Route}_k.\mathrm{nInsert}(\pick{r_j}) \times \cpvar{Route}_k.\mathrm{nInsert}(\drop{r_j})} }$  \label{alg:branching_get_request} \;
$\textit{branches} \gets \{\}$ \;
\For{$k \in K$}{
    $I^+ \gets \cpvar{Route}_k.\mathrm{getInsert}(\pick{r_i})$ \; \label{alg:branching:select-pickup-insert}
    \For{$p^+ \in I^+$}{
        $I^- \gets \cpvar{Route}_k.\mathrm{getInsertAfter}(\drop{r_i}, p^+) \cup \{ \pick{r_i} \}$ \label{alg:branching:select-drop-insert}\;
        \For{$p^- \in I^-$}{
            $\textit{branches} \gets \textit{branches} \cup \left\{( \cpvar{Route}_k.\mathrm{insert}(p^+, \pick{r_i}) \land \cpvar{Route}_k.\mathrm{insert}(p^-, \drop{r_i}) ) \right\}$ \label{alg:branching:perform-insertion}
        }
    }
}
sort $\textit{branches}$ by increasing heuristic value \label{alg:branching_sort} \;
\Return $\textit{branches}$ \;
\end{algorithm}

The DARP branching prioritizes the unvisited request $r_i$ with the fewest remaining feasible insertions (line~\ref{alg:branching_get_request}). Given this request, all insertion points $I^+$ for its pickup $\pick{r_i}$ are retrieved (line \ref{alg:branching:select-pickup-insert}), followed by suitable insertions $I^-$ for its delivery $\drop{r_i}$ (line \ref{alg:branching:select-drop-insert}). All valid combinations of predecessors are considered as an alternative decision to explore (line \ref{alg:branching:perform-insertion}).

\begin{example}
The example refers to Figure \ref{fig:darp-branching} with only one vehicle. A route serves both visits of $r_1$. The request $r_2$ is selected, resulting in six possible sequences for inserting its two visits, each corresponding to a possible branching decision in the search tree. These insertion options are sorted by their cost.
\begin{figure}[!ht]
    \centering
    \input{figures/darp-branching-color}
    \caption{Paths generated by Algorithm \ref{alg:branching_darp} from an initial path (left). Numbers indicate the order in which the paths are considered.}
    \label{fig:darp-branching}
\end{figure}
\end{example}

We use a simple LNS that always relaxes 10 requests chosen randomly and uses Algorithm \ref{alg:branching_darp} for branching. The request to insert is the one that has the minimum number of insertions, defined as the sum of insertions for its pickup and drop. Insertions with a small increase in distance and high-preserved time slack are considered first, as in \citep{lnsffpa}. When considering an insertion of node $v_j$ between nodes $v_i, v_k$ in a vehicle, the heuristic cost is defined as:
\begin{equation}
    C_1 (\bm{d}_{i, j} + \bm{d}_{j, k} - \bm{d}_{i, k}) - C_2 (\lceil \cpvar{Time}_k \rceil - \lfloor \cpvar{Time}_i \rfloor - \bm{s}_i - \bm{d}_{i, j} - \bm{s}_j - \bm{d}_{j, k})
    \label{eq:darp-insertion-cost}
\end{equation}

where constants $C_1, C_2$ control the importance of the detour cost and the preserved time slack, respectively. The heuristic cost for inserting a request $r_i \in R$ is the sum of the costs for inserting its pickup $\pick{i}$ and drop $\drop{i}$.
We use $C_1=80$ and $C_2 = 1$ as in \citep{lnsffpa}.

\subsubsection{Computational Results}

We compare the sequence variable approach with other approaches suitable for modeling VRP:
\begin{itemize}
    \item \textbf{A successor model} written in MiniZinc \citep{minizinc} (version 2.10.0) and run with the Gecode solver \citep{schulte2006gecode} (version 6.4.0), which performed best on the DARP across all MiniZinc backends. The underlying CP model is essentially the same as in \citep{berbeglia2011checking}. Both exhaustive search (\textbf{Succ}) and LNS (\textbf{Succ-LNS}) are reported.
    \item \textbf{OR-Tools Routing} (Java API, version 9.11.4210), which relies on local search \citep{ortools,cuvelier2023or}. All combinations of first solution strategy and local-search metaheuristics in the solver have been tried, and we report the best overall configuration.
    \item \textbf{Hexaly} (version 14.5), a commercial solver specialized in routing \citep{hexaly}, using a model provided by the Hexaly team.
    \item \textbf{CP Optimizer (CPO)} (version 22.1.2), a commercial scheduling solver, whose model is written using head-tail sequence variables.
    \item \textbf{The Mixed-Integer Programming model} from \cite{rist2021new}, denoted as \textbf{RF (2021)}, using the authors' implementation (relying on Python 3.7.0 and Gurobi 9.1.0). This is the state-of-the-art exact method for the DARP. Its decision variables are fragments of paths that are combined, and whose enumeration exploits dominance rules specific to the DARP.
\end{itemize}

Details on the models are in Appendix \ref{sec:darp}. 
All experiments were conducted on a machine running Ubuntu 24.04, equipped with an AMD Ryzen™ 9 9950X processor (16 cores, 32 threads, 1280~KB of L1 cache, 16~MB of L2 cache, and 64~MB of L3 cache) and 64~GB of RAM. No GPU was used.
The Java-based implementations were executed with OpenJDK 25.0.3; the sequence variable was implemented in Java using MaxiCP \citep{MaxiCP2024}, an extended and open-source version of MiniCP \citep{minicp}.
\begin{figure}[!t]
    \centering
    \includegraphics[width=0.85\linewidth]{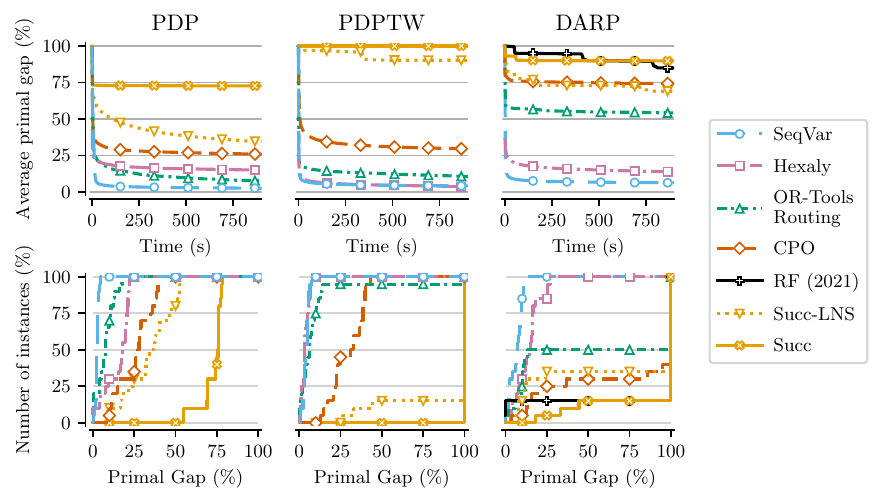}
\caption{Primal-gap performance on the PDP, PDPTW, and DARP. The top row shows the mean primal gap over time; lower-left curves indicate faster convergence. The bottom row shows the percentage of instances with final gap at most the x-axis value after 15 minutes; upper-left curves are better.}
    \label{fig:instances-over-gap}
\end{figure}

GNU Parallel (version 20231122) \citep{tange2023gnuparallel} was used to execute 32 single-threaded solver processes concurrently. A time limit of 15 minutes was imposed for each run.
Each solver was run 10 times per instance, and all reported averages are computed over these 10 runs.

The results for the 20 instances from \cite{cordeau2003} are shown in Figure~\ref{fig:instances-over-gap}. The top row shows the average primal gap to the best known solution over time for each solver, where the gap is first averaged over runs for each instance and then averaged over all instances. The bottom row shows the proportion of instances whose final average gap is at most $\tau$: $\Gamma_s(\tau) = | \{\, p \in \mathcal{P} : \gamma_{p,s} \leq \tau \,\} | / |\mathcal{P}|$, where $\mathcal{P}$ denotes the set of problem instances and $\gamma_{p,s}$ denotes the final primal gap of solver $s$ on instance $p$, averaged over all runs. In both plots, the gap is set to 1 for any run in which no incumbent solution is found.

The approach using a sequence variable, although based on a simple LNS, consistently achieves solutions within 15\% of the best-known results for the DARP.
The exact approach RF (2021) is the best performer on 3 instances only, and fails to provide any feasible solution on the remaining instances, as the enumeration of the fragments of paths times out.
Detailed results per instance are provided in Appendix \ref{sec:darp:detailed-results}.

We also report performance when (i) the ride time constraints and maximum route duration of the DARP are removed—reducing it to a Pickup and Delivery Problem with Time Windows (PDPTW)—and (ii) when time windows are additionally removed—reducing it further to a Pickup and Delivery Problem (PDP). 
The \textbf{RF (2021)} method, however, is only reported on the DARP, as it requires substantial changes to be adapted to the PDP and PDPTW.
These VRP variants illustrate the adaptability of the proposed models. For these cases as well, the best results are obtained using the sequence variable approach.

\subsection{A Prize Collecting Scheduling Problem}

Although this problem is not a VRP, it illustrates the applicability of sequence variables to scheduling problems.

The Prize-Collecting Job Sequencing with One Common and Multiple Secondary Resources (PC-JSOCMSR) is a scheduling problem in which the objective is to maximize the reward collected from the selected tasks, while satisfying several resource and timing constraints.
It was introduced in \cite{horn2018job} and is motivated by the scheduling of cancer patients requiring particle therapy.
The problem is defined on a set of $n$ tasks $t_i \in T$, each with processing duration $d_i$ and reward $z_i$ when executed.
Every task requires the global resource and one particular secondary resource $j \in R$, where $R$ is the set of secondary resources.
The tasks associated with secondary resource $j$ are denoted by $T(j)$.
Neither the global resource nor any secondary resource can process more than one task simultaneously.
Moreover, executing a task $t_i$ reserves its secondary resource for a pre-processing duration $d_i^{pre}$ and a post-processing duration $d_i^{post}$.
Each task must be executed non-preemptively in one of its $\omega_i$ time windows $w_i = \{ w_{ik} \mid k = 0, \dots, \omega_i - 1 \}$, where $w_{ik} = [ w_{ik}^{\textit{start}}, w_{ik}^{\textit{end}} ]$.
The objective is to maximize the total reward of the executed tasks.

To simplify the naming, the problem will be abbreviated to the Prize Collecting Scheduling Problem (PCSP) instead of PC-JSOCMSR.
An example of a PCSP solution is presented in Figure \ref{fig:pc-jso-example}.

\begin{figure}[!ht]
	\centering
	\input{figures/pc-problem-example}
	\caption{PCSP solution}
	\label{fig:pc-jso-example}
\end{figure}

The best published approaches for this problem rely on Mixed Integer Linear Programming for exact search, and local search for heuristic approaches for the largest instances \cite{froger2023new}.
Some other methods, using dynamic programming, decision diagrams, or CP, have also been developed in \cite{horn2021based}.

\newpage
\paragraph{Model}

Our model for the PCSP is as follows:

\begin{align}
	& \max \sum_{t_i \in T}{z_i \cdot \nodeRequired{\cpvar{G}}{t_i}} \label{eq:PC-JSOCMSR-objective} \\
	\text{subject to:} \notag \\
	& \mathrm{TransitionTimes}(\cpvar{G}, \cpvar{Time}, \bm{d}, \bm{O}) & \label{eq:PC-JSOCMSR-tt-global}\\
	& \mathrm{TransitionTimes}(\cpvar{R}_j, \cpvar{Time}, \bm{d}, \bm{Tr}) &\quad \forall j \in R \label{eq:PC-JSOCMSR-tt-secondary}\\
	& \mathrm{SubSequence}(\cpvar{G}, \cpvar{R}_j) &\quad \forall j \in R \label{eq:PC-JSOCMSR-tt-subsequence} \\
	& \nodeRequired{\cpvar{R}_j}{t_i} = \nodeRequired{\cpvar{G}}{t_i} &\quad \forall j \in R, \; \forall t_i \in T(j) \label{eq:PC-JSOCMSR-required-equality} \\
	& \nodeRequired{\cpvar{R}_j}{t_i} = 0 &\quad \forall j \in R, \; \forall t_i \in (T \setminus T(j)) \label{eq:PC-JSOCMSR-excluded-outside-sequence} \\
	& \cpvar{Time}_i \in w_i &\quad \forall t_i \in T \label{eq:PC-JSOCMSR-multi-tw} \\
	& \mathrm{first}(\cpvar{G}) = \first, \quad \mathrm{last}(\cpvar{G}) = \last \label{eq:PC-JSOCMSR-first-global} \\
	& \mathrm{first}(\cpvar{R}_j) = \first, \quad \mathrm{last}(\cpvar{R}_j) = \last &\quad \forall j \in R \label{eq:PC-JSOCMSR-first}
\end{align}

The model uses one sequence variable for each resource: $\cpvar{G}$ represents the global resource, while $\cpvar{R}_j$ represents the secondary resource $j \in R$.
The endpoints of these sequences are dummy nodes.
The objective maximizes the total reward of the executed tasks \eqref{eq:PC-JSOCMSR-objective}.
TransitionTimes constraints enforce the time-window requirements on both the global sequence \eqref{eq:PC-JSOCMSR-tt-global} and the secondary sequences \eqref{eq:PC-JSOCMSR-tt-secondary}.
For each task $t_i$, $\cpvar{Time}_i$ is derived from its original time windows $w_i$ \eqref{eq:PC-JSOCMSR-multi-tw}; together, these variables form the array $\cpvar{Time}$.
The transition times on the global resource are all zero and are represented by the matrix $\bm{O} \in \mathbb{Z}^{T \times T}$ \eqref{eq:PC-JSOCMSR-tt-global}.
For a secondary resource, the transition-time matrix $\bm{Tr} \in \mathbb{Z}^{T \times T}$ has entries $\bm{Tr}_{i,i'} = d_{i}^{post} + d_{i'}^{pre}$.
Consistency between the global and secondary resource schedules is imposed by the SubSequence constraint \eqref{eq:PC-JSOCMSR-tt-subsequence}, which requires $\seq{R}_j$ to be a subsequence of $\cpvar{G}$: $\cpvar{R}_j \subseteq \cpvar{G}$.
It can be viewed as a special case of the Precedence constraint from Section \ref{sec:precedence}, in which the subsequence is itself a sequence variable.
Further details on this constraint are provided in Appendix \ref{sec:subsequence-details}.
If a task is executed on the global resource, it must also be executed on its associated secondary resource \eqref{eq:PC-JSOCMSR-required-equality}.
Finally, a task cannot be executed on a secondary resource with which it is not associated \eqref{eq:PC-JSOCMSR-excluded-outside-sequence}.

\paragraph{Search}

We use the same LNS framework as for the DARP, but considering each task individually instead of working with requests. 
The first incumbent is obtained with a simple constructive search that explicitly handles optional tasks: required-but-uninserted tasks are inserted first, selecting the one with the fewest insertion points.
Otherwise, the optional task with the largest reward is selected. 
The branching enumerates all feasible insertion positions and, for optional tasks, one additional exclusion branch.
This is similar to the enumeration of branching choices from Algorithm~\ref{alg:branching_seq}.

During the LNS repair search, we use a binary insertion branching. 
The task to insert is selected according to the number of available insertions scaled by its reward. 
Its insertion position is chosen with the same distance/slack criterion as in Equation~\eqref{eq:darp-insertion-cost}, using transition times on the considered sequence. 
The two branches either insert the task at its best position or forbid this insertion.

\subsubsection{Computational Results}

We use the instances from \cite{froger2023new} and report the quality of the found solutions relative to the best solutions reported by the authors.
Given the large dataset available (1200 instances), we restricted ourselves to 3 instances per number of tasks and secondary resources, for $n \in \{50, 100, 150, 200, 250, 300, 400, 500\}$ and $m \in \{ 3, 4\}$ (Avionic dataset) or $m \in \{2, 3\}$ (Therapy dataset).
This leaves 96 instances ranging from 50 to 500 tasks and 2 to 4 secondary resources.

We compare to the same technologies as in Section \ref{sec:darp-xp}.
For MiniZinc, we compare to both a successor model (MZN-Succ) and the model from \cite{horn2018job} (MZN-Disj), which model each machine's occupancy using disjunctive constraints and prevent the use of a transition-time matrix.
We used the Chuffed solver \cite{chu2018chuffed} in both cases, which outperformed the other ones for the successor model and was already used in \cite{horn2018job}.
LNS is not shown for MiniZinc as the Chuffed backend does not implement it.

Figure \ref{fig:instances-over-gap-pcsp} shows the average gap over time (left) and the final distribution (right) over the instances.
The CPO model is the clear winner, while the sequence variable is a close second.
Compared to the DARP, the routing models from Hexaly and OR-Tools lag further behind the sequence variable approach, and OR-Tools dominates Hexaly.
As for the DARP, the MiniZinc models are the least performant out of the tried approaches.
A detailed breakdown of the final gap values per instance size is provided in Appendix \ref{sec:pcsp-detailed-results}, along with further details on the set of instances used.

\begin{figure}[!ht]
	\centering
	\includegraphics[width=\linewidth]{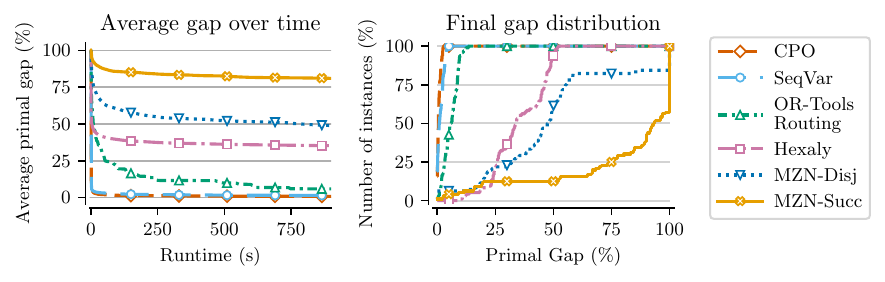}
	%\caption{Average gap over time to the b.k.s. (left) and final gap distribution on the instances (right). Curves on the bottom left are preferred for the left graph, and curves on the top left are preferred for the right graph.}
	\caption{Primal-gap performance on the PCSP. The left plot shows the mean primal gap over time; lower-left curves indicate faster convergence. The right plot shows the percentage of instances with final gap at most the x-axis value after 15 minutes; upper-left curves are better.}
	\label{fig:instances-over-gap-pcsp}
\end{figure}

\section{Conclusion}

This paper enhances previous proposals of sequence variables, which are used in CP to tackle vehicle routing and sequencing problems. Their domain representation, implementation, and interactions with Boolean variables are proposed and formalized. They ease the modeling of complex VRPs such as the DARP while staying close to the state of the art in terms of performance. These variables are compatible with optional visits, insertion-based heuristics, and can be easily combined with Large Neighborhood Search, making them practical for solving complex VRPs in a CP framework.

\begin{acks}
This work was supported by Service Public de Wallonie Recherche under grant n°2010235 – ARIAC by DIGITALWALLONIA4.A
	
The authors would like to thank the Hexaly team for providing their model of the DARP, as well as a license key for running their solver.
\end{acks}

\printbibliography

\appendix

\section{Domain Implementation Details}
\label{sec:implementation-invariants}

\subsection{Initialization}
When initialized, a compact sequence domain is defined on the set of nodes $V$ belonging to the graph $G(V, E)$, and its successor array $\directSuccessor{}$ contains one entry per node $v \in V$. The two starting and ending nodes, $\first$ and $\last$, are identified upon the initialization of the domain. These nodes constitute the initial sequence of nodes represented by the domain, which form a cycle in the successor array.
\begin{equation}
    \directSuccessor{\first} = \last \land \directPredecessor{\last} = \first \land \directSuccessor{\last} = \first \land \directPredecessor{\first} = \last \label{eq:init-seqvar-succ}
\end{equation}
Note that a link from the end node $\last$ to the first node $\first$ is encoded, thus representing a cycle instead of a path, to ease invariant encoding in further equations. The other nodes point to themselves as self-loops. At initialization, the edge set forms a complete graph, with the exception of edges that originate at $\last$ and end at $\first$: $E = \{ (v_1, v_2) \mid v_1, v_2 \in V : v_1 \neq \last \land v_2 \neq \first \land v_1 \neq v_2 \} \cup \{ (\last,  \first) \}$. The set of nodes that may be inserted is defined as $I = V \setminus \{ \first, \last \}$.

\subsection{Invariants}

First, we introduce one predicate that will be instrumental in defining the consistency invariants. It determines (in constant time) if a node $v_i$ belongs to the partial sequence (\textit{i.e.}, if $v_i \in \seq{s}$ holds) by verifying if it is not a self-loop:
\begin{equation}
	\mathrm{isMember}(v_i) \equiv \directSuccessor{i} \neq v_i .
	\label{eq:is-member}
\end{equation}

The lower-level consistency invariants expressed on the data structures are given next. We first identify the invariants describing the channeling between some data structures.
\begin{align}
	& \forall v_i, v_j \in V: v_i \in \predecessors{j} \iff v_j \in \successors{i} \label{eq:pred-and-succ-link}\\
	& \forall v_i, v_j \in V: \directSuccessor{i} = v_j \iff \directPredecessor{j} = v_i \label{eq:successor-array}
\end{align}

The invariant \eqref{eq:pred-and-succ-link} ensures that each edge $(v_i, v_j)$ appears twice in the data structures: one for the adjacency set of the incoming node, and one for the outgoing node. The invariant \eqref{eq:successor-array} ensures that for any two nodes $v_i$ and $v_j$ in the graph, $v_i$ is the predecessor of $v_j$ if and only if $v_j$ is the successor of $v_i$. It guarantees a consistent and mutual relationship between successor and predecessor links for every node in the graph. In addition to channeling, specific invariants are used to maintain the counters previously introduced.
\begin{align}
	& nS = \left| \{ v_i \mid v_i \in V \land \mathrm{isMember}(v_i) \} \right| \label{eq:implem-nS-tracking} \\
	& \forall v_j \in I: nI_j = |\{ v_i \mid v_i \in \predecessors{j} \land \mathrm{isMember}(v_i) \}| \label{eq:implem-insertion-counter} \\
	& \forall v_i \in V : (nI_i \geq 1) \leftrightarrow (v_i \in I) \label{eq:implem-at-least-one-insertion}
\end{align}
The length of the partial sequence is tracked in \eqref{eq:implem-nS-tracking}. Invariant \eqref{eq:implem-insertion-counter} tracks how many insertions are feasible for a given node $v_j \in I$. This counter is used to ensure that every node $v_j \in I$ has at least one possible insertion \eqref{eq:implem-at-least-one-insertion}: otherwise, the node is not insertable and thus has zero insertions.

Given the array of successors $\directSuccessor{}$, one can define the set of nodes reachable in the circuit containing node $v_i \in V$:
\begin{align}
	\mathrm{circuit}(v_i) &= \mathrm{circuit}(v_i, \emptyset)  \label{eq:reachable-definition-one-arg} \\
	\mathrm{circuit}(v_i, S) &= \begin{cases}
		S & \text{ if } v_i \in S \\
		\mathrm{circuit}(\directSuccessor{i}, S \cup \{ v\} ) & \text{ otherwise } \\
	\end{cases} \label{eq:reachable-definition-recursive}
\end{align}
Intuitively, from \eqref{eq:reachable-definition-one-arg}, $\mathrm{circuit}(v_i)$ returns all nodes in the sequence from $\first$ to $\last$ if $v_i$ is in the sequence; otherwise, it returns a set containing only $v_i$. This is done by recursively following the pointers $\directSuccessor{}$ of the successor array. Using this definition, the implementation invariants are as follows.
\begin{align}
	& \directSuccessor{\last} = \first \land \directPredecessor{\first} = \last \label{eq:successor-of-last} \\ 
	& \forall v_i, v_j \in V, v_i \neq v_j: \left( \directSuccessor{i} \makebox[0pt][l]{\phantom{\directSuccessor{j}}} = v_j \implies v_j \in \successors{i}\right) \land \left( \directPredecessor{j} = v_i \implies v_i \in \predecessors{j}\right) \label{eq:only-one-succ-and-pred-is-member} \\
	& \forall v_i, v_j \in V, v_i \neq v_j: \directSuccessor{i} \neq \directSuccessor{j} \label{eq:unique-successor}\\ 
	& \forall v_i \in V: \mathrm{isMember}(v_i) \iff \mathrm{circuit}(v_i) = \mathrm{circuit}(\first) \label{eq:path-to-last} \\
	& \forall v_i, v_j \in V, v_i \neq v_j: \neg \mathrm{isMember}(v_i) \land \neg \mathrm{isMember}(v_j) \implies v_i \in \predecessors{j} \label{eq:clique-insertable-nodes} \\
	& \forall v_i, v_j, v_k \in V, v_i \neq v_j \neq v_k: \directSuccessor{i} = v_k \implies \left( v_i \in \predecessors{j} \iff v_k \in \successors{j} \right) \label{eq:detour-needs-two-edges}
\end{align}

Invariant~\eqref{eq:successor-of-last} ensures that the successor of the last node $\last$ being visited always points toward the first node $\first$. Invariant \eqref{eq:only-one-succ-and-pred-is-member} enforces that the current successor of a node $v_i \in \seq{s}$ exists within the outgoing edges of the node $v_i$. Invariants \eqref{eq:unique-successor} and \eqref{eq:path-to-last} ensure that only one sub-circuit is encoded within the successor array. All successors must be different, and if the successor of a node $v_i \in V$ is set (\textit{i.e.} $v_i$ belongs to the partial sequence) then $v_i$ belongs to the circuit of the first node $\first$ (and the last node \(\last\) given that \(\mathrm{circuit}(\first) = \mathrm{circuit}(\last)\)). Invariant \eqref{eq:clique-insertable-nodes} enforces that all insertable nodes form a clique. Finally, \eqref{eq:detour-needs-two-edges} ensures that the two edges \((v_i, v_j), (v_j, v_k)\) needed to insert a node \(v_j \in I\) after node \(v_i\) (with \(v_k\) being the current successor of \(v_i\)) are either both absent or both present.

Lastly, some invariants interact specifically with the required $R$ and excluded nodes $X$. Moreover, the implementation may itself modify the set of required nodes $R$, so that it always contains nodes from the partial sequence: \(\seq{s} \subseteq R\), given that nodes that are part of the partial sequence are always visited in all sequences in the domain. The invariants are:
\begin{align}
	& v_i \in X \iff \predecessors{i} = \emptyset \iff \successors{i} = \emptyset \label{eq:excluded-means-no-edge}\\
	& \forall v_i \in V : \mathrm{isMember}(v_i) \implies v_i \in R \label{eq:if-succ-then-required} \\
	& \forall v_i \in R \cap I : nI_i > 1 \label{eq:required-insertable-counter-above-1}
\end{align}

Invariant \eqref{eq:excluded-means-no-edge} ensures that an excluded node has no edge attached to it. Invariant~\eqref{eq:if-succ-then-required} captures the fact that nodes that are part of the partial sequence \(\seq{s}\) are always visited, and thus considered as mandatory. Finally, invariant~\eqref{eq:required-insertable-counter-above-1} guarantees that required nodes not part of the sequence have at least two insertions remaining. Otherwise, if only one insertion remained for such a node, it would directly be used to add the node to the partial sequence \(\seq{s}\).

\section{Domain updates}
\label{sec:implementation-updates}

\subsection{Insertion}  

Algorithm~\ref{alg:weak-seqvar-insertion} is used to perform an $\textit{Sq}.\mathrm{insert}(v_1, v_2)$ operation on a compact sequence domain $\textit{Sq}$, provided that $(v_1, v_2)$ is a feasible insertion. The inserted node $v_2$ is first marked as required, removed from the insertable nodes and the size $nS$ of the partial sequence $\lbSeq$ is increased (lines \ref{alg:seqvar-insertion-update-R} to \ref{alg:seqvar-insertion-update-S}). Then, every node $v_i$ linked to $v_2$ is inspected (line \ref{alg:seqvar-insertion-loop}). If a node $v_i$ belongs to both the partial sequence $\lbSeq$ and to the incoming edges of $v_2$, it could previously be used to perform an insertion using $(v_i, v_2)$. Given that $v_2$ is already inserted, such edges are removed (lines \ref{alg:weak-seqvar-insertion-detour-removal-in-sequence-begin} to \ref{alg:weak-seqvar-insertion-detour-removal-in-sequence-end}). Otherwise, if $v_i$ does not belong to $\lbSeq$, two situations may occur: either $v_i$ cannot be placed directly after $v_1$, and thus cannot be placed directly after $v_2$ either, according to \eqref{eq:forbidden-insertions-update-insertion} (lines \ref{alg:weak-seqvar-insertion-detour-removal-kept-1} and \ref{alg:weak-seqvar-insertion-detour-removal-kept-2}); or the node $v_i$ can be inserted directly after $v_1$ and can now be inserted directly after $v_2$ as well, increasing its counter $nI_i$ (line \ref{alg:weak-seqvar-insertion-counter-increment}). Lastly, the edges, successor and predecessor of $v_1, v_2$ and $v_3$ (the previous successor of $v_1$) are updated to reflect the insertion (lines \ref{alg:weak-seqvar-insertion-successor-update} to \ref{alg:seqvar-insertion-member-edges-update}).

\begin{algorithm}[!ht]
	\caption{$\textit{Sq}.\mathrm{insert}(v_1, v_2)$}
	\label{alg:weak-seqvar-insertion}
	\SetKwInput{Input}{Input}
	\SetKwInput{Pre}{Precondition}
	\Input{$\textit{Sq}$: compact sequence domain, $(v_1, v_2)$ feasible insertion to perform}
	$R \gets R \cup \{ v_2 \}$ \label{alg:seqvar-insertion-update-R}\;
	$I \gets I \setminus \{ v_2 \}$\;
	$nI_{2} \gets 0 \label{alg:seqvar-insertion-counter-to-zero}$ \;
	$nS \gets nS + 1$ \label{alg:seqvar-insertion-update-S}\;
	$v_3 \gets \directSuccessor{1}$ \;
	\For{$v_i \in \predecessors{2} \label{alg:seqvar-insertion-loop} $}{
		\uIf{$\textit{Sq}.\mathrm{isMember}(v_i)$}{
			\If{$v_i \neq v_1 \label{alg:weak-seqvar-insertion-detour-removal-in-sequence-begin}$}{
				$\successors{i} \gets \successors{i} \setminus \{ v_2 \}, \predecessors{2} \gets \predecessors{2} \setminus \{ v_i \}$ \;
				$v_j \gets \directSuccessor{i}$ \;
				$\successors{2} \gets \successors{2} \setminus \{ v_j \}, \predecessors{j} \gets \predecessors{j} \setminus \{ v_2 \}$ \; \label{alg:weak-seqvar-insertion-detour-removal-in-sequence-end}
			}
		}
		\uElseIf{\textbf{not } $\textit{Sq}.\mathrm{canInsert}(v_1, v_i)$}{
			$\successors{2} \gets \successors{2} \setminus \{ v_i \}, \predecessors{i} \gets \predecessors{i} \setminus \{ v_2 \}$ \label{alg:weak-seqvar-insertion-detour-removal-kept-1}\;
			$\successors{i} \gets \successors{i} \setminus \{ v_2 \}, \predecessors{2} \gets \predecessors{2} \setminus \{ v_i \}$ \label{alg:weak-seqvar-insertion-detour-removal-kept-2}\;
		}
		\Else{
			$nI_i \gets nI_i + 1$ \label{alg:weak-seqvar-insertion-counter-increment}\;
		}
	}
	$\directSuccessor{1} \gets v_2, \directSuccessor{2} \gets v_3$ \label{alg:weak-seqvar-insertion-successor-update}\;
	$\directPredecessor{3} \gets v_2, \directPredecessor{2} \gets v_1$ \label{alg:weak-seqvar-insertion-predecessor-update}\;
	$\successors{1} \gets \successors{1} \setminus \{ v_3 \}, \predecessors{3} \gets \predecessors{3} \setminus \{ v_1 \}$ \label{alg:seqvar-insertion-member-edges-update}\;
\end{algorithm}

In the implementation, if $(v_1, v_2)$ does not define a feasible insertion, two situations may occur:
\begin{enumerate}
	\item If $v_2$ is already within the partial sequence and lies after $v_1$ (\textit{i.e.} $v_1, v_2 \in \lbSeq \land v_1 \precIn{} v_2$), the insertion is considered as having already been performed. Nothing happens in this case.
	\item Otherwise, $v_2$ cannot be inserted. This corresponds to a domain wipeout.
\end{enumerate}

\subsection{NotBetween} The algorithm \ref{alg:seqvar-not-between} performs a $\textit{Sq}.\mathrm{notBetween}(v_1, v_2, v_3)$ operation on a compact sequence domain $\textit{Sq}$. It iterates over the nodes $v_i$ between the nodes $v_1$ and $v_3$ (line \ref{alg:seqvar-detour-removal:for-loop}), removes the edges allowing the insertion of $v_2$ after $v_i$, and decrements its insertion counter $nI_{2}$ (lines \ref{alg:seqvar-detour-removal:edge-removal} to \ref{alg:seqvar-detour-removal:counter-update}). If the insertion counter reaches 0, the node must be excluded due to \eqref{eq:implem-at-least-one-insertion}, therefore removing the node $v_2$ from the set of insertable nodes $I$ and marking it as excluded (lines \ref{alg:seqvar-detour-removal:force-exclusion-X} and \ref{alg:seqvar-detour-removal:force-exclusion}). In this case, all edges passing through the node are removed. In contrast, if the insertion counter reaches 1 and the node is required (line \ref{alg:notbetween-check-insert}), it is automatically inserted at its only remaining insertion point (lines \ref{alg:notbetween-retrieve-insert} and \ref{alg:notbetween-insert}).

\begin{algorithm}[!ht]
	\caption{$\textit{Sq}.\mathrm{notBetween}(v_1, v_2, v_3)$}
	\label{alg:seqvar-not-between}
	\SetKwInput{Input}{Input}
        \Input{$\textit{Sq}$: compact sequence domain, $v_1$ and $v_3$ are two nodes in the partial sequence $\seq{s}$ between which node $v_2 \in V \setminus \seq{s}$ cannot appear.}
	\For{$v_i \in \seq{s} \mid v_1 \preceq v_i \prec v_3 \label{alg:seqvar-detour-removal:for-loop}$}{
		\If{$\textit{Sq}.\mathrm{canInsert}(v_i, v_2)$}{
			$v_j \gets \directSuccessor{i}$ \;
			$\predecessors{2} \gets \predecessors{2} \setminus \{ v_i \}, \successors{i} \gets \successors{i} \setminus \{ v_2 \}$ \label{alg:seqvar-detour-removal:edge-removal}\;
			$\predecessors{j} \gets \predecessors{j} \setminus \{ v_2 \}, \successors{2} \gets \successors{2} \setminus \{ v_j \}$ \;
			$nI_{2} \gets nI_{2} - 1$ \label{alg:seqvar-detour-removal:counter-update}\;
			\uIf{$nI_{2} = 0$}{
				$X \gets X \cup \{ v_2 \}$ \label{alg:seqvar-detour-removal:force-exclusion-X} \;
				$I \gets I \setminus \{ v_2 \}$ \label{alg:seqvar-detour-removal:force-exclusion} \;
				\For{$v_k \in I$}{
					$\predecessors{k} \gets \predecessors{k} \setminus \{ v_2 \}, \successors{k} \gets \successors{k} \setminus \{ v_2 \}$ \;
				}
				$\predecessors{2} \gets \emptyset, \successors{2} \gets \emptyset$ \;
				\textbf{return} \;
			}
		}
	}
	\If{$nI_{2} = 1 \text{ \textbf{and} } v_2 \in R \label{alg:notbetween-check-insert}$}{
		$\{v_i\} \gets \textit{Sq}.\mathrm{getInsert}(v_2)$ \label{alg:notbetween-retrieve-insert}\;
		$\textit{Sq}.\mathrm{insert}(v_i, v_2)$ \label{alg:notbetween-insert}\;
	}
\end{algorithm}

The algorithm \ref{alg:seqvar-not-between} is frequently called with $v_3$ being the direct successor of $v_1$: $v_1 \followedIn{} v_3$. If so, only one iteration occurs at line \ref{alg:seqvar-detour-removal:for-loop}, with $v_i = v_1$, and edge deletion occurs in constant time if $v_2$ is not excluded or inserted.

A specific situation must be checked if $v_2$ already belongs to the partial sequence. If the nodes are consecutive in the partial sequence ($v_1, v_2, v_3 \in \lbSeq \land v_1 \precIn{} v_2 \precIn{} v_3$), a domain wipeout is triggered. No filtering occurs if $v_3 \preceq v_1$.

\subsection{Require}

Requiring a node $v_i$ is obtained by marking it as required ($R \gets R \cup \{v_i\}$) and inserting it if only one insertion remains for it $(nI_i = 1)$, similarly to lines \ref{alg:notbetween-retrieve-insert} and \ref{alg:notbetween-insert} of Algorithm \ref{alg:seqvar-not-between}. If the node is excluded, a domain wipeout occurs.

\subsection{Exclude}

The exclusion of a node $v_i$ is obtained by marking it as excluded ($X \gets X \cup \{v_i\}$) and removing all edges connected to it. Although edge removal occurs, only the insertion counter $nI_i$ changes. The counters related to other nodes $v_j \in I \land v_i \neq v_j$ are not affected by edge removal, since any edge $(v_i, v_j)$ (or $(v_j, v_i)$) being removed could not define an insertion: no endpoint of the edge is within the partial sequence $\lbSeq$. If the node is required, a domain wipeout occurs.

\section{Filtering implementations}
\label{sec:filtering}

\subsection{TransitionTimes}

The filtering occurs in two steps. First, the bounds of the time windows of visited nodes are updated by iterating over the partial sequence $\lbSeq$:
\begin{align}
	\lfloor \cpvar{Time}_j \rfloor \gets & \max{\left(\lfloor \cpvar{Time}_j \rfloor, \lfloor \cpvar{Time}_i \rfloor + s_i + d_{i, j}\right)} & \quad \forall v_i \followedIn{} v_j \\
	\lceil \cpvar{Time}_i \rceil \gets  &\min{\left(\lceil \cpvar{Time}_i \rceil, \lceil \cpvar{Time}_j \rceil - s_i - d_{i, j}\right)} & \quad \forall v_i \followedIn{} v_j
\end{align}

This step can be implemented in \(\mathcal{O}(|\seq{s}|)\). Next, all insertions of a node \(v_j \in I\) between consecutive nodes \(v_i, v_k \in \lbSeq, v_i \followedIn{} v_k\) are inspected, similarly to line \ref{alg:distance:nested-loop} of Algorithm \ref{alg:distance}. Each insertion can be used to define \(\textit{ea}\) and \(\textit{la}\): the earliest and latest arrival at node \(v_j\), if performed.
\begin{align}
	\textit{ea} = \lfloor \cpvar{Time}_i \rfloor + s_i + d_{i,j} \\
	\textit{la} = \lceil \cpvar{Time}_k \rceil - s_j - d_{jk}
\end{align}
Insertions corresponding to time-window violations are removed:
\begin{equation}
		(\textit{ea} > \lceil \cpvar{Time}_j \rceil) \lor (\textit{la} < \lfloor \cpvar{Time}_j \rfloor) \lor (\textit{ea} > \textit{la}) \Longrightarrow
		\seq{S}.\mathrm{notBetween}(v_i, v_j, v_k)
	\label{eq:tt:insertion-filtering}
\end{equation}

In \eqref{eq:tt:insertion-filtering}, if either reaching node $v_j$ cannot be done within its time window, or if doing a detour through it would violate the time window of $v_k$, the insertion is removed. This shares some similarities with the work from \cite{savelsbergh1985local}, where time windows are used for checking the validity of moves in local search. Finally, a time window update is performed for nodes $v_j \in R \setminus \seq{s}$ that are required but not yet inserted:
\begin{align}
	\lfloor \cpvar{Time}_j \rfloor \gets &\max{\left(\lfloor \cpvar{Time}_j \rfloor, \min_{v_i \in \seq{S}.\mathrm{getInsert}(v_j)}{\lfloor \cpvar{Time}_i \rfloor + s_i + d_{i, j}}\right)}
	\label{eq:tt:tw_update_earliest} \\
	\lceil \cpvar{Time}_j \rceil \gets &  \min{\left(\lceil \cpvar{Time}_j \rceil, \max_{v_i \in \seq{S}.\mathrm{getInsert}(v_j), v_i \;\followedIn{}\; v_k}{\lceil \cpvar{Time}_k \rceil - s_j - d_{j, k}}\right)}
	\label{eq:tt:tw_update_latest}
\end{align}

The visit time of such nodes is updated based on their earliest predecessor and latest successor in \eqref{eq:tt:tw_update_earliest} and \eqref{eq:tt:tw_update_latest}, respectively. The time complexity of the filtering is the same as in the Distance constraint: $\mathcal{O}(|E|)$. Although we reason over a set of required nodes as in \citep{thomas2020insertion}, we do not ensure that a valid transition exists among all required nodes: this problem is NP-complete and would be too computationally expensive to perform at every filtering.

\subsection{Precedence}

The filtering consists of two main steps.
\begin{enumerate}
        \item First, it considers the partial sequence $\seq{s}$ and ensures that nodes belonging to $\seq{o} \cap \seq{s}$ appear in the same order in both $\seq{o}$ and $\seq{s}$. This step can be implemented in $\mathcal{O}(\max(|\seq{o}|, |\seq{s}|))$ and may trigger a failure.
	\item Next, it considers the insertable nodes from the ordering $\seq{o}$ to respect, removing insertions that would violate the ordering if performed. This is described in Algorithm~\ref{alg:precedence}. A set $Q$ tracks the nodes whose insertions will be filtered. The main loop iterates over each node $v_k \in \seq{o}$, as well as the last node $\last$. If $v_k$ is insertable (\textit{i.e.}, it is not yet in $\seq{s}$), it is added to $Q$ for potential filtering (line \ref{alg:precedence:node-to-filter}). On the contrary, if $v_k$ is already in $\seq{s}$, it serves as a boundary for the nodes in $Q$. For each node $v_j$ in $Q$, we enforce that it can only be inserted between $v_i$ and $v_k$, where:
	\begin{itemize}
		\item $v_i$ is the previous node in \(\seq{o}\) that also belongs to \(\seq{s}\), found in a previous iteration.
		\item $v_k$ (currently being iterated over) is the next node after \(v_j\) in \(\seq{o}\) that belongs to \(\lbSeq\).
	\end{itemize}
        To enforce this consistency, we forbid all sequences where $v_j$ appears before $v_i$ or after $v_k$ (line \ref{alg:precedence:detour-removal}), preserving only insertions consistent with the required precedence.
	
	\begin{algorithm}[!ht]
		\SetNoFillComment
		\caption{$\mathrm{Precedence}(\protect\seq{S}, \protect\seq{o})$ constraint filtering for invalid insertions.}
		\label{alg:precedence}
		$Q \gets \emptyset$  \;
		$v_i \gets \first$ \;
		\For{$v_k \in \seq{o} \cdot \last \label{alg:precedence:main-loop}$}{
			\uIf{$\seq{S}.\mathrm{isInsertable}(v_k)$}{
				$Q \gets Q \cup \{ v_k \}$ \label{alg:precedence:node-to-filter}\;
			}
			\uElseIf{$\seq{S}.\mathrm{isMember}(v_k)$}{
				\For{$v_j \in Q \label{alg:precedence:detour-removal}$ }{
					\tcp*[l]{$v_j$ can only be inserted between $v_i$ and $v_k$}
					$\seq{S}.\mathrm{notBetween}(\first, v_j, v_i)$ \label{alg:precedence:detour-removal1} \;
					$\seq{S}.\mathrm{notBetween}(v_k, v_j, \last)$ \label{alg:precedence:detour-removal2}\;
				}
				$Q \gets \emptyset$ \;
				$v_i \gets v_k$ \;
			}
		}
	\end{algorithm}
\end{enumerate}

The time complexity of the filtering is dominated by Algorithm \ref{alg:precedence}, running in $\mathcal{O}(|\seq{o}| \cdot |V|)$.

\subsection{Cumulative}

Firstly, to ensure that the start $p_i$ and the end $d_i$ of an activity $i \in A$ are visited together and in a valid order, two constraints are added per request. The constraint \(\nodeRequired{\seq{S}}{p_i} = \nodeRequired{\seq{S}}{d_i}\) ensures that the two nodes \(p_i, d_i\) appear together, and \(\mathrm{Precedence}(\seq{S}, (p_i \cdot d_i))\) ensures that the start \(p_i\) appears before the end \(d_i\). The remaining filtering consists of three steps: computing a load profile, filtering the partially inserted activities and then the non-inserted activities.

A lower bound on the load profile is computed based on the activities that are fully or partially inserted, representing the sum of resources consumed at each node. It is described by three values for each inserted node $v \in \seq{s}$:
\begin{itemize}
	\item \(\loadBefore{v}\) for the accumulated load before the visit of \(v\) ;
	\item \(\loadAt{v}\) for the accumulated load at the visit of \(v\) ;
	\item \(\loadAfter{v}\) for the accumulated load between the visit of \(v\) and its successor.
\end{itemize}

A load profile example is shown in Figure \ref{fig:cumulative-load-profile}. Using three values per node to represent the load profile is needed to accurately represent whether activities may be put between a node and its successor, as shown in Figure \ref{fig:cumulative-load-3-values}.

For each fully inserted activity $i$, the load of $i$ is added between the nodes: \(\forall p_i \prec v \preceq d_i : \loadBefore{v} \gets \loadBefore{v} + \bm{l}_i\), \(\forall p_i \preceq v \prec d_i : \loadAt{v} \gets \loadAt{v} + \bm{l}_i \land \loadAfter{v} \gets \loadAfter{v} + \bm{l}_i\). When considering only the inserted activities, it follows that \(\forall v \in \seq{s} : \loadAt{v} = \loadAfter{v}\). Those equalities no longer hold after considering the partially inserted activities.

For partially inserted activities with a start inserted, a node from the partial sequence $\seq{s}$ is used instead of the non-inserted end node to compute the load. It corresponds to the earliest node after the start node or the start node itself, after which the non-inserted end can be inserted. Partially inserted activities with the end inserted behave similarly, considering the latest node preceding the end, after which the start node can be inserted.

Setting an entry in \(\loadBefore{}, \loadAt{}\) or \(\loadAfter{}\) exceeding the capacity triggers a failure.

Given the load profile, we filter insertions for the partially inserted activities and the non-inserted activities. Algorithm \ref{alg:cumul_filtering} depicts the filtering for every non-inserted activity \(i \in A\) whose start \(p_i\) is inserted but not its corresponding end \(d_i\). It first finds the closest node \(v\) after which the end \(d_i\) can be inserted (line \ref{alg:cumul_filtering_find_closest}), triggering a failure if no such node exists (line \ref{alg:cumul_filtering_failure}). Note that this closest node \(v\) was already used to set the load of activity \(i\) during the load profile computation, and therefore is already marked as valid (line \ref{alg:get_next_v_before_iteration}). The next nodes to inspect are therefore the nodes following \(v\), in order. As soon as the capacity occurring at a node \(v\) does not allow inserting the end \(d_i\) of the activity, all insertions between this invalid node \(v\) and the end \(\last\) of the sequence are removed (line \ref{alg:cumul_filtering_exceeding_capa}). A similar filtering is performed in a mirror fashion, considering the non-inserted activities whose ends are inserted but not the corresponding start. 

Finally, the filtering from \citep{thomas2020insertion} is used for non-inserted activities, inspecting every start (and end) of activities, checking if a matching end (and start) can be found, and removing insertions when no match exists.

\begin{algorithm}[!ht]
    \caption{Filtering the $\mathrm{Cumulative}(\seq{S}, \bm{p}, \bm{d}, \bm{l}, c)$ constraint for activities whose starts are inserted.}
	\label{alg:cumul_filtering}
	\SetKwInput{Input}{Input}
	\SetKwRepeat{Do}{do}{while}
	\SetKw{MyIf}{if}
	\SetKw{MyThen}{then}
	\SetKw{MyBreak}{break}
	\SetKw{MyElse}{else}
	\SetKw{MyReturn}{return}
	\Input{$\seq{S}$: sequence variable, $\bm{p}, \bm{d}, \bm{l}$: start, end and load of activities, $c$: capacity.}
	\For{$i \in A \text{ \textbf{s.t.} } \left(\seq{S}.\mathrm{isMember}(p_i) \text{ \textbf{and not} } \seq{S}.\mathrm{isMember}(d_i)\right)$}{
		$v \gets p_i$\; 
		\While{\textbf{not} $\seq{S}.\mathrm{canInsert}(v, d_i)$ \label{alg:cumul_filtering_find_closest}}{
			$v \gets \seq{S}.\mathrm{getNext}(v)$ \;
			\If{$v = \last$}{
				\textbf{return failure} \label{alg:cumul_filtering_failure}\;
			}
		}
		$v \gets \seq{S}.\mathrm{getNext}(v)$ \label{alg:get_next_v_before_iteration}\;
		\While{$v \neq \last$}{
			\If{$\max(\loadBefore{v}, \loadAt{v}) + \bm{l}_i > c$ \label{alg:cumul_filtering_exceeding_capa}}{
				$\seq{S}.\mathrm{notBetween}(v, d_i, \last)$ \;
				\textbf{break} \;
			}
			$v \gets \seq{S}.\mathrm{getNext}(v)$ \;
		}
	}
\end{algorithm}

\begin{figure}[!ht]
    \centering
    \input{figures/cumulative-sequence}
    \caption{A sequence variable $\seq{S}$ with partial sequence  $\seq{s} = \first \cdot s_0 \cdot s_1 \cdot e_2 \cdot s_3 \cdot e_3 \cdot \last$ (top) and its corresponding load profile (bottom) when using a $\mathrm{Cumulative}(\seq{S}, (s_0, s_1, s_2, s_3), (e_0, e_1, e_2, e_3), (2, 1, 1, 3), 4)$ constraint. Edges between insertable nodes are not shown. For the partially inserted activity $1$, the closest node after which $e_1$ can be inserted is $e_2$; hence, rectangle 1 ends at $\loadAt{e_2}$ instead of $\loadAt{s_1}$ on the load profile.}
    \label{fig:cumulative-load-profile}
\end{figure}

\begin{figure}[!ht]
    \centering
    \input{figures/cumulative-3-values}
    \caption{Two load profiles for different sequences, with a capacity of 2. With $\loadAfter{s_0}$, we detect that an activity of load 2 can be put between $s_0$ and $e_1$ on the left but cannot be put between $s_0$ and $e_0$ on the right.}
    \label{fig:cumulative-load-3-values}
\end{figure}

\subsection{SubSequence}
\label{sec:subsequence-details}

The SubSequence constraint links two sequence variables $\seq{S}_m, \seq{S}_s$, ensuring that a subsequence $\seq{S}_s$ is contained within a supersequence (or master sequence) $\seq{S}_m$.

\begin{equation}
		\mathrm{SubSequence}(\seq{S}_m, \seq{S}_s) \leftrightarrow \left( v_i \in \seq{S}_s \implies v_i \in \seq{S}_m \right) \land
	    \left( \forall v_i, v_j \in \seq{S}_s: v_i \precIn{\seq{S}_s} v_j \implies v_i \precIn{\seq{S}_m} v_j \right)
\end{equation}

\paragraph{Filtering}

Firstly, to ensure that every node $v_i$ within the subsequence $\seq{S}_s$ is visited by the master sequence $\seq{S}_m$, implication constraints are used:

\begin{equation}
	\nodeRequired{\seq{S}_s}{i} \implies \nodeRequired{\seq{S}_m}{i} \quad \forall v_i \in V
\end{equation}

where $V$ is the set of nodes over which the subsequence $\seq{S}_s$ is defined.

Apart from those implications, the filtering is essentially the same as the one performed by the Precedence constraint from Section \ref{sec:precedence}.
The partial sequences $\seq{s}_m, \seq{s}_s$ of the two sequence variables $\seq{S}_m, \seq{S}_s$ are retrieved, and the filtering ensures that nodes in common to the two partial sequences appear in the same order:

\begin{equation}
	\forall v_i, v_j \in \seq{s}_s \cap \seq{s}_m : v_i \precIn{\seq{s}_s} v_j \iff v_i \precIn{\seq{s}_m} v_j
	\label{eq:subsequence-check}
\end{equation}

If \eqref{eq:subsequence-check} does not hold, a failure is triggered.
Then, filtering of detour edges is performed by calling twice Algorithm \ref{alg:precedence} from the Precedence constraint.
It is called once with $\seq{S} = \seq{S}_m$ and $ \seq{o} = \seq{s}_s$, enforcing precedences from the partial subsequence $\seq{s}_s$ onto the master sequence $\seq{S}_m$, and a second time with $\seq{S} = \seq{S}_s$ and $\seq{o} = \seq{s}_m$, enforcing precedences from the partial master sequence $\seq{s}_m$ onto the subsequence $\seq{S}_s$.

\section{DARP}
\label{sec:darp}

In both the sequence model and the successor model, the set of nodes $V$ and the distance matrix $\bm{d}$ are extended to introduce nodes representing the start and end nodes of the vehicles, which are duplicates of the depot. The values of $C_1, C_2$ in \eqref{eq:darp-insertion-cost} were taken from \citep{lnsffpa} and set to 80 and 1, respectively.

Search annotations were provided for the successor model, as they greatly improve the quality of solutions. They specify that branching first considers the successor variables \cpvar{Succ} using the DomWDeg heuristic \cite{boussemart2004boosting}, which prioritizes variables associated with small domains and a large number of encountered failures. The LNS consists of fixing 85\% of variables to their values in a previous solution and is restarted after a constant number of 10,000 search nodes.

\subsection{Results per instance}
\label{sec:darp:detailed-results}

Table \ref{tab:darp-landscape} shows the performance of the different solvers for each instance.
For the RF (2021) method from \cite{rist2021new}, we directly used the authors' implementation.
The solver managed to provide a proof of optimality within 57s, 14m25s and 14m30s on average on R1a, R1b and R7a, respectively, but failed to provide any feasible solution on the other instances as the enumeration of the fragments of paths timed out.

%\begin{table}[!ht]
%	\centering
%	\caption{Primal gap (\%) to the best known solution (bks) on DARP instances after 15 minutes of runtime. Values are averaged or minimized over 10 runs; “–” indicates that no run found a solution. Best results are in bold.}
%	\label{tab:darp-lns}
 %   \resizebox{\linewidth}{!}{%
%	\input{tables/darp_comparison}
%	}
%\end{table}

\input{tables/darp-landscape}

\section{Prize Collecting Scheduling Problem}
\label{sec:pcsp-detailed-results}

\paragraph{Instance selection.}
The benchmark set of \cite{froger2023new} contains two datasets, denoted
\texttt{A\_E} (Avionic) and \texttt{T} (Therapy). For every combination of
dataset, number of tasks, and number of secondary resources, 30 instances are
available, identified by a numerical suffix ranging from \texttt{\_\_000} to
\texttt{\_\_029}. The Avionic instances use $m\in\{3,4\}$ secondary resources,
whereas the Therapy instances use $m\in\{2,3\}$.

We retained the task counts
$n\in\{50,100,150,200,250,300,400,500\}$. Instances with 350 and 450 tasks
were excluded because the reference results used to compute the primal gaps
do not provide best-known values for these sizes. Within each
$(\text{dataset},n,m)$ subgroup, instances were ordered by their numerical
suffix and three approximately evenly spaced instances were selected:
\texttt{\_\_000}, \texttt{\_\_014}, and \texttt{\_\_029}. The selection
therefore depends only on the instance identifiers and was performed
independently of the results obtained by the compared methods. This gives
$2 \times 8 \times 2 \times 3 = 96$ instances.

Table \ref{tab:pcsp-landscape} shows the performance of the different solvers for each instance size.
Each row corresponds to the average across 3 instances.
We compare to the best known solution from \cite{froger2023new}.
One should note that the CPO model managed to improve the best known solutions on 10 instances from the T dataset.

%\begin{table}[!ht]
%	\centering
%	\caption{Primal gap (\%) to the best known solution (bks) on PCSP instances after 15 minutes of runtime. Values are averaged or minimized over 10 runs; “–” indicates that no run found a solution. Best results are in bold.}
%	\label{tab:pcsp}
%	\input{tables/pcsp_comparison}
%\end{table}

\input{tables/pcsp-landscape}

\end{document}

%% file: commands.tex
\NewDocumentCommand{\domain}{g}{%
	\IfNoValueTF{#1}
	{\ensuremath{\mathcal{D}}}%   % What to do if no argument was given
	{\ensuremath{\mathcal{D}(#1)}}% What to do if an argument was given
}

\newcommand{\cpvar}[1]{\textbf{\textit{#1}}}

\DeclareRobustCommand{\seq}[1]{%
	\ensuremath{\scaleobj{.85}{\overrightarrow{#1}}}%
}

\newcommand{\precIn}[1]{% \precIn{<seq name>}
	\ensuremath{\mathrel{
			\overset{
                \count0=\fam\mbox{$\fam=\count0 \scriptstyle#1$}  % shamelessly inspired by https://tex.stackexchange.com/questions/211314/make-box-in-math-mode-with-current-math-font 
			}{
				\prec
			}
	}}
}

\newcommand{\preceqIn}[1]{% \preceqIn{<seq name>}
	\ensuremath{\mathrel{
			\overset{
				 \count0=\fam\mbox{$\fam=\count0 \scriptstyle#1$}
			}{
				\preceq
			}
	}}
}

\newcommand{\followedIn}[1]{% \followedIn{<seq name>}
	\ensuremath{
		\xrightarrow{ #1}
	}
}

\newcommand{\notFollowedIn}[1]{%
	\ensuremath{
		\not{\xrightarrow{#1}}
	}
}

\newcommand{\singleDerives}[1]{% \singleDerives{<insert set name>}
	\ensuremath{\mathrel{
			\underset{
				\scriptstyle#1
			}{
				\longmapsto
			}
	}}
}

\newcommand{\first}{\alpha}

\newcommand{\last}{\omega}

\newcommand{\predset}[2]{% \predset{<type>}{<subscript>}
	\ensuremath{
		{#1}_{ #2}^{-}
	}
}

\newcommand{\succset}[2]{% \succset{<type>}{<subscript>}
	\ensuremath{
		{#1}_{ #2}^{+}
	}
}

\newcommand{\predecessors}[1]{% \predecessors{<node>}
    \predset{E}{#1}
}

\newcommand{\successors}[1]{% \successors{<node>}
    \succset{E}{#1}
}

\newcommand{\directPredecessor}[1]{% \directPredecessor{<node>}
    \predset{\mathbf{s}}{#1}
}

\newcommand{\directSuccessor}[1]{% \directSuccessor{<node>}
    \succset{\mathbf{s}}{#1}
}

\newcommand{\pick}[1]{% \pick{<stop name>}
	\ensuremath{
		{#1}^+
	}
}

\newcommand{\drop}[1]{% \drop{<stop name>}
	\ensuremath{
		{#1}^-
	}
}

\newcommand{\nodeRequired}[2]{% \nodeRequired{<sequence>}{<node>}
	\ensuremath{
		\mathcal{R}_{#2}{(#1)}
	}
}

\newcommand{\loadAt}[1]{% \loadAt{<node>}
	\ensuremath{
		\bm{l}_{#1}^+
	}
}

\newcommand{\loadBefore}[1]{% \loadAt{<node>}
	\ensuremath{
		\bm{l}_{#1}^-
	}
}

\newcommand{\loadAfter}[1]{% \loadAfter{<node>}
	\ensuremath{
        \bm{l}_{#1}^{#1 +1}
	}
}

\newcommand{\extDom}{% \extDom{variable}
	\ensuremath{
		\mathcal{D}%^{\ast}
	}
}

\newcommand*\mcircBigger[1]{\text{\raisebox{.1ex}{\scalebox{1.0}{\textcircled{\scriptsize #1}}}}}

\newcommand{\extDomS}[1]{% \extDom{variable}
	\ensuremath{
		\mathcal{D}^{\mcircBigger{s}}(#1)
	}
}
\newcommand{\extDomR}[1]{% \extDom{variable}
	\ensuremath{
		\mathcal{D}^{\mcircBigger{r}}(#1)
	}
}
\newcommand{\extDomX}[1]{% \extDom{variable}
	\ensuremath{
		\mathcal{D}^{\mcircBigger{x}}(#1)
	}
}
\newcommand{\extDomF}[1]{% \extDom{variable}
	\ensuremath{
		\mathcal{D}^{\mcircBigger{\raisebox{-0.0ex}{{n}}}}(#1)
	}
}

\newcommand{\lbSeq}{% \lb{variable}
	\ensuremath{
		\seq{s}
	}
}

\newcommand{\forbiddenInsertions}[1]{% \extDom{variable}
	\ensuremath{
		\mathcal{I}^{\mcircBigger{\raisebox{-0.2ex}{f}}}(#1)
	}
}

\newcommand{\NB}{\ensuremath{N}}

\newcommand{\NBR}{\ensuremath{\hat{N}}}

\SetCommentSty{mycommfont}

\DontPrintSemicolon
\SetArgSty{}
\SetNlSty{bfseries}{\color{black}}{}  % impose black numbers for margin in algorithm2e

\newcommand{\PreserveBackslash}[1]{\let\temp=\\#1\let\\=\temp}
\newcolumntype{C}[1]{>{\PreserveBackslash\centering}p{#1}}
\newcolumntype{R}[1]{>{\PreserveBackslash\raggedleft}p{#1}}
\newcolumntype{L}[1]{>{\PreserveBackslash\raggedright}p{#1}}
\usetikzlibrary{positioning, fit, arrows.meta, shapes}
\usetikzlibrary{arrows}
\usetikzlibrary{decorations.markings}
\usetikzlibrary{calc}
\usetikzlibrary{decorations.pathreplacing}

\definecolor{removed}{HTML}{D55E00}
\definecolor{member}{HTML}{56B4E9}
\definecolor{required}{HTML}{009E73}
\definecolor{possible}{HTML}{E69F00}
\definecolor{removing}{HTML}{CC79A7}

\definecolor{color1}{HTML}{D55E00}
\definecolor{color2}{HTML}{56B4E9}
\definecolor{color3}{HTML}{E69F00}
\definecolor{color4}{HTML}{009E73}

\definecolor{orangeC}{HTML}{E69F00}
\definecolor{skyBlueC}{HTML}{56B4E9}
\definecolor{bluishGreenC}{HTML}{009E73}
\definecolor{yellowC}{HTML}{F0E442}
\definecolor{blueC}{HTML}{0072B2}
\definecolor{vermilionC}{HTML}{D55E00}
\definecolor{reddishPurpleC}{HTML}{CC79A7}

\DeclareFontFamily{U}{mathb}{\hyphenchar\font45}
\DeclareFontShape{U}{mathb}{m}{n}{
<-6> mathb5 <6-7> mathb6 <7-8> mathb7
<8-9> mathb8 <9-10> mathb9
<10-12> mathb10 <12-> mathb12
}{}

\newcolumntype{M}{>{\centering\arraybackslash}m{1cm}}

%% file: figures/weak-seqvar-representation.tex
	\begin{tikzpicture}[node distance=1.0cm and 1.0cm,
	% GLOBAL CFG
	>=LaTeX,
	% Styles
	member/.style={
		minimum size=0.5cm,
		circle, 
		inner sep=1.5pt,
		fill=member,
	},
	possible/.style={
		minimum size=0.5cm,
		circle, 
		inner sep=1.5pt,
		fill=possible,
	},
	required/.style={
		minimum size=0.5cm,
		circle, 
		inner sep=1.5pt,
		fill=required,
	},
	excluded/.style={
		minimum size=0.5cm,
		circle, 
		inner sep=1.5pt,
		fill=removed,
	},
	insertionDirected/.style={
		densely dashed,
		gray!80!white,
		very thick,
		-{Stealth[gray!80!white]}
	},
	insertion/.style={
		densely dashed,
		gray!80!white,
		very thick,
		{Stealth[gray!80!white]}-{Stealth[gray!80!white]}
	},
	successor/.style={
		very thick,
		-{Stealth[black]}
	},
	gridEntry/.style={
		rectangle,
		draw,
		minimum height=5.5mm,
		minimum width=5.5mm,
		outer sep=0pt,
		text height = 1.5mm
	},
	index/.style={
		rectangle,
		minimum height=5.5mm,
		minimum width=5.5mm,
		text height = 0mm
	},
	]
	
	\node [member](first) {$\first$};
	\node [member, right =of first] (v1) {$v_1$};
	\node [member, right =of v1] (last) {$\last$};
	
	\node [excluded, below =of first] (v2) {$v_2$};
	\node [possible, below =of v1] (v3) {$v_3$};
	\node [possible, below =of last] (v4) {$v_4$};
	
	\draw [successor] (first) -- (v1);
	\draw [successor] (v1) -- (last);
	\draw [successor, bend right] (last) edge (first);
	
	\draw [insertionDirected] (first) -- (v3);
	\draw [insertion] (v3) -- (v1);
	\draw [insertion] (v3) -- (v4);
	\draw [insertionDirected] (v1) -- (v4);
	\draw [insertionDirected] (v3) -- (last);
	\draw [insertionDirected] (v4) -- (last);
	
	\node [below = 0.1cm of v3] (I) {
		\begin{tabular}{c|ccccc}
			& $\predecessors{}$ & $\successors{}$ & $nI$ & $\directPredecessor{}$ & $\directSuccessor{}$ \\
			\hline
			\textcolor{member}{$\bm{\first}$} & $\{ \last \}$ & $\{ v_1, v_3 \}$ & 0 & $\last$ & $v_1$ \\
			\textcolor{member}{$\bm{v_1}$} & $\{ \first, v_3\}$ & $ \{ v_3, v_4, \last\}$ & 0 & $\first$ & $\last$\\
			\textcolor{removed}{$\bm{v_2}$} & $ \emptyset $ & $\emptyset $ & 0 & $v_2$ & $v_2$\\
			\textcolor{possible}{$\bm{v_3}$} & $\{ \first, v_1, v_4\}$ & $\{v_1, v_4, \last\}$ & 2 & $v_3$ & $v_3$\\
			\textcolor{possible}{$\bm{v_4}$} & $\{v_1, v_3\}$ & $\{ v_3, \last\}$ & 1 & $v_4$ & $v_4$\\
			\textcolor{member}{$\bm{\last}$} & $\{v_1, v_3, v_4\}$ & $\{ \first\}$ & 0 & $v_1$ & $\first$\\
		\end{tabular}
	};
	
	%\node[gridEntry, below right =4.5cm and -0.15cm of v3] (sv4) {$v_4$};
	%\node[gridEntry, left = 0cm of sv4] (sv3) {$v_3$};
	%\node[gridEntry, left = 0cm of sv3] (sv2) {$v_2$};
	%\node[gridEntry, fill=member, left = 0cm of sv2] (sv1) {$\last$};
	%\node[gridEntry, fill=member, right = 0cm of sv4] (s-first) {$v_1$};
	%\node[gridEntry, fill=member, right = 0cm of s-first] (s-last) {$\first$};
	
	%\node[index, below=0.1cm of sv1] (sv1i) {$v_1$};
	%\node[index, below=0.1cm of sv2] (sv2i) {$v_2$};
	%\node[index, below=0.1cm of sv3] (sv3i) {$v_3$};
	%\node[index, below=0.1cm of sv4] (sv4i) {$v_4$};
	%\node[index, below=0.1cm of s-first] (s-first-i) {$\first$};
	%\node[index, below=0.1cm of s-last] (s-last-i) {$\last$};
	
	%\node[left =0cm of sv1] (succ) {$\directSuccessor{}$};

	%\node[gridEntry, below = 0.2cm of sv4i] (pv4) {$v_4$};
	%\node[gridEntry, left = 0cm of pv4] (pv3) {$v_3$};
	%\node[gridEntry, left = 0cm of pv3] (pv2) {$v_2$};
	%\node[gridEntry, fill=member, left = 0cm of pv2] (pv1) {$\first$};
	%\node[gridEntry, fill=member, right = 0cm of pv4] (p-first) {$\last$};
	%\node[gridEntry, fill=member, right = 0cm of p-first] (p-last) {$v_1$};
	
	%\node[index, below=0.1cm of pv1] (pv1i) {$v_1$};
	%\node[index, below=0.1cm of pv2] (pv2i) {$v_2$};
	%\node[index, below=0.1cm of pv3] (pv3i) {$v_3$};
	%\node[index, below=0.1cm of pv4] (pv4i) {$v_4$};
	%\node[index, below=0.1cm of p-first] (p-first-i) {$\first$};
	%\node[index, below=0.1cm of p-last] (p-last-i) {$\last$};
	
	%\node[left =0cm of pv1] (pred) {$\directPredecessor{}$};
	\node[below =0.2cm of I, align=center] (sets-1) {$nS = 3, I = \{ v_3, v_4\}, R = \{ \first, \last, v_1\}, X = \{ v_2\}$};
	
\end{tikzpicture}

%% file: figures/weak-seqvar-representation-after-insert.tex
	\begin{tikzpicture}[node distance=1.0cm and 1.0cm,
		% GLOBAL CFG
		>=LaTeX,
		% Styles
		member/.style={
			minimum size=0.5cm,
			circle, 
			inner sep=1.5pt,
			fill=member,
		},
		possible/.style={
			minimum size=0.5cm,
			circle, 
			inner sep=1.5pt,
			fill=possible,
		},
		required/.style={
			minimum size=0.5cm,
			circle, 
			inner sep=1.5pt,
			fill=required,
		},
		excluded/.style={
			minimum size=0.5cm,
			circle, 
			inner sep=1.5pt,
			fill=removed,
		},
		insertionDirected/.style={
			densely dashed,
			gray!80!white,
			very thick,
			-{Stealth[gray!80!white]}
		},
		insertion/.style={
			densely dashed,
			gray!80!white,
			very thick,
			{Stealth[gray!80!white]}-{Stealth[gray!80!white]}
		},
		successor/.style={
			very thick,
			-{Stealth[black]}
		},
		gridEntry/.style={
			rectangle,
			draw,
			minimum height=5.5mm,
			minimum width=5.5mm,
			outer sep=0pt,
			text height = 1.5mm
		},
		index/.style={
			rectangle,
			minimum height=5.5mm,
			minimum width=5.5mm,
			text height = 0mm
		},
		]
		
		\node [member](first) {$\first$};
		\node [member, right =of first] (v1) {$v_1$};
		\node [member, right =of v1] (last) {$\last$};
		
		\node [excluded, below =of first] (v2) {$v_2$};
		\node [member, below =of v1] (v3) {$v_3$};
		\node [possible, below =of last] (v4) {$v_4$};
		
		\draw [successor] (first) -- (v3);
		\draw [successor] (v3) -- (v1);
		\draw [successor] (v1) -- (last);
		\draw [successor, bend right] (last) edge (first);
		
		\draw [insertionDirected] (v1) -- (v4);
		\draw [insertionDirected] (v4) -- (last);
		
		\node [below = 0.1cm of v3] (I) {
			\begin{tabular}{c|ccccc}
				& $\predecessors{}$ & $\successors{}$ & $nI$ & $\directPredecessor{}$ & $\directSuccessor{}$ \\
				\hline
				\textcolor{member}{$\bm{\first}$} & \makebox[\widthof{$\{v_1, v_3, v_4\}$}][c]{$\{ \last \}$} & \makebox[\widthof{$\{ v_3, v_4, \last\}$}][c]{$\{ v_3 \}$} & 0 & $\last$ & $v_3$ \\
				\textcolor{member}{\bm{$v_1$}} & $\{ v_3\}$ & $ \{ v_4, \last\}$ & 0 & $v_3$ & $\last$\\
				\textcolor{removed}{$\bm{v_2}$} & $ \emptyset $ & $\emptyset $ & 0 & $v_2$ & $v_2$ \\
				\textcolor{member}{$\bm{v_3}$} & $\{ \first \}$ & $\{ v_1 \}$ & 0 & $\first$ & $v_1$\\
				\textcolor{possible}{$\bm{v_4}$} & $\{v_1 \}$ & $\{ \last\}$ & 1& $v_4$ & $v_4$\\
				\textcolor{member}{$\bm{\last}$} & $\{v_1, v_4\}$ & $\{ \first\}$ & 0 & $v_1$ & $\first$\\
			\end{tabular}
		}; 
		
		%\node[gridEntry, below right =4.5cm and -0.15cm of v3] (sv4) {$v_4$};
		%\node[gridEntry, fill=member, left = 0cm of sv4] (sv3) {$v_1$};
		%\node[gridEntry, left = 0cm of sv3] (sv2) {$v_2$};
		%\node[gridEntry, fill=member, left = 0cm of sv2] (sv1) {$\last$};
		%\node[gridEntry, fill=member, right = 0cm of sv4] (s-first) {$v_3$};
		%\node[gridEntry, fill=member, right = 0cm of s-first] (s-last) {$\first$};
		
		%\node[index, below=0.1cm of sv1] (sv1i) {$v_1$};
		%\node[index, below=0.1cm of sv2] (sv2i) {$v_2$};
		%\node[index, below=0.1cm of sv3] (sv3i) {$v_3$};
		%\node[index, below=0.1cm of sv4] (sv4i) {$v_4$};
		%\node[index, below=0.1cm of s-first] (s-first-i) {$\first$};
		%\node[index, below=0.1cm of s-last] (s-last-i) {$\last$};
		
		%\node[left =0cm of sv1] (succ) {$\directSuccessor{}$};

		%\node[gridEntry, below = 0.2cm of sv4i] (pv4) {$v_4$};
		%\node[gridEntry, fill=member, left = 0cm of pv4] (pv3) {$\first$};
		%\node[gridEntry, left = 0cm of pv3] (pv2) {$v_2$};
		%\node[gridEntry, fill=member, left = 0cm of pv2] (pv1) {$v_3$};
		%\node[gridEntry, fill=member, right = 0cm of pv4] (p-first) {$\last$};
		%\node[gridEntry, fill=member, right = 0cm of p-first] (p-last) {$v_1$};
		
		%\node[index, below=0.1cm of pv1] (pv1i) {$v_1$};
		%\node[index, below=0.1cm of pv2] (pv2i) {$v_2$};
		%\node[index, below=0.1cm of pv3] (pv3i) {$v_3$};
		%\node[index, below=0.1cm of pv4] (pv4i) {$v_4$};
		%\node[index, below=0.1cm of p-first] (p-first-i) {$\first$};
		%\node[index, below=0.1cm of p-last] (p-last-i) {$\last$};
		
		%\node[left =0cm of pv1] (pred) {$\directPredecessor{}$};
		\node[below =0.2cm of I, align=center] (sets-1) {$nS = 4, I = \{ v_4\}, R = \{ \first, \last, v_1, v_3\}, X = \{ v_2\}$};
		
	\end{tikzpicture}

%% file: figures/cumulative-example-fixed.tex
	\definecolor{colorActivity0}{HTML}{D55E00}
	\definecolor{colorActivity1}{HTML}{56B4E9}
	\definecolor{colorActivity2}{HTML}{009E73}
	\definecolor{colorActivity3}{HTML}{E69F00}
	
	\begin{tikzpicture}[node distance=0.8cm and 0.8cm,
		% GLOBAL CFG
		>=LaTeX,
		% Styles
		member/.style={
			circle, 
			inner sep=1.5pt,
			minimum size=6mm,
		},
		possible/.style={
			circle, 
			inner sep=1.5pt,
			minimum size=6mm,
		},
		excluded/.style={
			circle, 
			inner sep=1.5pt,
			minimum size=6mm,
		},
		insertionDirected/.style={
			densely dashed,
			gray!80!white,
			very thick,
			-{Stealth[gray!80!white]}
		},
		insertion/.style={
			densely dashed,
			gray!80!white,
			very thick,
			{Stealth[gray!80!white]}-{Stealth[gray!80!white]}
		},
		successor/.style={
			very thick,
			-{Stealth[black]}
		},
		xtick/.style={
			rectangle,
			minimum height=0.6cm,
		},
		]
		\newcommand{\distanceBelowGraph}{2cm}
		
		% ========= sequence part =========
		
		\node [member](start) {$\first$};
		\node [member, right= of start, fill=colorActivity0](s0) {$s_0$};
		\node [member, right= of s0, fill=colorActivity1](s1) {$s_1$};
		\node [member, right= of s1, fill=colorActivity1](e1) {$e_1$};
		\node [member, right= of e1, fill=colorActivity0](e0) {$e_0$};
		\node [member, right= of e0, fill=colorActivity3](s3) {$s_3$};
		\node [member, right= of s3, fill=colorActivity3](e3) {$e_3$};
		\node [member, right= of e3](end) {$\last$};
		
		\node [possible, right= 1.2cm of end, fill=colorActivity2] (s2) {$s_2$};
		\node [possible, right= of s2, fill=colorActivity2] (e2) {$e_2$};
		
		\draw [successor] (start) edge (s0);
		\draw [successor] (s0) edge (s1);
		\draw [successor] (s1) edge (e1);
		\draw [successor] (e1) edge (e0);
		\draw [successor] (e0) edge (s3);
		\draw [successor] (s3) edge (e3);
		\draw [successor] (e3) edge (end);
		
		% ========= graph part =========
		
		% label for load at a given node
		\node[below = \distanceBelowGraph of start, xtick, anchor=north] (Start) {${\first}$};
		\node[below = \distanceBelowGraph of s0, xtick, anchor=north] (LoadS0) {${s_0}$};
		\node[below = \distanceBelowGraph of s1, xtick, anchor=north] (LoadS1) {${s_1}$};
		\node[below = \distanceBelowGraph of e1, xtick, anchor=north] (LoadE1) {${e_1}$};
		\node[below = \distanceBelowGraph of e0, xtick, anchor=north] (LoadE0) {${e_0}$};
		\node[below = \distanceBelowGraph of s3, xtick, anchor=north] (LoadS3) {${s_3}$};
		\node[below = \distanceBelowGraph of e3, xtick, anchor=north] (LoadE3) {${e_3}$};
		\node[below = \distanceBelowGraph of end, xtick, anchor=north] (End) {${\last}$};
		
		% load from activity 0
		\coordinate[above left = 1cm and -0.1cm of LoadS0] (A0-begin);
		\coordinate[above left = 0cm and 0cm of LoadE0] (A0-end);
		\fill [fill=colorActivity0, fill opacity=1] (A0-begin) rectangle (A0-end);
		% Calculate the midpoint for the label
		\coordinate (A0-center) at ($ (A0-begin)!0.5!(A0-end) $);
		% Place the label at the midpoint
		\node at (A0-center) {$0$};
		%
		% load from activity 1
		\coordinate[above left = 1.5cm and -0.1cm of LoadS1] (A1-begin);
		\coordinate[above left = 1cm and 0cm of LoadE1] (A1-end);
		\fill [fill=colorActivity1, fill opacity=1] (A1-begin) rectangle (A1-end);
		% Calculate the midpoint for the label
		\coordinate (A1-center) at ($ (A1-begin)!0.5!(A1-end) $);
		% Place the label at the midpoint
		\node at (A1-center) {$1$};
		%
		% load from activity 3 
		\coordinate[above left = 1cm and -0.1cm of LoadS3] (A3-begin);
		\coordinate[above left= 0cm and 0cm of LoadE3] (A3-end);
		\fill [fill=colorActivity3, fill opacity=1] (A3-begin) rectangle (A3-end);
		% Calculate the midpoint for the label
		\coordinate (A3-center) at ($ (A3-begin)!0.5!(A3-end) $);
		% Place the label at the midpoint
		\node at (A3-center) {$3$};
		%
		% x axis
		\coordinate[above left = 0cm and 0cm of Start] (origin);
		\coordinate[above right = 0cm and 0.5cm of End] (endingSequence);
		\draw[->, thick] (origin) -- (endingSequence);
		\node[right=0cm of endingSequence] (label-sequence) {$\seq{S}$};
		%
		% y axis
    	\coordinate[above = 2.0cm of origin] (endingCapacity);
    	\draw[->, thick] (origin) -- (endingCapacity);
    	\coordinate[above= 1.5cm of origin] (coord3);
    	\coordinate[above= 1cm of origin] (coord2);
    	\coordinate[above= 0.5cm of origin] (coord1);
    	\node[left=0.1cm of coord3] (label3) {$3$};
    	\node[left=0.1cm of coord2] (label2) {$2$};
    	\node[left=0.1cm of coord1] (label1) {$1$};
    	\draw (coord3) -- (label3);
    	\draw (coord2) -- (label2);
    	\draw (coord1) -- (label1);
    	\node[rectangle, minimum size=1cm, left =0.5cm of label2, rotate=90, xshift=0.5cm, yshift=-0.2cm] (label-load) {Load};
		
	\end{tikzpicture}

%% file: figures/sequence-branching.tex
	\begin{tikzpicture}[scale=1.2, node distance=0.8cm and 0.2cm,
		every node/.style={scale=0.9},
		state/.style={
			rectangle, 
			inner sep=1.5pt,
			minimum size=6mm,
		},]
	
	    % === top part ===
	
		% bottom layer
		\node[state] (v1v2v3A) {$v_1 v_2 v_3$};
		\node[state, right =of v1v2v3A] (v1v3v2A) {$v_1 v_3 v_2$};
		\node[state, right =of v1v3v2A] (v3v1v2A) {$v_3 v_1 v_2$};
		\node[state, right =of v3v1v2A] (v2v1v3A) {$v_2 v_1 v_3$};
		\node[state, right =of v2v1v3A] (v2v3v1A) {$v_2 v_3 v_1$};
		\node[state, right =of v2v3v1A] (v3v2v1A) {$v_3 v_2 v_1$};
		
		% middle layer
		\node[state, above =of v1v3v2A] (v1v2A) {$v_1 v_2$};
		\node[state, above =of v3v1v2A] (v1v3A) {$v_1 v_3$};
		\node[state, above =of v2v1v3A] (v3v1A) {$v_3 v_1$};
		\node[state, above =of v2v3v1A] (v2v1A) {$v_2 v_1$};
		\coordinate (middleLayerCenterA) at ($ (v1v3A)!0.5!(v3v1A) $);
		
		% top layer
		\node[state, above=of middleLayerCenterA] (v1A) {$v_1$};
		
		% top to middle edges
		\draw[->] (v1A) edge (v1v2A);
		\draw[->] (v1A) edge (v2v1A);
		\draw[->] (v1A) edge (v3v1A);
		\draw[->] (v1A) edge (v1v3A);
		
		% middle to bottom edges
		\draw[->] (v1v2A) edge (v1v2v3A);
		\draw[->] (v1v2A) edge (v1v3v2A);
		\draw[->] (v1v2A) edge (v3v1v2A);
		
		\draw[->] (v1v3A) edge (v2v1v3A);
		\draw[->] (v1v3A) edge (v1v2v3A);
		\draw[->] (v1v3A) edge (v1v3v2A);
		
		\draw[->] (v2v1A) edge (v3v2v1A);
		\draw[->] (v2v1A) edge (v2v3v1A);
		\draw[->] (v2v1A) edge (v2v1v3A);
		
		\draw[->] (v3v1A) edge (v2v3v1A);
		\draw[->] (v3v1A) edge (v3v2v1A);
		\draw[->] (v3v1A) edge (v3v1v2A);
		
		% === bottom part ===

		% bottom layer
		\node[state, right= 7.5cm of v1v2v3A] (v1v2v3B) {$v_1 v_2 v_3$};
		\node[state, right =of v1v2v3B] (v1v3v2B) {$v_1 v_3 v_2$};
		\node[state, right =of v1v3v2B] (v3v1v2B) {$v_3 v_1 v_2$};
		\node[state, right =of v3v1v2B] (v2v1v3B) {$v_2 v_1 v_3$};
		\node[state, right =of v2v1v3B] (v2v3v1B) {$v_2 v_3 v_1$};
		\node[state, right =of v2v3v1B] (v3v2v1B) {$v_3 v_2 v_1$};
		
		% middle layer
		\node[state, above =of v1v3v2B] (v1v2B) {$v_1 v_2$};
		\node[state, above =of v2v3v1B] (v2v1B) {$v_2 v_1$};
		\coordinate (middleLayerCenterB) at ($ (v1v2B)!0.5!(v2v1B) $);
		
		% top layer
		\node[state, above=of middleLayerCenterB] (v1B) {$v_1$};
		
		% top to middle edges
		\draw[->] (v1B) edge (v1v2B);
		\draw[->] (v1B) edge (v2v1B);
		
		% middle to bottom edges
		\draw[->] (v1v2B) edge (v1v2v3B);
		\draw[->] (v1v2B) edge (v1v3v2B);
		\draw[->] (v1v2B) edge (v3v1v2B);
		
		\draw[->] (v2v1B) edge (v2v1v3B);
		\draw[->] (v2v1B) edge (v2v3v1B);
		\draw[->] (v2v1B) edge (v3v2v1B);

	\end{tikzpicture}
	

%% file: figures/darp-branching-color.tex
	\begin{tikzpicture}[node distance=1cm and 1cm,
	% GLOBAL CFG
	>=LaTeX,
	every node/.style={scale=0.9},
	% Styles
	depot/.style={
		minimum size=0.35cm,
		rectangle, 
		inner sep=1.5pt,
		fill=black,
	},
	nodeVisited/.style={
		minimum size=0.3cm,
		rectangle, 
	},
	nodeNotVisited/.style={
		minimum size=0.3cm,
		rectangle, 
	},
	succ/.style={
		very thick,
		-{Stealth[black]}
	},
	legend/.style={draw, thick, rounded corners, gray!80!white},
	boundingBox/.style={
		rounded corners=1mm,
		draw,
		thick,
		gray,
	},
	transition/.style={
		line width=1.0mm,
		draw,
		gray,
		->,
	},
	myPackage/.style={
		packageNew, w=0.05, draw=black, minimum width=0.4cm,
	},
	myHouse/.style={
		house, draw=black, minimum width=0.4cm, line width=0.75pt,
	},
	myPackageA/.style={
		myPackage, fill=skyBlueC,
	},
	myHouseA/.style={
		myHouse, fill=skyBlueC,
	},
	myPackageB/.style={
		myPackage, fill=orangeC,
	},
	myHouseB/.style={
		myHouse, fill=orangeC,
	},
	]
	
	% child A
	\node[depot] (depotBA) {};
	\node[myPackageA, above right = 1.35cm and 0.5cm of depotBA, anchor=center] (pR1BA) {};
	\node[myHouseA, above right = 0cm and 2cm  of depotBA, anchor=center] (dR1BA) {};
	\node[myPackageB, above left = 1.0cm and 0.3cm of depotBA, anchor=center] (pR2BA) {};
	\node[myHouseB, above right = 1.4cm and 1.8cm  of depotBA, anchor=center] (dR2BA) {};
	
	\coordinate (pR1BACenter) at ([xshift=0.2cm, yshift=0.2cm]pR1BA.center);
	\coordinate (dR1BACenter) at ([xshift=0.2cm, yshift=0.2cm]dR1BA.center);
	\coordinate (pR2BACenter) at ([xshift=0.2cm, yshift=0.2cm]pR2BA.center);
	\coordinate (dR2BACenter) at ([xshift=0.2cm, yshift=0.2cm]dR2BA.center);
	
	\node[] at ([xshift=0.1cm, yshift=-0.5cm]pR1BACenter) (pR1BALabel) {$\pick{r_1}$};
	\node[] at ([xshift=0.05cm, yshift=-0.45cm]dR1BACenter) (dR1BALabel) {$\drop{r_1}$};
	\node[] at ([xshift=-0.0cm, yshift=0.55cm]pR2BACenter) (pR2BALabel) {$\pick{r_2}$};
	\node[] at ([xshift=0.2cm, yshift=0.45cm]dR2BACenter) (dR2BALabel) {$\drop{r_2}$};
	
	\draw[succ, shorten >=0.30cm] (depotBA) to (pR2BACenter);
	\draw[succ, shorten >=0.30cm, shorten <=0.30cm] (pR2BACenter) to (pR1BACenter);
	\draw[succ, shorten >=0.30cm, shorten <=0.30cm] (pR1BACenter) to (dR2BACenter);
	\draw[succ, shorten >=0.30cm, shorten <=0.30cm] (dR2BACenter) to (dR1BACenter);
	\draw[succ, shorten <=0.30cm] (dR1BACenter) to (depotBA);
	
	%\draw[decoration={brace, raise=2.9cm},decorate, thick, gray!80!white]
	%([xshift=-0.5cm]depotBA.west) -- node[above=3cm] {\textcolor{black}{1}} ([xshift=2.8cm]depotBA.east);
	
	% child B
	\node[depot, right = 3.4cm of depotBA] (depotBB) {};
	\node[myPackageA, above right = 1.35cm and 0.5cm of depotBB, anchor=center] (pR1BB) {};
	\node[myHouseA, above right = 0cm and 2cm  of depotBB, anchor=center] (dR1BB) {};
	\node[myPackageB, above left = 1.0cm and 0.3cm of depotBB, anchor=center] (pR2BB) {};
	\node[myHouseB, above right = 1.4cm and 1.8cm  of depotBB, anchor=center] (dR2BB) {};
	
	\coordinate (pR1BBCenter) at ([xshift=0.2cm, yshift=0.2cm]pR1BB.center);
	\coordinate (dR1BBCenter) at ([xshift=0.2cm, yshift=0.2cm]dR1BB.center);
	\coordinate (pR2BBCenter) at ([xshift=0.2cm, yshift=0.2cm]pR2BB.center);
	\coordinate (dR2BBCenter) at ([xshift=0.2cm, yshift=0.2cm]dR2BB.center);
	
	\node[] at ([xshift=0.1cm, yshift=-0.5cm]pR1BBCenter) (pR1BBLabel) {$\pick{r_1}$};
	\node[] at ([xshift=0.05cm, yshift=-0.45cm]dR1BBCenter) (dR1BBLabel) {$\drop{r_1}$};
	\node[] at ([xshift=-0.0cm, yshift=0.55cm]pR2BBCenter) (pR2BBLabel) {$\pick{r_2}$};
	\node[] at ([xshift=0.2cm, yshift=0.45cm]dR2BBCenter) (dR2BBLabel) {$\drop{r_2}$};
	
	\draw[succ, shorten >=0.30cm] (depotBB) to (pR1BBCenter);
	\draw[succ, shorten >=0.30cm, shorten <=0.30cm] (pR1BBCenter) to (pR2BBCenter);
	\draw[succ, bend left=40, shorten >=0.30cm, shorten <=0.45cm] (pR2BBCenter) to (dR2BBCenter);
	\draw[succ, shorten >=0.30cm, shorten <=0.30cm] (dR2BBCenter) to (dR1BBCenter);
	\draw[succ, shorten <=0.30cm] (dR1BBCenter) to (depotBB);
	
	%\draw[decoration={brace, raise=2.9cm},decorate, thick, gray!80!white]
	%([xshift=-0.5cm]depotBB.west) -- node[above=3cm] {\textcolor{black}{2}} ([xshift=2.8cm]depotBB.east);
	
	% child C
	\node[depot, right = 3.4cm of depotBB] (depotBC) {};
	\node[myPackageA, above right = 1.35cm and 0.5cm of depotBC, anchor=center] (pR1BC) {};
	\node[myHouseA, above right = 0cm and 2cm  of depotBC, anchor=center] (dR1BC) {};
	\node[myPackageB, above left = 1.0cm and 0.3cm of depotBC, anchor=center] (pR2BC) {};
	\node[myHouseB, above right = 1.4cm and 1.8cm  of depotBC, anchor=center] (dR2BC) {};
	
	\coordinate (pR1BCCenter) at ([xshift=0.2cm, yshift=0.2cm]pR1BC.center);
	\coordinate (dR1BCCenter) at ([xshift=0.2cm, yshift=0.2cm]dR1BC.center);
	\coordinate (pR2BCCenter) at ([xshift=0.2cm, yshift=0.2cm]pR2BC.center);
	\coordinate (dR2BCCenter) at ([xshift=0.2cm, yshift=0.2cm]dR2BC.center);
	
	\node[] at ([xshift=0.1cm, yshift=-0.5cm]pR1BCCenter) (pR1BCLabel) {$\pick{r_1}$};
	\node[] at ([xshift=0.05cm, yshift=-0.45cm]dR1BCCenter) (dR1BCLabel) {$\drop{r_1}$};
	\node[] at ([xshift=-0.0cm, yshift=0.55cm]pR2BCCenter) (pR2BCLabel) {$\pick{r_2}$};
	\node[] at ([xshift=0.2cm, yshift=0.45cm]dR2BCCenter) (dR2BCLabel) {$\drop{r_2}$};
	
	\draw[succ, shorten >=0.30cm] (depotBC) to (pR2BCCenter);
	\draw[succ, bend left = 40, shorten >=0.30cm, shorten <=0.45cm] (pR2BCCenter) to (dR2BCCenter);
	\draw[succ, shorten >=0.30cm, shorten <=0.30cm] (dR2BCCenter) to (pR1BCCenter);
	\draw[succ, shorten >=0.30cm, shorten <=0.30cm] (pR1BCCenter) to (dR1BCCenter);
	\draw[succ, shorten <=0.30cm, shorten <=0.30cm] (dR1BCCenter) to (depotBC);
	
	%\draw[decoration={brace, raise=2.9cm},decorate, thick, gray!80!white]
	%([xshift=-0.5cm]depotBC.west) -- node[above=3cm] {\textcolor{black}{3}} ([xshift=2.8cm]depotBC.east);
	
	% child D
	\node[depot, below = 2.6cm of depotBA] (depotBD) {};
	\node[myPackageA, above right = 1.35cm and 0.5cm of depotBD, anchor=center] (pR1BD) {};
	\node[myHouseA, above right = 0cm and 2cm  of depotBD, anchor=center] (dR1BD) {};
	\node[myPackageB, above left = 1.0cm and 0.3cm of depotBD, anchor=center] (pR2BD) {};
	\node[myHouseB, above right = 1.4cm and 1.8cm  of depotBD, anchor=center] (dR2BD) {};
	
	\coordinate (pR1BDCenter) at ([xshift=0.2cm, yshift=0.2cm]pR1BD.center);
	\coordinate (dR1BDCenter) at ([xshift=0.2cm, yshift=0.2cm]dR1BD.center);
	\coordinate (pR2BDCenter) at ([xshift=0.2cm, yshift=0.2cm]pR2BD.center);
	\coordinate (dR2BDCenter) at ([xshift=0.2cm, yshift=0.2cm]dR2BD.center);
	
	\node[] at ([xshift=0.1cm, yshift=-0.5cm]pR1BDCenter) (pR1BDLabel) {$\pick{r_1}$};
	\node[] at ([xshift=0.05cm, yshift=-0.45cm]dR1BDCenter) (dR1BDLabel) {$\drop{r_1}$};
	\node[] at ([xshift=-0.0cm, yshift=0.55cm]pR2BDCenter) (pR2BDLabel) {$\pick{r_2}$};
	\node[] at ([xshift=0.2cm, yshift=0.45cm]dR2BDCenter) (dR2BDLabel) {$\drop{r_2}$};
	
	\draw[succ, shorten >=0.30cm] (depotBD) to (pR2BDCenter);
	\draw[succ, shorten >=0.30cm, shorten <=0.30cm] (pR2BDCenter) to (pR1BDCenter);
	\draw[succ, shorten >=0.30cm, shorten <=0.30cm] (pR1BDCenter) to (dR1BDCenter);
	\draw[succ, shorten >=0.30cm, shorten <=0.30cm] (dR1BDCenter) to (dR2BDCenter);
	\draw[succ, shorten <=0.30cm] (dR2BDCenter) to (depotBD);
	
	%\draw[decoration={brace, raise=0.3cm, mirror},decorate, thick, gray!80!white]
	%([xshift=-0.5cm]depotBD.west) -- node[below=0.4cm] {\textcolor{black}{4}} ([xshift=2.8cm]depotBD.east);
	
	% child E
	\node[depot, right = 3.4cm of depotBD] (depotBE) {};
	\node[myPackageA, above right = 1.35cm and 0.5cm of depotBE, anchor=center] (pR1BE) {};
	\node[myHouseA, above right = 0cm and 2cm  of depotBE, anchor=center] (dR1BE) {};
	\node[myPackageB, above left = 1.0cm and 0.3cm of depotBE, anchor=center] (pR2BE) {};
	\node[myHouseB, above right = 1.4cm and 1.8cm  of depotBE, anchor=center] (dR2BE) {};
	
	\coordinate (pR1BECenter) at ([xshift=0.2cm, yshift=0.2cm]pR1BE.center);
	\coordinate (dR1BECenter) at ([xshift=0.2cm, yshift=0.2cm]dR1BE.center);
	\coordinate (pR2BECenter) at ([xshift=0.2cm, yshift=0.2cm]pR2BE.center);
	\coordinate (dR2BECenter) at ([xshift=0.2cm, yshift=0.2cm]dR2BE.center);
	
	\node[] at ([xshift=0.1cm, yshift=-0.5cm]pR1BECenter) (pR1BELabel) {$\pick{r_1}$};
	\node[] at ([xshift=0.05cm, yshift=-0.45cm]dR1BECenter) (dR1BELabel) {$\drop{r_1}$};
	\node[] at ([xshift=-0.0cm, yshift=0.55cm]pR2BECenter) (pR2BELabel) {$\pick{r_2}$};
	\node[] at ([xshift=0.2cm, yshift=0.45cm]dR2BECenter) (dR2BELabel) {$\drop{r_2}$};

	\draw[succ, shorten >=0.30cm] (depotBE) to (pR1BECenter);
	\draw[succ, shorten >=0.30cm, shorten <=0.30cm] (pR1BECenter) to (pR2BECenter);
	\draw[succ, shorten >=0.30cm, shorten <=0.30cm] (pR2BECenter) to (dR1BECenter);
	\draw[succ, shorten >=0.30cm, shorten <=0.30cm] (dR1BECenter) to (dR2BECenter);
	\draw[succ, shorten <=0.30cm] (dR2BECenter) to (depotBE);
	
	%\draw[decoration={brace, raise=0.3cm, mirror},decorate, thick, gray!80!white]
	%([xshift=-0.5cm]depotBE.west) -- node[below=0.4cm] {\textcolor{black}{5}} ([xshift=2.8cm]depotBE.east);
	
	% child F
	\node[depot, right = 3.4cm of depotBE] (depotBF) {};
	\node[myPackageA, above right = 1.35cm and 0.5cm of depotBF, anchor=center] (pR1BF) {};
	\node[myHouseA, above right = 0cm and 2cm  of depotBF, anchor=center] (dR1BF) {};
	\node[myPackageB, above left = 1.0cm and 0.3cm of depotBF, anchor=center] (pR2BF) {};
	\node[myHouseB, above right = 1.4cm and 1.8cm  of depotBF, anchor=center] (dR2BF) {};
	
	\coordinate (pR1BFCenter) at ([xshift=0.2cm, yshift=0.2cm]pR1BF.center);
	\coordinate (dR1BFCenter) at ([xshift=0.2cm, yshift=0.2cm]dR1BF.center);
	\coordinate (pR2BFCenter) at ([xshift=0.2cm, yshift=0.2cm]pR2BF.center);
	\coordinate (dR2BFCenter) at ([xshift=0.2cm, yshift=0.2cm]dR2BF.center);
	
	\node[] at ([xshift=0.1cm, yshift=-0.5cm]pR1BFCenter) (pR1BFLabel) {$\pick{r_1}$};
	\node[] at ([xshift=0.05cm, yshift=-0.45cm]dR1BFCenter) (dR1BFLabel) {$\drop{r_1}$};
	\node[] at ([xshift=-0.0cm, yshift=0.55cm]pR2BFCenter) (pR2BFLabel) {$\pick{r_2}$};
	\node[] at ([xshift=0.2cm, yshift=0.45cm]dR2BFCenter) (dR2BFLabel) {$\drop{r_2}$};
	
	\draw[succ, shorten >=0.30cm] (depotBF) to (pR1BFCenter);
	\draw[succ, shorten >=0.30cm, shorten <=0.30cm] (pR1BFCenter) to (dR1BFCenter);
	\draw[succ, shorten >=0.30cm, shorten <=0.30cm] (dR1BFCenter) to (pR2BFCenter);
	\draw[succ, bend left=40, shorten >=0.30cm, shorten <=0.45cm] (pR2BFCenter) to (dR2BFCenter);
	\draw[succ, shorten <=0.30cm] (dR2BFCenter) to (depotBF);
	
	%\draw[decoration={brace, raise=0.3cm, mirror},decorate, thick, gray!80!white]
	%([xshift=-0.5cm]depotBF.west) -- node[below=0.4cm] {\textcolor{black}{6}} ([xshift=2.8cm]depotBF.east);
	
	% root node
	\coordinate (midHeightLeft) at ($(depotBA)!0.5!(depotBD)$);
	\node[depot, left=5.0cm of midHeightLeft] (depotAA) {};
	\node[myPackageA, above right = 1.35cm and 0.5cm of depotAA, anchor=center] (pR1AA) {};
	\node[myHouseA, above right = 0cm and 2cm  of depotAA, anchor=center] (dR1AA) {};
	\node[myPackageB, above left = 1.0cm and 0.3cm of depotAA, anchor=center] (pR2AA) {};
	\node[myHouseB, above right = 1.4cm and 1.8cm  of depotAA, anchor=center] (dR2AA) {};
	
	\coordinate (pR1AACenter) at ([xshift=0.2cm, yshift=0.2cm]pR1AA.center);
	\coordinate (dR1AACenter) at ([xshift=0.2cm, yshift=0.2cm]dR1AA.center);
	\coordinate (pR2AACenter) at ([xshift=0.2cm, yshift=0.2cm]pR2AA.center);
	\coordinate (dR2AACenter) at ([xshift=0.2cm, yshift=0.2cm]dR2AA.center);
	
	\node[] at ([xshift=0.1cm, yshift=-0.5cm]pR1AACenter) (pR1AALabel) {$\pick{r_1}$};
	\node[] at ([xshift=0.05cm, yshift=-0.45cm]dR1AACenter) (dR1AALabel) {$\drop{r_1}$};
	\node[] at ([xshift=-0.0cm, yshift=0.55cm]pR2AACenter) (pR2AALabel) {$\pick{r_2}$};
	\node[] at ([xshift=0.2cm, yshift=0.45cm]dR2AACenter) (dR2AALabel) {$\drop{r_2}$};
	
	\draw[succ, shorten >=0.30cm] (depotAA) to (pR1AACenter);
	\draw[succ, shorten <=0.30cm, shorten >=0.30cm] (pR1AACenter) to (dR1AACenter);
	\draw[succ, shorten <=0.30cm] (dR1AACenter) to (depotAA);
	
	%\draw[decoration={brace},decorate, thick, gray!80!white]
	%([xshift=2.8cm, yshift=2.7cm]depotAA.east) -- node[right=0.1cm] {\textcolor{black}{Path to extend}} ([xshift=2.8cm, yshift=-0.4cm]depotAA.east);
	
	% bounding boxes
	
	% bounding box arround root node
	\coordinate (partAALeftSide) at ([xshift=-0.4cm]depotAA.west);
	\coordinate (partAATopRight) at ([xshift=2.7cm, yshift=2.45cm]depotAA);
	\coordinate (partAABottomSide) at ([yshift=-0.2cm]depotAA.south);
	\coordinate (partAABottomLeft) at (partAALeftSide |- partAABottomSide);
	\draw[boundingBox] (partAABottomLeft) rectangle (partAATopRight); 
	\coordinate (partAAMidPoint) at ($(partAABottomLeft)!0.5!(partAATopRight)$);
	\coordinate (partAATopMidPoint) at (partAAMidPoint |- partAATopRight);
	\node[above = 0.1cm of partAATopMidPoint] (partAALabel) {Path to extend};
	
	% bounding box arround child A
	\coordinate (partBALeftSide) at ([xshift=-0.4cm]depotBA.west);
	\coordinate (partBATopRight) at ([xshift=2.7cm, yshift=2.45cm]depotBA);
	\coordinate (partBABottomSide) at ([yshift=-0.2cm]depotBA.south);
	\coordinate (partBABottomLeft) at (partBALeftSide |- partBABottomSide);
	\draw[boundingBox] (partBABottomLeft) rectangle (partBATopRight); 
	\coordinate (partBAMidPoint) at ($(partBABottomLeft)!0.5!(partBATopRight)$);
	\coordinate (partBATopMidPoint) at (partBAMidPoint |- partBATopRight);
	\node[above = 0.0cm of partBATopMidPoint] (partBALabel) {1};
	
	% bounding box arround child B
	\coordinate (partBBLeftSide) at ([xshift=-0.4cm]depotBB.west);
	\coordinate (partBBTopRight) at ([xshift=2.7cm, yshift=2.45cm]depotBB);
	\coordinate (partBBBottomSide) at ([yshift=-0.2cm]depotBB.south);
	\coordinate (partBBBottomLeft) at (partBBLeftSide |- partBBBottomSide);
	\draw[boundingBox] (partBBBottomLeft) rectangle (partBBTopRight); 
	\coordinate (partBBMidPoint) at ($(partBBBottomLeft)!0.5!(partBBTopRight)$);
	\coordinate (partBBTopMidPoint) at (partBBMidPoint |- partBBTopRight);
	\node[above = 0.0cm of partBBTopMidPoint] (partBBLabel) {2};
	
	% bounding box arround child C
	\coordinate (partBCLeftSide) at ([xshift=-0.4cm]depotBC.west);
	\coordinate (partBCTopRight) at ([xshift=2.7cm, yshift=2.45cm]depotBC);
	\coordinate (partBCBottomSide) at ([yshift=-0.2cm]depotBC.south);
	\coordinate (partBCBottomLeft) at (partBCLeftSide |- partBCBottomSide);
	\draw[boundingBox] (partBCBottomLeft) rectangle (partBCTopRight); 
	\coordinate (partBCMidPoint) at ($(partBCBottomLeft)!0.5!(partBCTopRight)$);
	\coordinate (partBCTopMidPoint) at (partBCMidPoint |- partBCTopRight);
	\node[above = 0.0cm of partBCTopMidPoint] (partBCLabel) {3};
	
	% bounding box arround child D
	\coordinate (partBDLeftSide) at ([xshift=-0.4cm]depotBD.west);
	\coordinate (partBDTopRight) at ([xshift=2.7cm, yshift=2.45cm]depotBD);
	\coordinate (partBDBottomSide) at ([yshift=-0.2cm]depotBD.south);
	\coordinate (partBDBottomLeft) at (partBDLeftSide |- partBDBottomSide);
	\draw[boundingBox] (partBDBottomLeft) rectangle (partBDTopRight); 
	\coordinate (partBDMidPoint) at ($(partBDBottomLeft)!0.5!(partBDTopRight)$);
	\coordinate (partBDBottomMidPoint) at (partBDMidPoint |- partBDBottomSide);
	\node[below = 0.0cm of partBDBottomMidPoint] (partBDLabel) {4};
	
	% bounding box arround child E
	\coordinate (partBELeftSide) at ([xshift=-0.4cm]depotBE.west);
	\coordinate (partBETopRight) at ([xshift=2.7cm, yshift=2.45cm]depotBE);
	\coordinate (partBEBottomSide) at ([yshift=-0.2cm]depotBE.south);
	\coordinate (partBEBottomLeft) at (partBELeftSide |- partBEBottomSide);
	\draw[boundingBox] (partBEBottomLeft) rectangle (partBETopRight); 
	\coordinate (partBEMidPoint) at ($(partBEBottomLeft)!0.5!(partBETopRight)$);
	\coordinate (partBEBottomMidPoint) at (partBEMidPoint |- partBEBottomSide);
	\node[below = 0.0cm of partBEBottomMidPoint] (partBELabel) {5};
	
	% bounding box arround child F
	\coordinate (partBFLeftSide) at ([xshift=-0.4cm]depotBF.west);
	\coordinate (partBFTopRight) at ([xshift=2.7cm, yshift=2.45cm]depotBF);
	\coordinate (partBFBottomSide) at ([yshift=-0.2cm]depotBF.south);
	\coordinate (partBFBottomLeft) at (partBFLeftSide |- partBFBottomSide);
	\draw[boundingBox] (partBFBottomLeft) rectangle (partBFTopRight); 
	\coordinate (partBFMidPoint) at ($(partBFBottomLeft)!0.5!(partBFTopRight)$);
	\coordinate (partBFBottomMidPoint) at (partBFMidPoint |- partBFBottomSide);
	\node[below = 0.0cm of partBFBottomMidPoint] (partBFLaBFl) {6};
	
	% transition label
	\coordinate[] (beginTransitionRoot) at ([xshift=-2.35cm, yshift=1.05cm]midHeightLeft);
	\coordinate[] (endTransitionRoot) at ([xshift=-0.95cm, yshift=1.05cm]midHeightLeft);
	\draw[transition, ultra thick] (beginTransitionRoot) to node[midway, above=0.1cm, align=center] {\textcolor{black}{Extend}\\\textcolor{black}{path}} (endTransitionRoot);
	
	\coordinate (bracketTop) at ([yshift=0.3cm]partBALeftSide |- partBATopRight);
	\coordinate (bracketBottom) at ([yshift=-0.3cm]partBDBottomLeft);
	\draw[decoration={brace, raise=0.15cm, mirror},decorate, ultra thick, gray!80!white]
	(bracketTop) -- (bracketBottom);

	% legend
	
	\node[myPackageA] at ([xshift=-0cm,yshift=-2.5cm]pR2AA) (markerPickup) {};
	\node[minimum size=0.4cm] at ([xshift=1.05cm,yshift=0.2cm]markerPickup) (LegendPickup) {Pickup};
	
	\node[myHouseA] at ([xshift=2.0cm,yshift=0cm]markerPickup) (markerDrop) {};
	\node[minimum size=0.4cm] at ([xshift=0.90cm,yshift=0.2cm]markerDrop) (LegendDrop) {Drop};
	
	\node[depot] at ([xshift=0.17cm,yshift=-0.4cm]markerPickup) (markerDepot) {};
	\node[minimum size=0.4cm] at ([xshift=0.82cm,yshift=0.0cm]markerDepot.center) (LegendDepot) {Depot};
	
	\coordinate[anchor=south] (markerV1Begin) at ([xshift=1.83cm,yshift=0cm]markerDepot);
	\coordinate (markerV1End) at ([xshift=0.4cm,yshift=0cm]markerV1Begin.center);
	\draw[succ] (markerV1Begin) -- (markerV1End);
	\node[minimum size=0.4cm, anchor=west] at (markerV1End.east) (LegendV1) {Vehicle path};

	\draw[draw, thick, rounded corners, gray!80!white] ([yshift=0.6cm,xshift=-0.1cm]markerPickup) rectangle ([yshift=0.0cm,xshift=0.0cm]LegendV1.south east);
	
\end{tikzpicture}

%% file: figures/pc-problem-example.tex
\definecolor{colorRessource1}{HTML}{D55E00}
\definecolor{colorRessource2}{HTML}{56B4E9}
\definecolor{colorRessource3}{HTML}{E69F00}

		\begin{tikzpicture}
	
	% Define styles for tasks and resources
	\tikzset{
		task/.style 2 args={
			thick, fill=#1, minimum width=#2, minimum height=0.5cm, rounded corners=1mm, anchor=west, draw
		},
		resource/.style 2 args={
			thick, fill=#1, minimum width=#2, minimum height=0.5cm, rounded corners=1mm, anchor=west, draw, opacity=0.6
		},
		prepost/.style={draw, fill=#1, opacity=0.3, minimum height=0.5cm}
	}
	
	% Tasks
	\node[task={colorRessource1}{0.5cm}] (t1) at (0.25,3.5) {1};
	\node[task={colorRessource1}{1cm}] (t2) at (1.25,3.5) {2};
	\node[task={colorRessource2}{0.75cm}] (t4) at (2.25,3.5) {4};
	\node[task={colorRessource3}{2cm}] (t6) at (3,3.5) {6};
	\node[task={colorRessource2}{0.5cm}] (t8) at (5.25,3.5) {8};
	
	% Resources and pre/post processing
	\node[resource={colorRessource1}{1cm}] (r1) at (0,2.5) {};
	\node[task={colorRessource1}{0.5cm}] (r1-label) at (0.25,2.5) {1};
	
	\node[resource={colorRessource1}{1.75cm}] (r2) at (1,2.5) {};
	\node[task={colorRessource1}{1cm}] (r2-label) at (1.25,2.5) {2};
	
	\node[resource={colorRessource2}{3cm}] (r4) at (1.25,1.75) {};
	\node[task={colorRessource2}{0.75cm}] (r4-label) at (2.25,1.75) {4};
	
	\node[resource={colorRessource2}{2cm}] (r8) at (4.25,1.75) {};
	\node[task={colorRessource2}{0.5cm}] (r8-label) at (5.25,1.75) {8};
	
	\node[resource={colorRessource3}{3.75cm}] (r6) at (2.25,1) {};
	\node[task={colorRessource3}{2cm}] (r6-label) at (3,1) {6};
	
	% Time arrow
	\draw[-Latex, thick, line width=0.8pt] (-0.5,4) -- (7.0,4) node[anchor=south east] {Time};
	
	% labels for the machines
	\node[] (global-label) at (-0.5,3.5) {$G$};
	\node[] (r1-label) at (-0.5,2.5) {$R_1$};
	\node[] (r2-label) at (-0.5,1.75) {$R_2$};
	\node[] (r3-label) at (-0.5,1) {$R_3$};
	
	% Draw the curly braces
	\draw[decorate, decoration={brace, amplitude=5pt, mirror}] 
	($(r6.south west) - (0, 0.2)$) -- ($(r6.south west) - (-0.75, 0.2)$) node [midway, yshift=-12pt] {$d^{\textit{pre}}_6$}; % Label for the curly brace
	
	\draw[decorate, decoration={brace, amplitude=5pt, mirror}] 
	($(r6.south west) - (-0.75, 0.2)$) -- ($(r6.south west) - (-2.75, 0.2)$) node [midway, yshift=-14pt] {$d_6$}; % Label for the curly brace
	
	\draw[decorate, decoration={brace, amplitude=5pt, mirror}] 
	($(r6.south west) - (-2.75, 0.2)$) -- ($(r6.south west) - (-3.75, 0.2)$) node [midway, yshift=-12pt] {$d^{\textit{post}}_6$}; % Label for the curly brace
	
\end{tikzpicture}

%% file: figures/cumulative-sequence.tex
	\definecolor{colorActivity0}{HTML}{D55E00}
	\definecolor{colorActivity1}{HTML}{56B4E9}
	\definecolor{colorActivity2}{HTML}{009E73}
	\definecolor{colorActivity3}{HTML}{E69F00}
    
    \begin{tikzpicture}[node distance=0.8cm and 1.5cm,
		% GLOBAL CFG
		>=LaTeX,
		% Styles
		member/.style={
			circle, 
			inner sep=1.5pt,
			minimum size=4mm,
            scale=1.0,
		},
		possible/.style={
			circle, 
			inner sep=1.5pt,
			minimum size=4mm,
            scale=1.0,
		},
		excluded/.style={
			circle, 
			inner sep=1.5pt,
			minimum size=6mm,
		},
		insertionDirected/.style={
			densely dashed,
			gray!80!white,
			very thick,
			-{Stealth[gray!80!white]}
		},
		insertion/.style={
			densely dashed,
			gray!80!white,
			very thick,
			{Stealth[gray!80!white]}-{Stealth[gray!80!white]}
		},
		successor/.style={
			very thick,
			-{Stealth[black]}
		},
		xtick/.style={
			rectangle,
			minimum height=0.6cm,
		},
		]
		\newcommand{\distanceBelowLeftGraph}{4.0cm and 0.4cm}
		\newcommand{\distanceBelowRightGraph}{4.0cm and -0.2cm}
		
		% ========= sequence part =========
		
		\node [member](start) {$\first$};
		\node [member, right= of start, fill=colorActivity0](s0) {$s_0$};
		\node [member, right= of s0, fill=colorActivity1](s1) {$s_1$};
		\node [member, right= of s1, fill=colorActivity2](e2) {$e_2$};
		\node [member, right= of e2, fill=colorActivity3](s3) {$s_3$};
		\node [member, right= of s3, fill=colorActivity3](e3) {$e_3$};
		\node [member, right= of e3](end) {$\last$};
		
		\coordinate (midway1) at ($(s0)!0.5!(s1)$);
		\coordinate (midway2) at ($(s1)!0.5!(e3)$);
		
		\node [possible, below = of midway1, fill=colorActivity2] (s2) {$s_2$};
		\node [possible, below = of midway2, fill=colorActivity1] (e1) {$e_1$};
		\node [possible, above = of midway2, fill=colorActivity0] (e0) {$e_0$};
		
		\draw [successor] (start) edge (s0);
		\draw [successor] (s0) edge (s1);
		\draw [successor] (s1) edge (e2);
		\draw [successor] (e2) edge (s3);
		\draw [successor] (s3) edge (e3);
		\draw [successor] (e3) edge (end);
		
		\draw [insertionDirected, bend left=15] (s0) edge (e0);
		\draw [insertion, bend left=10] (s1)  edge (e0);
		\draw [insertion, bend left=10] (e2)  edge (e0);
		\draw [insertion, bend right=10] (s3)  edge (e0);
		\draw [insertion, bend right=10] (e3)  edge (e0);
		\draw [insertionDirected, bend left=15] (e0)  edge (end);
		
		\draw [insertionDirected, bend right=10] (e2)  edge (e1);
		\draw [insertion, bend left=10] (s3)  edge (e1);
		\draw [insertionDirected, bend right=10] (e1)  edge (e3);
		
		\draw [insertionDirected, bend right=10] (start) edge (s2);
		\draw [insertion, bend right=10] (s0)  edge (s2);
		\draw [insertion, bend left=10] (s1)  edge (s2);
		\draw [insertionDirected, bend right=10] (s2)  edge (e2);
		
		% ========= graph part =========
		
		% label for load at a given node
		%\node[below left = \distanceBelowLeftGraph of start, xtick, anchor=north] (StartBefore) {$\loadBefore{\first}$};
		\node[below left = \distanceBelowLeftGraph of s0, xtick, anchor=north] (S0Before) {$\loadBefore{s_0}$};
		\node[below left = \distanceBelowLeftGraph of s1, xtick, anchor=north] (S1Before) {$\loadBefore{s_1}$};
		\node[below left = \distanceBelowLeftGraph of e2, xtick, anchor=north] (E2Before) {$\loadBefore{e_2}$};
		\node[below left = \distanceBelowLeftGraph of s3, xtick, anchor=north] (S3Before) {$\loadBefore{s_3}$};
		\node[below left = \distanceBelowLeftGraph of e3, xtick, anchor=north] (E3Before) {$\loadBefore{e_3}$};
		\node[below left = \distanceBelowLeftGraph of end, xtick, anchor=north] (EndBefore) {$\loadBefore{\last}$};
		
		\node[below right = \distanceBelowRightGraph of start, xtick, anchor=north] (Start) {$\loadAt{\first}$};
		\node[below right = \distanceBelowRightGraph of s0, xtick, anchor=north] (S0) {$\loadAt{s_0}$};
		\node[below right = \distanceBelowRightGraph of s1, xtick, anchor=north] (S1) {$\loadAt{s_1}$};
		\node[below right = \distanceBelowRightGraph of e2, xtick, anchor=north] (E2) {$\loadAt{e_2}$};
		\node[below right = \distanceBelowRightGraph of s3, xtick, anchor=north] (S3) {$\loadAt{s_3}$};
		\node[below right = \distanceBelowRightGraph of e3, xtick, anchor=north] (E3) {$\loadAt{e_3}$};
		\node[below right = \distanceBelowRightGraph of end, xtick, anchor=north] (End) {$\loadAt{\last}$};
		
		% coordinates of label for load after a given node
		\coordinate (Start-after) at ($(Start)!0.5!(S0Before)$);
		\coordinate (S0-after) at ($(S0)!0.5!(S1Before)$);
		\coordinate (S1-after) at ($(S1)!0.5!(E2Before)$);
		\coordinate (E2-after) at ($(E2)!0.5!(S3Before)$);
		\coordinate (S3-after) at ($(S3)!0.5!(E3Before)$);
		\coordinate (E3-after) at ($(E3)!0.5!(EndBefore)$);
		
		% label for load after a given node
		\node[xtick] at (Start-after) {$\loadAfter{\first}$};
		\node[xtick] at (S0-after) {$\loadAfter{s_0}$};
		\node[xtick] at (S1-after) {$\loadAfter{s_1}$};
		\node[xtick] at (E2-after) {$\loadAfter{e_2}$};
		\node[xtick] at (S3-after) {$\loadAfter{s_3}$};
		\node[xtick] at (E3-after) {$\loadAfter{e_3}$};
		
		% load from activity 0
		\coordinate[above left = 1.0cm and -0.1cm of S0] (A0-begin);
		\coordinate[above right = 0cm and -0.1cm of S0] (A0-end);
		\fill [fill=colorActivity0, fill opacity=1] (A0-begin) rectangle (A0-end);
		% Calculate the midpoint for the label
		\coordinate (A0-center) at ($ (A0-begin)!0.5!(A0-end) $);
		% Place the label at the midpoint
		\node at (A0-center) {0};
		%
		% load from activity 1
		\coordinate[above left = 0.5cm and -0.1cm of S1] (A1-begin);
		\coordinate[above right = 0cm and -0.1cm of E2] (A1-end);
		\fill [fill=colorActivity1, fill opacity=1] (A1-begin) rectangle (A1-end);
		% Calculate the midpoint for the label
		\coordinate (A1-center) at ($ (A1-begin)!0.5!(A1-end) $);
		% Place the label at the midpoint
		\node at (A1-center) {{1}};
		%
		% load from activity 2
		\coordinate[above left = 1.0cm and -0.1cm of E2Before] (A2-begin);
		\coordinate[above right = 0.5cm and -0.1cm of E2Before] (A2-end);
		\fill [fill=colorActivity2, fill opacity=1] (A2-begin) rectangle (A2-end);
		% Calculate the midpoint for the label
		\coordinate (A2-center) at ($ (A2-begin)!0.5!(A2-end) $);
		% Place the label at the midpoint
		\node at (A2-center) {{2}};
		%
		% load from activity 3 
		\coordinate[above left = 1.5cm and -0.1cm of S3] (A3-begin);
		\coordinate[above right = 0cm and -0.1cm of E3Before] (A3-end);
		\fill [fill=colorActivity3, fill opacity=1] (A3-begin) rectangle (A3-end);
		% Calculate the midpoint for the label
		\coordinate (A3-center) at ($ (A3-begin)!0.5!(A3-end) $);
		% Place the label at the midpoint
		\node at (A3-center) {3};
		%
		% x axis
		\coordinate[above left = -0.02cm and 0cm of Start] (origin);
		\coordinate[above right = -0.02cm and 0cm of End] (endingSequence);
		\draw[->, thick] (origin) -- (endingSequence);
		\node[right=0cm of endingSequence] (label-sequence) {$\seq{s}$};
		%
		% y axis
		\coordinate[above = 2.6cm of origin] (endingCapacity);
		\draw[->, thick] (origin) -- (endingCapacity);
		\node[left =0cm of endingCapacity] (label-load) {Load};
		\coordinate[above= 2cm of origin] (coord4);
		\coordinate[above= 1.5cm of origin] (coord3);
		\coordinate[above= 1.0cm of origin] (coord2);
		\coordinate[above= 0.5cm of origin] (coord1);
		\node[left=0.1cm of coord4] (label4) {{4}};
		\node[left=0.1cm of coord3] (label3) {{3}};
		\node[left=0.1cm of coord2] (label2) {{2}};
		\node[left=0.1cm of coord1] (label1) {{1}};
		\draw[thick] (coord4) -- (label4);
		\draw[thick] (coord3) -- (label3);
		\draw[thick] (coord2) -- (label2);
		\draw[thick] (coord1) -- (label1);
		
	\end{tikzpicture}

%% file: figures/cumulative-3-values.tex
\definecolor{colorActivity0}{HTML}{D55E00}
\definecolor{colorActivity1}{HTML}{56B4E9}
\definecolor{colorActivity2}{HTML}{009E73}
\definecolor{colorActivity3}{HTML}{E69F00}

\begin{tikzpicture}[node distance=1.4cm and 1.2cm,
	every node/.style={scale=1.0},
	% GLOBAL CFG
	>=LaTeX,
	% Styles
	member/.style={
		circle, 
		inner sep=1.5pt,
			minimum size=4mm,
            scale=1.0,
	},
	possible/.style={
		circle, 
		inner sep=1.5pt,
			minimum size=4mm,
            scale=1.0,
	},
	excluded/.style={
		circle, 
		inner sep=1.5pt,
		minimum size=6mm,
	},
	insertionDirected/.style={
		densely dashed,
		gray!80!white,
		very thick,
		-{Stealth[gray!80!white]}
	},
	insertion/.style={
		densely dashed,
		gray!80!white,
		very thick,
		{Stealth[gray!80!white]}-{Stealth[gray!80!white]}
	},
	successor/.style={
		very thick,
		-{Stealth[black]}
	},
	xtick/.style={
		rectangle,
		minimum height=0.6cm,
        scale=1.0,
	},
	]
	\newcommand{\distanceBelowLeftGraph}{1.6cm and 0.4cm}
	\newcommand{\distanceBelowRightGraph}{1.6cm and -0.2cm}
	
	% part A
	% ========= sequence part =========
	
	\node [member](partAstart) {$\first$};
	\node [member, right= of partAstart, fill=colorActivity0](partAs0) {$s_0$};
	\node [member, right= of partAs0, fill=colorActivity1](partAe1) {$e_1$};
	\node [member, right= of partAe1](partAend) {$\last$};
	
	\draw [successor] (partAstart) edge (partAs0);
	\draw [successor] (partAs0) edge (partAe1);
	\draw [successor] (partAe1) edge (partAend);
	
	% ========= graph part =========
	
	% label for load at a given node
	%\node[below left = \distanceBelowLeftGraph of start, xtick, anchor=north] (StartBefore) {$\loadBefore{\first}$};
	\node[below left = \distanceBelowLeftGraph of partAs0, xtick, anchor=north] (partAS0Before) {$\loadBefore{s_0}$};
	\node[below left = \distanceBelowLeftGraph of partAe1, xtick, anchor=north] (partAE1Before) {$\loadBefore{e_1}$};
	\node[below left = \distanceBelowLeftGraph of partAend, xtick, anchor=north] (partAEndBefore) {$\loadBefore{\last}$};
	
	\node[below right = \distanceBelowRightGraph of partAstart, xtick, anchor=north] (partAStart) {$\loadAt{\first}$};
	\node[below right = \distanceBelowRightGraph of partAs0, xtick, anchor=north] (partAS0) {$\loadAt{s_0}$};
	\node[below right = \distanceBelowRightGraph of partAe1, xtick, anchor=north] (partAE1) {$\loadAt{e_1}$};
	\node[below right = \distanceBelowRightGraph of partAend, xtick, anchor=north] (partAEnd) {$\loadAt{\last}$};
	
	% coordinates of label for load after a given node
	\coordinate (partAStart-after) at ($(partAStart)!0.5!(partAS0Before)$);
	\coordinate (partAS0-after) at ($(partAS0)!0.5!(partAE1Before)$);
	\coordinate (partAE1-after) at ($(partAE1)!0.5!(partAEndBefore)$);
	
	% label for load after a given node
	\node[xtick] at (partAStart-after) {$\loadAfter{\first}$};
	\node[xtick] at (partAS0-after) {$\loadAfter{s_0}$};
	\node[xtick] at (partAE1-after) {$\loadAfter{e_1}$};
	
	% load from activity 0
	\coordinate[above left = 0.5cm and -0.1cm of partAS0] (partAA0-begin);
	\coordinate[above right = 0cm and -0.1cm of partAS0] (partAA0-end);
	\fill [fill=colorActivity0, fill opacity=1] (partAA0-begin) rectangle (partAA0-end);
	% Calculate the midpoint for the label
	\coordinate (partAA0-center) at ($ (partAA0-begin)!0.5!(partAA0-end) $);
	% Place the label at the midpoint
	\node at (partAA0-center) {{0}};
	%
	% load from activity 1
	\coordinate[above left = 0.5cm and -0.1cm of partAE1Before] (partAA1-begin);
	\coordinate[above right = 0cm and -0.1cm of partAE1Before] (partAA1-end);
	\fill [fill=colorActivity1, fill opacity=1] (partAA1-begin) rectangle (partAA1-end);
	% Calculate the midpoint for the label
	\coordinate (partAA1-center) at ($ (partAA1-begin)!0.5!(partAA1-end) $);
	% Place the label at the midpoint
	\node at (partAA1-center) {{1}};
	%
	%
	% x axis
	\coordinate[above left = -0.02cm and 0cm of partAStart] (partAorigin);
	\coordinate[above right = -0.02cm and 0cm of partAEnd] (partAendingSequence);
	\draw[->, thick] (partAorigin) -- (partAendingSequence);
	\node[right=0cm of partAendingSequence] (partAlabel-sequence) {$\seq{S}$};
	%
	% y axis
	\coordinate[above = 1.5cm of partAorigin] (partAendingCapacity);
	\draw[->, thick] (partAorigin) -- (partAendingCapacity);
	\node[left =0cm of partAendingCapacity] (partAlabel-load) {Load};
	\coordinate[above= 1.0cm of partAorigin] (partAcoord2);
	\coordinate[above= 0.5cm of partAorigin] (partAcoord1);
	\node[left=0.1cm of partAcoord2] (partAlabel2) {{2}};
	\node[left=0.1cm of partAcoord1] (partAlabel1) {{1}};
	\draw (partAcoord2) -- (partAlabel2);
	\draw (partAcoord1) -- (partAlabel1);

	% part B
	% ========= sequence part =========
	
	\node [member, right=2.4cm of partAend](partBstart) {$\first$};
	\node [member, right= of partBstart, fill=colorActivity0](partBs0) {$s_0$};
	\node [member, right= of partBs0, fill=colorActivity0](partBe0) {$e_0$};
	\node [member, right= of partBe0](partBend) {$\last$};
	
	\draw [successor] (partBstart) edge (partBs0);
	\draw [successor] (partBs0) edge (partBe0);
	\draw [successor] (partBe0) edge (partBend);
	
	% ========= graph part =========
	
	% label for load at a given node
	%\node[below left = \distanceBelowLeftGraph of start, xtick, anchor=north] (StartBefore) {$\loadBefore{\first}$};
	\node[below left = \distanceBelowLeftGraph of partBs0, xtick, anchor=north] (partBS0Before) {$\loadBefore{s_0}$};
	\node[below left = \distanceBelowLeftGraph of partBe0, xtick, anchor=north] (partBE0Before) {$\loadBefore{e_0}$};
	\node[below left = \distanceBelowLeftGraph of partBend, xtick, anchor=north] (partBEndBefore) {$\loadBefore{\last}$};
	
	\node[below right = \distanceBelowRightGraph of partBstart, xtick, anchor=north] (partBStart) {$\loadAt{\first}$};
	\node[below right = \distanceBelowRightGraph of partBs0, xtick, anchor=north] (partBS0) {$\loadAt{s_0}$};
	\node[below right = \distanceBelowRightGraph of partBe0, xtick, anchor=north] (partBE0) {$\loadAt{e_0}$};
	\node[below right = \distanceBelowRightGraph of partBend, xtick, anchor=north] (partBEnd) {$\loadAt{\last}$};
	
	% coordinates of label for load after a given node
	\coordinate (partBStart-after) at ($(partBStart)!0.5!(partBS0Before)$);
	\coordinate (partBS0-after) at ($(partBS0)!0.5!(partBE0Before)$);
	\coordinate (partBE0-after) at ($(partBE0)!0.5!(partBEndBefore)$);
	
	% label for load after a given node
	\node[xtick] at (partBStart-after) {$\loadAfter{\first}$};
	\node[xtick] at (partBS0-after) {$\loadAfter{s_0}$};
	\node[xtick] at (partBE0-after) {$\loadAfter{e_1}$};
	
	% load from activity 0
	\coordinate[above left = 0.5cm and -0.1cm of partBS0] (partBA0-begin);
	\coordinate[above right = 0cm and -0.1cm of partBE0Before] (partBA0-end);
	\fill [fill=colorActivity0, fill opacity=1] (partBA0-begin) rectangle (partBA0-end);
	% Calculate the midpoint for the label
	\coordinate (partBA0-center) at ($ (partBA0-begin)!0.5!(partBA0-end) $);
	% Place the label at the midpoint
	\node at (partBA0-center) {{0}};
	%
	%
	%
	% x axis
	\coordinate[above left = -0.02cm and 0cm of partBStart] (partBorigin);
	\coordinate[above right = -0.02cm and 0cm of partBEnd] (partBendingSequence);
	\draw[->, thick] (partBorigin) -- (partBendingSequence);
	\node[right=0cm of partBendingSequence] (partBlabel-sequence) {$\seq{S}$};
	%
	% y axis
	\coordinate[above = 1.5cm of partBorigin] (partBendingCapacity);
	\draw[->, thick] (partBorigin) -- (partBendingCapacity);
	\node[left =0cm of partBendingCapacity] (partBlabel-load) {Load};
	\coordinate[above= 1.0cm of partBorigin] (partBcoord2);
	\coordinate[above= 0.5cm of partBorigin] (partBcoord1);
	\node[left=0.1cm of partBcoord2] (partBlabel2) {{2}};
	\node[left=0.1cm of partBcoord1] (partBlabel1) {{1}};
	\draw (partBcoord2) -- (partBlabel2);
	\draw (partBcoord1) -- (partBlabel1);
	
\end{tikzpicture}

%% file: tables/darp-landscape.tex
% Requires \usepackage{pdflscape} and \usepackage{siunitx} in the document preamble.
\begin{landscape}
	\begin{table}[t]
		\centering
		\caption{Primal gaps (\%) to the best known solution (bks) on DARP instances after 15 minutes of runtime. Entries report the mean primal gap $\pm$ the sample standard deviation and the best primal gap over ten runs. A dash indicates that no run found a feasible solution. Except for entries marked with $^\dagger$, numerical entries are based on ten feasible runs. For marked entries, statistics use the feasible runs specified in the table note. Best mean and best gap values are in bold.}
		\label{tab:darp-landscape}
		\begingroup
		\footnotesize
		\setlength{\tabcolsep}{2.5pt}
		\renewcommand{\arraystretch}{0.95}
		\sisetup{detect-weight = true, detect-family = true, table-number-alignment = right}
		\begin{tabular}{lrrrS[table-format=2.2]@{\,}c@{\,}lrS[table-format=2.2]@{\,}c@{\,}lrS[table-format=2.2]@{\,}c@{\,}lrS[table-format=2.2]@{\,}c@{\,}lrS[table-format=2.2]@{\,}c@{\,}lrS[table-format=2.2]@{\,}c@{\,}lrS[table-format=2.2]@{\,}c@{\,}lr}
			\toprule
			& & & & \multicolumn{4}{c}{Succ} & \multicolumn{4}{c}{Succ-LNS} & \multicolumn{4}{c}{CPO} & \multicolumn{4}{c}{RF (2021)} & \multicolumn{4}{c}{\shortstack{OR-Tools Routing}} & \multicolumn{4}{c}{Hexaly} & \multicolumn{4}{c}{Seqvar}\\
			\cmidrule(lr){5-8}\cmidrule(lr){9-12}\cmidrule(lr){13-16}\cmidrule(lr){17-20}\cmidrule(lr){21-24}\cmidrule(lr){25-28}\cmidrule(lr){29-32}
			Name & $k$ & $|R|$ & bks & \multicolumn{3}{c}{Mean $\pm$ SD} & Best & \multicolumn{3}{c}{Mean $\pm$ SD} & Best & \multicolumn{3}{c}{Mean $\pm$ SD} & Best & \multicolumn{3}{c}{Mean $\pm$ SD} & Best & \multicolumn{3}{c}{Mean $\pm$ SD} & Best & \multicolumn{3}{c}{Mean $\pm$ SD} & Best & \multicolumn{3}{c}{Mean $\pm$ SD} & Best\\
			\midrule
			R1a & 3 & 24 & 190.02 & 18.26 & $\pm$ & 0.00 & 18.26 & 6.03 & $\pm$ & 2.00 & 3.19 & 3.62 & $\pm$ & 1.54 & 0.71 & \bfseries 0.00 & $\pm$ & 0.00 & \textbf{0.00} & \multicolumn{1}{c}{} & - & \multicolumn{1}{c}{} & \multicolumn{1}{c}{-} & 2.98 & $\pm$ & 1.42 & 0.54 & \bfseries 0.00 & $\pm$ & 0.00 & \textbf{0.00} \\
			R1b & 3 & 24 & 164.46 & 33.22 & $\pm$ & 0.00 & 33.22 & 7.81 & $\pm$ & 3.66 & 4.34 & 12.81 & $\pm$ & 4.15 & 4.32 & \bfseries 0.00 & $\pm$ & 0.00 & \textbf{0.00} & \bfseries 0.00 & $\pm$ & 0.00 & \textbf{0.00} & 3.69 & $\pm$ & 1.33 & 1.38 & 0.82 & $\pm$ & 1.35 & \textbf{0.00} \\
			R7a & 4 & 36 & 291.71 & \multicolumn{1}{c}{} & - & \multicolumn{1}{c}{} & \multicolumn{1}{c}{-} & 11.22 & $\pm$ & 4.71 & 4.71 & 23.55 & $\pm$ & 14.85 & 6.09 & \bfseries 0.00 & $\pm$ & 0.00 & \textbf{0.00} & \multicolumn{1}{c}{} & - & \multicolumn{1}{c}{} & \multicolumn{1}{c}{-} & 4.39 & $\pm$ & 2.03 & 1.08 & 2.06 & $\pm$ & 1.34 & \textbf{0.00} \\
			R7b & 4 & 36 & 248.21 & \multicolumn{1}{c}{} & - & \multicolumn{1}{c}{} & \multicolumn{1}{c}{-} & 10.87 & $\pm$ & 1.96 & 7.33 & 12.95 & $\pm$ & 3.48 & 8.11 & \multicolumn{1}{c}{} & - & \multicolumn{1}{c}{} & \multicolumn{1}{c}{-} & \multicolumn{1}{c}{} & - & \multicolumn{1}{c}{} & \multicolumn{1}{c}{-} & 8.58 & $\pm$ & 2.32 & 5.75 & \bfseries 2.41 & $\pm$ & 1.46 & \textbf{1.07} \\
			R2a & 5 & 48 & 301.34 & \multicolumn{1}{c}{} & - & \multicolumn{1}{c}{} & \multicolumn{1}{c}{-} & 8.03 & $\pm$ & 2.79 & 2.87 & 29.61 & $\pm$ & 24.07\textsuperscript{\ensuremath{\dagger}} & 12.59 & \multicolumn{1}{c}{} & - & \multicolumn{1}{c}{} & \multicolumn{1}{c}{-} & \multicolumn{1}{c}{} & - & \multicolumn{1}{c}{} & \multicolumn{1}{c}{-} & 6.70 & $\pm$ & 4.01 & 2.39 & \bfseries 1.55 & $\pm$ & 1.05 & \textbf{0.29} \\
			R2b & 5 & 48 & 295.66 & 44.39 & $\pm$ & 0.00 & 44.39 & 12.40 & $\pm$ & 3.24 & 7.33 & 15.76 & $\pm$ & 2.11 & 11.88 & \multicolumn{1}{c}{} & - & \multicolumn{1}{c}{} & \multicolumn{1}{c}{-} & \bfseries 1.46 & $\pm$ & 0.00 & 1.46 & 8.07 & $\pm$ & 2.22 & 4.77 & 2.41 & $\pm$ & 0.80 & \textbf{0.76} \\
			R8a & 6 & 72 & 487.84 & \multicolumn{1}{c}{} & - & \multicolumn{1}{c}{} & \multicolumn{1}{c}{-} & \multicolumn{1}{c}{} & - & \multicolumn{1}{c}{} & \multicolumn{1}{c}{-} & \multicolumn{1}{c}{} & - & \multicolumn{1}{c}{} & \multicolumn{1}{c}{-} & \multicolumn{1}{c}{} & - & \multicolumn{1}{c}{} & \multicolumn{1}{c}{-} & \multicolumn{1}{c}{} & - & \multicolumn{1}{c}{} & \multicolumn{1}{c}{-} & 14.91 & $\pm$ & 3.43 & 10.81 & \bfseries 6.45 & $\pm$ & 1.63 & \textbf{3.73} \\
			R8b & 6 & 72 & 458.73 & \multicolumn{1}{c}{} & - & \multicolumn{1}{c}{} & \multicolumn{1}{c}{-} & 24.17 & $\pm$ & 2.77 & 20.69 & \multicolumn{1}{c}{} & - & \multicolumn{1}{c}{} & \multicolumn{1}{c}{-} & \multicolumn{1}{c}{} & - & \multicolumn{1}{c}{} & \multicolumn{1}{c}{-} & 10.15 & $\pm$ & 0.00 & 10.15 & 12.48 & $\pm$ & 2.08 & 8.62 & \bfseries 7.06 & $\pm$ & 2.14 & \textbf{3.39} \\
			R3a & 7 & 72 & 532.00 & \multicolumn{1}{c}{} & - & \multicolumn{1}{c}{} & \multicolumn{1}{c}{-} & \multicolumn{1}{c}{} & - & \multicolumn{1}{c}{} & \multicolumn{1}{c}{-} & \multicolumn{1}{c}{} & - & \multicolumn{1}{c}{} & \multicolumn{1}{c}{-} & \multicolumn{1}{c}{} & - & \multicolumn{1}{c}{} & \multicolumn{1}{c}{-} & \multicolumn{1}{c}{} & - & \multicolumn{1}{c}{} & \multicolumn{1}{c}{-} & 10.21 & $\pm$ & 2.73 & 5.23 & \bfseries 4.04 & $\pm$ & 1.13 & \textbf{2.30} \\
			R3b & 7 & 72 & 484.83 & \multicolumn{1}{c}{} & - & \multicolumn{1}{c}{} & \multicolumn{1}{c}{-} & \multicolumn{1}{c}{} & - & \multicolumn{1}{c}{} & \multicolumn{1}{c}{-} & 50.86 & $\pm$ & --\textsuperscript{\ensuremath{\dagger}} & 50.86 & \multicolumn{1}{c}{} & - & \multicolumn{1}{c}{} & \multicolumn{1}{c}{-} & \bfseries 5.75 & $\pm$ & 0.00 & 5.75 & 12.15 & $\pm$ & 2.01 & 9.38 & 6.03 & $\pm$ & 0.97 & \textbf{4.97} \\
			R9a & 8 & 108 & 653.94 & \multicolumn{1}{c}{} & - & \multicolumn{1}{c}{} & \multicolumn{1}{c}{-} & \multicolumn{1}{c}{} & - & \multicolumn{1}{c}{} & \multicolumn{1}{c}{-} & \multicolumn{1}{c}{} & - & \multicolumn{1}{c}{} & \multicolumn{1}{c}{-} & \multicolumn{1}{c}{} & - & \multicolumn{1}{c}{} & \multicolumn{1}{c}{-} & \multicolumn{1}{c}{} & - & \multicolumn{1}{c}{} & \multicolumn{1}{c}{-} & 27.42 & $\pm$ & 2.52 & 22.50 & \bfseries 7.18 & $\pm$ & 1.84 & \textbf{3.99} \\
			R9b & 8 & 108 & 592.23 & \multicolumn{1}{c}{} & - & \multicolumn{1}{c}{} & \multicolumn{1}{c}{-} & \multicolumn{1}{c}{} & - & \multicolumn{1}{c}{} & \multicolumn{1}{c}{-} & \multicolumn{1}{c}{} & - & \multicolumn{1}{c}{} & \multicolumn{1}{c}{-} & \multicolumn{1}{c}{} & - & \multicolumn{1}{c}{} & \multicolumn{1}{c}{-} & \multicolumn{1}{c}{} & - & \multicolumn{1}{c}{} & \multicolumn{1}{c}{-} & 16.05 & $\pm$ & 3.11 & 13.02 & \bfseries 8.72 & $\pm$ & 1.75 & \textbf{6.10} \\
			R4a & 9 & 96 & 570.25 & \multicolumn{1}{c}{} & - & \multicolumn{1}{c}{} & \multicolumn{1}{c}{-} & \multicolumn{1}{c}{} & - & \multicolumn{1}{c}{} & \multicolumn{1}{c}{-} & \multicolumn{1}{c}{} & - & \multicolumn{1}{c}{} & \multicolumn{1}{c}{-} & \multicolumn{1}{c}{} & - & \multicolumn{1}{c}{} & \multicolumn{1}{c}{-} & 9.62 & $\pm$ & 0.00 & 9.62 & 16.61 & $\pm$ & 2.53 & 12.26 & \bfseries 7.82 & $\pm$ & 1.56 & \textbf{3.80} \\
			R4b & 9 & 96 & 529.33 & \multicolumn{1}{c}{} & - & \multicolumn{1}{c}{} & \multicolumn{1}{c}{-} & \multicolumn{1}{c}{} & - & \multicolumn{1}{c}{} & \multicolumn{1}{c}{-} & 36.66 & $\pm$ & 13.98 & 25.12 & \multicolumn{1}{c}{} & - & \multicolumn{1}{c}{} & \multicolumn{1}{c}{-} & 11.33 & $\pm$ & 0.00 & 11.33 & 16.44 & $\pm$ & 2.17 & 12.43 & \bfseries 9.12 & $\pm$ & 2.05 & \textbf{5.78} \\
			R10a & 10 & 144 & 845.47 & \multicolumn{1}{c}{} & - & \multicolumn{1}{c}{} & \multicolumn{1}{c}{-} & \multicolumn{1}{c}{} & - & \multicolumn{1}{c}{} & \multicolumn{1}{c}{-} & \multicolumn{1}{c}{} & - & \multicolumn{1}{c}{} & \multicolumn{1}{c}{-} & \multicolumn{1}{c}{} & - & \multicolumn{1}{c}{} & \multicolumn{1}{c}{-} & \multicolumn{1}{c}{} & - & \multicolumn{1}{c}{} & \multicolumn{1}{c}{-} & 25.56 & $\pm$ & 3.73 & 20.74 & \bfseries 10.77 & $\pm$ & 1.09 & \textbf{9.12} \\
			R10b & 10 & 144 & 783.81 & \multicolumn{1}{c}{} & - & \multicolumn{1}{c}{} & \multicolumn{1}{c}{-} & \multicolumn{1}{c}{} & - & \multicolumn{1}{c}{} & \multicolumn{1}{c}{-} & \multicolumn{1}{c}{} & - & \multicolumn{1}{c}{} & \multicolumn{1}{c}{-} & \multicolumn{1}{c}{} & - & \multicolumn{1}{c}{} & \multicolumn{1}{c}{-} & \multicolumn{1}{c}{} & - & \multicolumn{1}{c}{} & \multicolumn{1}{c}{-} & 25.55 & $\pm$ & 3.81 & 17.25 & \bfseries 11.68 & $\pm$ & 1.46 & \textbf{8.89} \\
			R5a & 11 & 120 & 625.64 & \multicolumn{1}{c}{} & - & \multicolumn{1}{c}{} & \multicolumn{1}{c}{-} & \multicolumn{1}{c}{} & - & \multicolumn{1}{c}{} & \multicolumn{1}{c}{-} & \multicolumn{1}{c}{} & - & \multicolumn{1}{c}{} & \multicolumn{1}{c}{-} & \multicolumn{1}{c}{} & - & \multicolumn{1}{c}{} & \multicolumn{1}{c}{-} & 10.06 & $\pm$ & 0.00 & 10.06 & 15.52 & $\pm$ & 2.15 & 12.42 & \bfseries 9.64 & $\pm$ & 1.36 & \textbf{7.75} \\
			R5b & 11 & 120 & 573.56 & \multicolumn{1}{c}{} & - & \multicolumn{1}{c}{} & \multicolumn{1}{c}{-} & \multicolumn{1}{c}{} & - & \multicolumn{1}{c}{} & \multicolumn{1}{c}{-} & \multicolumn{1}{c}{} & - & \multicolumn{1}{c}{} & \multicolumn{1}{c}{-} & \multicolumn{1}{c}{} & - & \multicolumn{1}{c}{} & \multicolumn{1}{c}{-} & 9.38 & $\pm$ & 1.14 & 7.73 & 14.66 & $\pm$ & 2.23 & 12.34 & \bfseries 8.21 & $\pm$ & 1.52 & \textbf{5.91} \\
			R6a & 13 & 144 & 783.78 & \multicolumn{1}{c}{} & - & \multicolumn{1}{c}{} & \multicolumn{1}{c}{-} & \multicolumn{1}{c}{} & - & \multicolumn{1}{c}{} & \multicolumn{1}{c}{-} & \multicolumn{1}{c}{} & - & \multicolumn{1}{c}{} & \multicolumn{1}{c}{-} & \multicolumn{1}{c}{} & - & \multicolumn{1}{c}{} & \multicolumn{1}{c}{-} & 10.99 & $\pm$ & 0.34 & 10.03 & 18.07 & $\pm$ & 2.88 & 14.00 & \bfseries 10.60 & $\pm$ & 1.82 & \textbf{8.58} \\
			R6b & 13 & 144 & 725.22 & \multicolumn{1}{c}{} & - & \multicolumn{1}{c}{} & \multicolumn{1}{c}{-} & \multicolumn{1}{c}{} & - & \multicolumn{1}{c}{} & \multicolumn{1}{c}{-} & \multicolumn{1}{c}{} & - & \multicolumn{1}{c}{} & \multicolumn{1}{c}{-} & \multicolumn{1}{c}{} & - & \multicolumn{1}{c}{} & \multicolumn{1}{c}{-} & 13.36 & $\pm$ & 0.59 & 11.67 & 16.21 & $\pm$ & 2.43 & 12.87 & \bfseries 9.76 & $\pm$ & 1.65 & \textbf{7.40} \\
			\bottomrule
		\end{tabular}
		\par\vspace{0.5ex}
		\begin{minipage}{0.98\linewidth}
			\textit{Note.} $^\dagger$ CPO on R2a: 2 feasible runs; CPO on R3b: 1 feasible run. For marked entries, the mean, sample standard deviation, and best gap are computed only over the stated feasible runs. The sample standard deviation for CPO on R3b is undefined for one feasible run and is shown as --.
		\end{minipage}
		\endgroup
	\end{table}
\end{landscape}

%% file: tables/pcsp-landscape.tex
% Requires \usepackage{pdflscape} and \usepackage{siunitx} in the document preamble.
\begin{landscape}
	\begin{table}[t]
		\centering
		\caption{Primal gaps (\%) to the best known solution (bks) on PCSP instances. Entries report the mean primal gap $\pm$ the sample standard deviation across constituent instance-level mean gaps; best gaps retain the existing mean of per-instance best gaps. A dash indicates that no feasible solution was found for the corresponding group. Except for entries marked with $^\dagger$, numerical entries use all ten runs of each constituent instance. For marked entries, the affected instance contribution uses its feasible runs only, as specified in the table note. Best mean and best gap values are in bold.}
		\label{tab:pcsp-landscape}
		\begingroup
		\footnotesize
		\setlength{\tabcolsep}{3pt}
		\renewcommand{\arraystretch}{0.95}
		\sisetup{detect-weight = true, detect-family = true, table-number-alignment = right}
		\begin{tabular}{lrrS[table-format=2.2]@{\,}c@{\,}lrS[table-format=2.2]@{\,}c@{\,}lrS[table-format=2.2]@{\,}c@{\,}lrS[table-format=2.2]@{\,}c@{\,}lrS[table-format=2.2]@{\,}c@{\,}lrS[table-format=2.2]@{\,}c@{\,}lr}
			\toprule
			& & & \multicolumn{4}{c}{MZN-Succ} & \multicolumn{4}{c}{MZN-Disj} & \multicolumn{4}{c}{CPO} & \multicolumn{4}{c}{\shortstack{OR-Tools Routing}} & \multicolumn{4}{c}{Hexaly} & \multicolumn{4}{c}{SeqVar}\\
			\cmidrule(lr){4-7}\cmidrule(lr){8-11}\cmidrule(lr){12-15}\cmidrule(lr){16-19}\cmidrule(lr){20-23}\cmidrule(lr){24-27}
			Dataset & $n$ & $m$ & \multicolumn{3}{c}{Mean $\pm$ SD} & Best & \multicolumn{3}{c}{Mean $\pm$ SD} & Best & \multicolumn{3}{c}{Mean $\pm$ SD} & Best & \multicolumn{3}{c}{Mean $\pm$ SD} & Best & \multicolumn{3}{c}{Mean $\pm$ SD} & Best & \multicolumn{3}{c}{Mean $\pm$ SD} & Best\\
			\midrule
			\multirow{16}{*}{A\_E} & \multirow{2}{*}{50} & 3 & 17.11 & $\pm$ & 2.19 & 16.40 & 19.34 & $\pm$ & 0.84 & 19.25 & \bfseries 0.00 & $\pm$ & 0.00 & \textbf{0.00} & 4.25 & $\pm$ & 1.59 & 4.25 & 32.30 & $\pm$ & 3.54 & 25.26 & \bfseries 0.00 & $\pm$ & 0.00 & \textbf{0.00} \\
			&  & 4 & 18.35 & $\pm$ & 1.94 & 17.81 & 20.02 & $\pm$ & 3.03 & 19.86 & 0.11 & $\pm$ & 0.18 & \textbf{0.00} & 3.81 & $\pm$ & 1.35 & 3.81 & 31.41 & $\pm$ & 2.59 & 28.11 & \bfseries 0.00 & $\pm$ & 0.00 & \textbf{0.00} \\
			& \multirow{2}{*}{100} & 3 & 61.79 & $\pm$ & 8.10 & 61.75 & 36.24 & $\pm$ & 3.76 & 36.24 & 0.75 & $\pm$ & 0.90 & \textbf{0.00} & 4.62 & $\pm$ & 0.29 & 4.62 & 39.28 & $\pm$ & 1.40 & 35.04 & \bfseries 0.41 & $\pm$ & 0.29 & 0.28 \\
			&  & 4 & 77.99 & $\pm$ & 11.88 & 77.99 & 56.11 & $\pm$ & 38.28 & 56.11 & 0.70 & $\pm$ & 0.74 & \textbf{0.00} & 3.93 & $\pm$ & 1.09 & 3.93 & 40.80 & $\pm$ & 3.84 & 37.13 & \bfseries 0.32 & $\pm$ & 0.43 & \textbf{0.00} \\
			& \multirow{2}{*}{150} & 3 & 96.44 & $\pm$ & 6.17 & 96.44 & 44.42 & $\pm$ & 2.57 & 44.29 & \bfseries 0.36 & $\pm$ & 0.53 & 0.28 & 3.83 & $\pm$ & 2.80 & 3.83 & 45.59 & $\pm$ & 1.73 & 38.84 & 0.59 & $\pm$ & 0.42 & \textbf{0.14} \\
			&  & 4 & 86.85 & $\pm$ & 4.92 & 86.85 & 45.05 & $\pm$ & 0.92 & 45.05 & 0.32 & $\pm$ & 0.32 & \textbf{0.00} & 5.09 & $\pm$ & 6.36 & 5.09 & 44.06 & $\pm$ & 1.35 & 40.54 & \bfseries 0.20 & $\pm$ & 0.30 & 0.11 \\
			& \multirow{2}{*}{200} & 3 & 79.93 & $\pm$ & 24.92 & 79.93 & 49.19 & $\pm$ & 0.79 & 49.19 & 0.45 & $\pm$ & 0.29 & \textbf{0.00} & 2.94 & $\pm$ & 1.48 & 2.94 & 47.28 & $\pm$ & 1.17 & 45.14 & \bfseries 0.10 & $\pm$ & 0.08 & 0.08 \\
			&  & 4 & 93.81 & $\pm$ & 5.49 & 93.81 & 49.08 & $\pm$ & 1.67 & 49.08 & \bfseries 0.13 & $\pm$ & 0.12 & \textbf{0.00} & 5.15 & $\pm$ & 1.46 & 5.15 & 44.64 & $\pm$ & 0.10 & 40.32 & 0.41 & $\pm$ & 0.52 & 0.05 \\
			& \multirow{2}{*}{250} & 3 & 96.68 & $\pm$ & 5.75 & 96.68 & 50.28 & $\pm$ & 2.75 & 50.11 & 0.22 & $\pm$ & 0.25 & \textbf{0.00} & 1.27 & $\pm$ & 1.81 & 1.27 & 47.85 & $\pm$ & 1.10 & 42.92 & \bfseries 0.18 & $\pm$ & 0.16 & \textbf{0.00} \\
			&  & 4 & 86.94 & $\pm$ & 17.62 & 86.94 & 51.84 & $\pm$ & 2.62 & 51.84 & 0.66 & $\pm$ & 0.34 & \textbf{0.21} & 4.03 & $\pm$ & 2.70 & 4.03 & 45.53 & $\pm$ & 0.69 & 42.14 & \bfseries 0.22 & $\pm$ & 0.39 & \textbf{0.21} \\
			& \multirow{2}{*}{300} & 3 & 89.93 & $\pm$ & 11.34 & 89.93 & 53.69 & $\pm$ & 3.29 & 53.62 & \bfseries 0.69 & $\pm$ & 0.13 & \textbf{0.08} & 3.48 & $\pm$ & 0.94 & 3.48 & 47.00 & $\pm$ & 0.96 & 44.05 & 0.82 & $\pm$ & 0.98 & 0.29 \\
			&  & 4 & 98.70 & $\pm$ & 2.25 & 98.70 & 53.74 & $\pm$ & 1.81 & 53.74 & \bfseries 0.86 & $\pm$ & 0.63 & \textbf{0.08} & 2.87 & $\pm$ & 1.55 & 2.87 & 47.51 & $\pm$ & 1.62 & 42.23 & 1.12 & $\pm$ & 1.02 & 0.72 \\
			& \multirow{2}{*}{400} & 3 & 99.46 & $\pm$ & 0.93 & 99.46 & 54.31 & $\pm$ & 4.22 & 54.08 & \bfseries 1.09 & $\pm$ & 1.07 & \textbf{0.31} & 2.67 & $\pm$ & 1.42 & 2.67 & 49.53 & $\pm$ & 0.60 & 45.85 & 1.31 & $\pm$ & 1.09 & 0.39 \\
			&  & 4 & 97.20 & $\pm$ & 4.84 & 97.20 & 69.70 & $\pm$ & 26.28 & 69.70 & \bfseries 0.39 & $\pm$ & 0.19 & \textbf{0.13} & 1.59 & $\pm$ & 0.34 & 1.59 & 49.16 & $\pm$ & 0.91 & 46.18 & 0.68 & $\pm$ & 0.29 & 0.15 \\
			& \multirow{2}{*}{500} & 3 & 97.06 & $\pm$ & 5.09 & 97.06 & 71.06 & $\pm$ & 25.12 & 70.73 & \bfseries 0.79 & $\pm$ & 0.23 & \textbf{0.00} & 2.35 & $\pm$ & 1.86 & 2.35 & 51.10 & $\pm$ & 0.78 & 48.40 & 0.90 & $\pm$ & 0.60 & 0.45 \\
			&  & 4 & 95.85 & $\pm$ & 4.77 & 95.41 & 72.07 & $\pm$ & 24.19 & 72.07 & \bfseries 0.50 & $\pm$ & 0.43 & \textbf{0.18} & 3.08 & $\pm$ & 0.77 & 3.08 & 49.78 & $\pm$ & 1.58 & 46.94 & 1.41 & $\pm$ & 0.47 & 0.54 \\
			\midrule
			\multirow{16}{*}{T} & \multirow{2}{*}{50} & 2 & 7.69 & $\pm$ & 4.58 & 6.36 & \bfseries 0.00 & $\pm$ & 0.00 & \textbf{0.00} & \bfseries 0.00 & $\pm$ & 0.00 & \textbf{0.00} & 5.64 & $\pm$ & 4.29 & 5.64 & 13.20 & $\pm$ & 4.80 & 7.91 & \bfseries 0.00 & $\pm$ & 0.00 & \textbf{0.00} \\
			&  & 3 & 2.39 & $\pm$ & 2.49 & 2.17 & \bfseries 0.00 & $\pm$ & 0.00 & \textbf{0.00} & \bfseries 0.00 & $\pm$ & 0.00 & \textbf{0.00} & 7.78 & $\pm$ & 2.18 & 7.78 & 12.09 & $\pm$ & 1.51 & 8.27 & \bfseries 0.00 & $\pm$ & 0.00 & \textbf{0.00} \\
			& \multirow{2}{*}{100} & 2 & 72.27 & $\pm$ & 22.53 & 72.01 & 20.94 & $\pm$ & 0.59 & 20.94 & 0.87 & $\pm$ & 0.62 & 0.38 & 7.46 & $\pm$ & 2.05 & 7.46 & 23.21 & $\pm$ & 0.10 & 18.18 & \bfseries 0.65 & $\pm$ & 0.62 & \textbf{0.00} \\
			&  & 3 & 72.78 & $\pm$ & 2.59 & 72.78 & 15.47 & $\pm$ & 1.88 & 15.25 & \bfseries 0.50 & $\pm$ & 0.23 & \textbf{0.00} & 7.56 & $\pm$ & 2.10 & 7.56 & 19.76 & $\pm$ & 1.20 & 15.36 & 0.51 & $\pm$ & 0.15 & \textbf{0.00} \\
			& \multirow{2}{*}{150} & 2 & 91.94 & $\pm$ & 13.95 & 91.94 & 32.49 & $\pm$ & 1.26 & 32.49 & \bfseries 2.00 & $\pm$ & 0.70 & \textbf{0.83} & 10.27 & $\pm$ & 2.89 & 10.27 & 29.88 & $\pm$ & 3.14 & 24.95 & 2.17 & $\pm$ & 0.87 & 1.25 \\
			&  & 3 & 82.11 & $\pm$ & 15.69 & 79.41 & 23.56 & $\pm$ & 2.41 & 23.56 & \bfseries 0.35 & $\pm$ & 0.36 & \textbf{0.08} & 6.81 & $\pm$ & 0.70 & 6.81 & 24.17 & $\pm$ & 0.23 & 20.17 & 2.03 & $\pm$ & 0.39 & 1.26 \\
			& \multirow{2}{*}{200} & 2 & 96.72 & $\pm$ & 5.68 & 96.72 & 43.53 & $\pm$ & 2.39 & 43.53 & \bfseries 2.08 & $\pm$ & 0.43 & \textbf{1.52} & 9.31 & $\pm$ & 0.76 & 9.31 & 32.12 & $\pm$ & 0.47 & 27.19 & 3.21 & $\pm$ & 0.02 & 2.20 \\
			&  & 3 & 88.27 & $\pm$ & 10.35 & 88.27 & 35.79 & $\pm$ & 2.19 & 35.36 & \bfseries 0.62 & $\pm$ & 0.28 & \textbf{0.27} & 6.58 & $\pm$ & 1.09 & 6.57 & 25.36 & $\pm$ & 1.88 & 21.70 & 2.10 & $\pm$ & 0.21 & 1.44 \\
			& \multirow{2}{*}{250} & 2 & \multicolumn{1}{c}{} & - & \multicolumn{1}{c}{} & \multicolumn{1}{c}{-} & 47.90 & $\pm$ & 2.61 & 47.90 & \bfseries 1.57 & $\pm$ & 0.70 & \textbf{0.82} & 8.44 & $\pm$ & 0.25 & 8.44 & 32.25 & $\pm$ & 1.70 & 27.14 & 2.83 & $\pm$ & 0.78 & 1.93 \\
			&  & 3 & 92.84 & $\pm$ & 7.48 & 92.79 & 44.55 & $\pm$ & 6.00 & 44.05 & \bfseries 0.19 & $\pm$ & 0.11 & \textbf{0.00} & 8.29 & $\pm$ & 1.22 & 8.29 & 25.99 & $\pm$ & 0.80 & 22.88 & 2.15 & $\pm$ & 0.53 & 1.35 \\
			& \multirow{2}{*}{300} & 2 & 96.26 & $\pm$ & 6.48 & 96.26 & 56.80 & $\pm$ & 9.47\textsuperscript{\ensuremath{\dagger}} & 53.46 & \bfseries 1.39 & $\pm$ & 0.04 & \textbf{0.81} & 9.26 & $\pm$ & 1.72 & 9.26 & 31.99 & $\pm$ & 1.28 & 27.32 & 2.44 & $\pm$ & 0.58 & 1.47 \\
			&  & 3 & 96.64 & $\pm$ & 5.73 & 96.64 & 68.25 & $\pm$ & 29.11\textsuperscript{\ensuremath{\dagger}} & 66.78 & \bfseries 0.27 & $\pm$ & 0.13 & \textbf{0.00} & 6.59 & $\pm$ & 0.22 & 6.59 & 25.36 & $\pm$ & 0.68 & 21.76 & 2.08 & $\pm$ & 0.30 & 1.51 \\
			& \multirow{2}{*}{400} & 2 & \multicolumn{1}{c}{} & - & \multicolumn{1}{c}{} & \multicolumn{1}{c}{-} & \multicolumn{1}{c}{} & - & \multicolumn{1}{c}{} & \multicolumn{1}{c}{-} & \bfseries 1.77 & $\pm$ & 0.13 & \textbf{1.06} & 8.40 & $\pm$ & 1.04 & 8.40 & 33.40 & $\pm$ & 4.53 & 29.98 & 3.36 & $\pm$ & 0.31 & 2.42 \\
			&  & 3 & 96.38 & $\pm$ & 3.52 & 96.16 & 63.25 & $\pm$ & 31.83 & 59.61 & \bfseries 0.26 & $\pm$ & 0.10 & \textbf{0.00} & 8.76 & $\pm$ & 0.77 & 8.67 & 24.83 & $\pm$ & 1.01 & 22.41 & 2.90 & $\pm$ & 0.44 & 2.05 \\
			& \multirow{2}{*}{500} & 2 & \multicolumn{1}{c}{} & - & \multicolumn{1}{c}{} & \multicolumn{1}{c}{-} & \multicolumn{1}{c}{} & - & \multicolumn{1}{c}{} & \multicolumn{1}{c}{-} & \bfseries 2.14 & $\pm$ & 0.23 & \textbf{1.46} & 10.22 & $\pm$ & 1.87 & 10.18 & 33.73 & $\pm$ & 0.50 & 31.78 & 4.05 & $\pm$ & 0.03 & 3.33 \\
			&  & 3 & \multicolumn{1}{c}{} & - & \multicolumn{1}{c}{} & \multicolumn{1}{c}{-} & \multicolumn{1}{c}{} & - & \multicolumn{1}{c}{} & \multicolumn{1}{c}{-} & \bfseries 0.69 & $\pm$ & 0.24 & \textbf{0.34} & 9.19 & $\pm$ & 0.41 & 9.19 & 25.59 & $\pm$ & 0.89 & 22.33 & 4.20 & $\pm$ & 0.72 & 3.28 \\
			\bottomrule
		\end{tabular}
		\par\vspace{0.5ex}
		\begin{minipage}{0.98\linewidth}
			\textit{Note.} $^\dagger$ MZN-Disj on T ($n$=300, $m$=2): 4 feasible runs; MZN-Disj on T ($n$=300, $m$=3): 4 feasible runs. For marked entries, the contribution of the stated instance to the mean, sample standard deviation, and best gap is computed only over its feasible runs.
		\end{minipage}
		\endgroup
	\end{table}
\end{landscape}